%% file: main_AISTATS.tex
\documentclass[twoside]{article}

\usepackage[accepted]{aistats2026}

\usepackage{graphicx}
\usepackage{booktabs} %
\usepackage{hyperref}

\usepackage{amsmath}
\usepackage{amssymb}
\usepackage{mathtools}
\usepackage{amsthm}

\usepackage[capitalize,noabbrev]{cleveref}

\input{AISTATS/header/header}

\input{AISTATS/header/def}
\input{AISTATS/header/def-letters}

\input{AISTATS/header/def-policy-gradient}

\usepackage[round]{natbib}

\hypersetup{
    colorlinks,
    allcolors={blue!80!black}
}

\begin{document}

\runningtitle{Global Convergence Rates for  Federated Softmax Policy Gradient under Heterogeneous Environments}

\runningauthor{Safwan Labbi, Paul Mangold, Daniil Tiapkin, Eric Moulines}

\twocolumn[

\aistatstitle{On Global Convergence Rates for  Federated Softmax Policy Gradient under Heterogeneous Environments}

\aistatsauthor{Safwan Labbi\textsuperscript{1} 
\, Paul Mangold\textsuperscript{1} \,  Daniil Tiapkin\textsuperscript{1,2}  \, Eric Moulines\textsuperscript{3,4}}

\aistatsaddress{
\textsuperscript{1}  CMAP, CNRS, École Polytechnique, Institut Polytechnique de Paris, 91120 Palaiseau, France \\
\textsuperscript{2}  
Université Paris-Saclay, CNRS, LMO, 91405, Orsay, France\\
\textsuperscript{3} Mohamed bin Zayed University of Artificial Intelligence, UAE
\\
\textsuperscript{4} LRE EPITA , 94270 Le Kremlin-Bicêtre, France
} ]

\begin{abstract}
We provide global convergence rates~for vanilla and entropy-regularized federated softmax stochastic policy gradient (\fedPG) with local training. 
We show that \fedPG\ converges to a near-optimal policy in terms of the average agent value, with a gap controlled by the level of heterogeneity.
Remarkably, we obtain the first convergence rates for entropy-regularized policy gradient \emph{with explicit constants}, leveraging a projection-like operator.
Our results build upon a new analysis of federated averaging for non-convex objectives, based on the observation that the Łojasiewicz-type inequalities from the single-agent setting \citep{mei2020global} do not hold for the federated objective. 
This uncovers a fundamental difference between single-agent and federated reinforcement learning: while single-agent optimal policies can be deterministic, federated objectives may inherently require stochastic policies.
\end{abstract}

\section{INTRODUCTION}
\input{AISTATS/src/introduction_new}

\section{RELATED WORK}
\label{sec:related-work}
\input{AISTATS/src/related_works_new}

\section{PRELIMINARIES}
\label{sec:bg}
\input{AISTATS/src/background}

\section{FEDAVG UNDER LOJASIEWICZ CONDITIONS}
\label{sec:analysis_fedAVG}
\input{AISTATS/src/FedAVG}

\section{ANALYSIS OF FEDPG}
\label{sec:fed-pg}

\input{AISTATS/src/solvingfrl}

\section{Structure of Heterogeneous FRL}
\label{sec:heteregenous_frl}

\input{AISTATS/src/Structure_FRL}

\section{EXPERIMENTS}
\label{sec:expe}
\input{AISTATS/src/experiments}

\section{CONCLUSION}
\input{AISTATS/src/Conclusion}

\section*{Acknowledgements}

The work of S. Labbi, and P. Mangold has been supported by Technology Innovation Institute (TII), project Fed2Learn. The work of D.Tiapkin has been supported by the Paris Île-de-France Région in the framework of DIM AI4IDF.
The work of E. Moulines has been partly funded by the European Union (ERC-2022-SYG-OCEAN-101071601). Views and opinions expressed are however those of the author(s) only and do not necessarily reflect those of the European Union or the European Research Council Executive Agency. Neither the European Union nor the granting authority can be held responsible for them.

\bibliography{references}
\clearpage
\newpage

\section*{Checklist}

\begin{enumerate}

  \item For all models and algorithms presented, check if you include:
  \begin{enumerate}
    \item A clear description of the mathematical setting, assumptions, algorithm, and/or model.
    \textit{Answer:} \textbf{Yes} \\
    \textit{Justification:} \textbf{The mathematical setting, the assumptions, and the algorithm are extensively described in \Cref{sec:bg,sec:analysis_fedAVG,sec:fed-pg}}
    \item An analysis of the properties and complexity (time, space, sample size) of any algorithm. 
    \textit{Answer:} \textbf{Yes}
    \\
    \textit{Justification:} \textbf{Sample and communication complexity results are derived in \Cref{sec:analysis_fedAVG,sec:fed-pg}.}
    \item (Optional) Anonymized source code, with specification of all dependencies, including external libraries. \\
    \textit{Answer:} \textbf{Yes}
  \end{enumerate}

  \item For any theoretical claim, check if you include:
  \begin{enumerate}
    \item Statements of the full set of assumptions of all theoretical results.\\
    \textit{Answer:} \textbf{Yes} \\
    \textit{Justification:} \textbf{Assumptions required explicitly stated in \Cref{sec:analysis_fedAVG,sec:fed-pg} and referenced in the main theorems.}
    \item Complete proofs of all theoretical results. \\
    \textit{Answer:} \textbf{Yes} \\
    \textit{Justification:} \textbf{The proofs of all results are provided in the appendix.}
    \item Clear explanations of any assumptions. \\
    \textit{Answer:} \textbf{Yes} \\
    \textit{Justification:} \textbf{We provide intuition and justification for assumptions in \Cref{sec:analysis_fedAVG,sec:fed-pg}.}
  \end{enumerate}

  \item For all figures and tables that present empirical results, check if you include:
  \begin{enumerate}
    \item The code, data, and instructions needed to reproduce the main experimental results (either in the supplemental material or as a URL).
    \\
    \textit{Answer:} \textbf{Yes} \\
    \textit{Justification:} \textbf{The complete code is provided in the supplementary material.}
    \item All the training details (e.g., data splits, hyperparameters, how they were chosen). \\
    \textit{Answer}: \textit{Yes} \\
    \textit{Justification:} \textbf{Training details and parameters are described in \Cref{sec:experiments}.}
    \item A clear definition of the specific measure or statistics and error bars (e.g., with respect to the random seed after running experiments multiple times). \\
    \textit{Answer:} \textbf{Yes} \\
    \textit{Justification:} \textbf{Results report averages over 4 runs with standard deviations \Cref{sec:experiments}.}
    \item A description of the computing infrastructure used. (e.g., type of GPUs, internal cluster, or cloud provider). \\
    \textit{Answer:} \textbf{Yes} 
  \end{enumerate}

  \item If you are using existing assets (e.g., code, data, models) or curating/releasing new assets, check if you include:
  \begin{enumerate}
    \item Citations of the creator If your work uses existing assets. \\
    \textit{Answer:} \textbf{Yes} 
    \item The license information of the assets, if applicable. \\
    \textit{Answer:} \textbf{Not Applicable}\\
    \textit{Justification:} \textbf{The environments used are standard benchmarks with no licensing restrictions.}
    \item New assets either in the supplemental material or as a URL, if applicable.\\
    \textit{Answer:} \textbf{Yes} 
    \item Information about consent from data providers/curators.\\
    \textit{Answer:} \textbf{Not Applicable}\\
    \item Discussion of sensible content if applicable, e.g., personally identifiable information or offensive content.\\
    \textit{Answer:} \textbf{Not Applicable}
  \end{enumerate}

  \item If you used crowdsourcing or conducted research with human subjects, check if you include:
  \begin{enumerate}
    \item The full text of instructions given to participants and screenshots.
    \\
    \textit{Answer:} \textbf{Not Applicable}
    \item Descriptions of potential participant risks, with links to Institutional Review Board (IRB) approvals if applicable. \\
    \textit{Answer:} \textbf{Not Applicable}
    \item The estimated hourly wage paid to participants and the total amount spent on participant compensation. \\
    \textit{Answer:} \textbf{Not Applicable}
  \end{enumerate}

\end{enumerate}

\clearpage
\appendix
\thispagestyle{empty}

\onecolumn
\aistatstitle{Supplementary Materials}

\allowdisplaybreaks

\tableofcontents

\clearpage
\newpage

\section{Notations} 
\begin{table}[h]
\centering
\begin{tabularx}{\textwidth}{>{\hsize=0.5\hsize}X >{\hsize=1.6\hsize}X >{\hsize=0.2\hsize}X}
\toprule
\textbf{Symbols} & \textbf{Meaning} & \textbf{Definition} \\
\midrule
$\S$ & State space  & \Cref{sec:fedrl} \\
$\A$ & Action space & \Cref{sec:fedrl} \\
$\nagent$ & Number of agents & \Cref{sec:fedrl} \\
$ \kerMDP[c]$ & Transition kernel of agent $c$ & \Cref{sec:fedrl} \\
$ \rewardMDP[c]$ & Reward function of agent $c$ & \Cref{sec:fedrl} \\
$\discount$ & Discount factor & \Cref{sec:fedrl} \\
$\initdist$ & Initial distribution  over the state space & \Cref{sec:fedrl} \\
$\hgkernel$ & Heterogeneity on the transition kernels & \Cref{def:hg_kernel} \\
$\hgreward$ & Heterogeneity on the Rewards & \Cref{def:hg_reward} \\
$\frlobjective$ & Global Non regularized FRL objective & \Cref{def:objective_unregularised} \\
$\frllocfunc[c]$ & Local Non regularized objective of agent $c$ & \Cref{def:local_objective_unregularised} \\
 $\temp$& Regularization temperature  & \Cref{sec:bg} \\
$\regfrlobjective$ & Global regularized FRL objective & \Cref{def:objective_regularised} \\
$\regfrllocfunc[c]$ & Local regularized objective of agent $c$ & \Cref{sec:bg} \\
\midrule
$\trounds$ & Number of Communication Rounds of \fedPG/\projfedAVG & \Cref{sec:analysis_fedAVG} \\
$\lsteps$ & Number of local steps of \fedPG/\projfedAVG & \Cref{sec:analysis_fedAVG} \\
$\lentrunc$ & Length of the sampled trajectories by $\fedPG$  & \Cref{sec:fed-pg} \\
$\sizebatch$ & Number of trajectories collected per iteration & \Cref{sec:fed-pg} \\
$\mathcal{T}$ & Projection operator used in \fedPG/\projfedAVG  & \Cref{sec:analysis_fedAVG} \\
$\policy_{\theta}$ & Softmax policy parametrized by $\param \in \logitspace $ & \Cref{sec:bg} \\
\midrule
$\gfunc$ & Global objective optimised by $\projfedAVG$ & \Cref{sec:analysis_fedAVG} \\
$\locfunc[c]$ & Local function of agent $c$ in \projfedAVG & \Cref{sec:analysis_fedAVG} \\
$\grb{c}{Z}(\theta)$ & Stochastic estimators of \projfedAVG & \Cref{sec:analysis_fedAVG} \\
$\egrb{c}(\theta)$ & Expected value of the stochastic estimators of \projfedAVG & \Cref{sec:analysis_fedAVG} \\
$\boundgrad$ & Bound on the gradients of $\locfunc[c]$ & \Cref{sec:analysis_fedAVG} \\
$\secboundgrad$ & Smoothness constants of $\locfunc[c]$ and $\egrb{c}$  & \Cref{sec:analysis_fedAVG} \\
$\thirddbound$ & Bounds on the third-order derivative tensors of $\locfunc[c]$ & \Cref{sec:analysis_fedAVG} \\
$\moments{p}{p}$ & Bounds on the $p$-th central moments of $\grb{c}{Z}(\theta)$ for $p\in\{2,4\}$ & \Cref{sec:analysis_fedAVG}\\
$\bias$ &  Bounds on bias of $\grb{c}{Z}(\theta)$
& \Cref{sec:analysis_fedAVG} \\
$\hgty$ & Bound on gradient heterogeneity of $\gfunc$ & \Cref{sec:analysis_fedAVG} \\
$\mufl[c]$ &  'Quadratic' Łojasiewicz coefficient of agent $c$ & \Cref{sec:analysis_fedAVG} \\
$\lmufl[c]$ &  Polyak-Łojasiewicz coefficient of agent $c$ & \Cref{sec:analysis_fedAVG} \\
\bottomrule
\end{tabularx}
\label{tab:summary_notations}
\end{table}
The cardinality (the number of elements) of a set $Y$ is denoted $|Y|$. We define the indicator function of an element $y\in Y$ as 
\begin{equation}
\begin{split}
\nonumber
\Ind_{y}(\cdot) \colon Y
& \longrightarrow \{0,1\}
\\
w
& \longmapsto
\begin{cases}1 & \text{ if } w=y  \eqsp, \\ 0, & \text { otherwise }  \eqsp. \end{cases}
\end{split}
\end{equation}
$\pscal{\cdot}{\cdot}$ denotes the Euclidean scalar product.
For a three-times differentiable function $f: \rset^d \rightarrow \rset$, we denote $\nabla f \in \rset^d$ its gradient, $\nabla^2 f \in \rset^{d \times d}$ its Hessian and $\nabla^3 f \in \rset^{d \times d \times d}$ its third-order derivative tensor.  %
and $X^{\otimes k}$ the $k$-th tensor power of a tensor $X$.
For two real-valued sequences $(a_r)_{r=0}^{\infty}$ and 
$(b_r)_{r=0}^{\infty}$, we write $a_r \hidleq b_r$ if there exists a constant $C>0$ such that $a_r\leq C b_r$ for any $r\geq 0$.

\section{Descent lemma}\label{secapp:ascent_lemma}
\input{AISTATS/appendix/fedavg}

\section{Analysis of \texorpdfstring{$\SoftfedPG$}{S-FedPG}}
\input{AISTATS/appendix/S_FedPG}

\section{Analysis of \texorpdfstring{$\RegSoftfedPG$}{RS-FedPG}}
\label{secappendix:regsoftfedpg}
\input{AISTATS/appendix/RS_FedPG}

\section{On the different classes of policies}
\input{AISTATS/appendix/counter_examples}

\section{Technical lemmas}
\input{AISTATS/appendix/technical_lemmas}

\end{document}

%% file: AISTATS/header/header.tex
\usepackage{xargs}
\usepackage{xspace}
\usepackage{xparse}

\usepackage{setspace}

\usepackage{xcolor}

\usepackage{booktabs}
\usepackage[flushleft]{threeparttable}

\usepackage{pifont}

\usepackage{xfrac}

\usepackage{graphicx}
\usepackage{subcaption}

\usepackage{algorithm}
\usepackage[noend]{algorithmic}

\usepackage{enumerate}
\usepackage{tabularx}

\usepackage{ifthen}
\usepackage{etoolbox}
\usepackage{mdframed}
\usepackage{colortbl} 
\mdfdefinestyle{myboxstyle}{
  linecolor=lightgray!20,
  linewidth=1pt,
  roundcorner=5pt,
  backgroundcolor=lightgray!50,
  innertopmargin=5pt,
  innerbottommargin=5pt,
  innerleftmargin=5pt,
  innerrightmargin=5pt,
  skipabove=0pt,
  skipbelow=0pt,
}

\usepackage{comment}

\definecolor{amaranth}{rgb}{0.8, 0, 0.34}

\usepackage{enumitem}

\usepackage{array} %
\usepackage{multirow}

\newcolumntype{C}{>{\centering\arraybackslash}X}

\usepackage{bbm}
\usepackage{bold-extra}

\usepackage[T1]{fontenc}

\usepackage{minitoc}

%% file: AISTATS/header/def.tex
\def\rset{\mathbb{R}}
\def\nset{\mathbb{N}}
\def\param{\theta}

\newcommandx{\firstsentence}[1]{\textcolor{purple}{#1}}

\newcommandx{\paul}[1]{\todo[color=green!20]{PM: #1}}
\newcommandx{\eric}[1]{\todo[color=blue!20]{EM: #1}}
\newcommandx{\pauli}[1]{\todo[inline,color=green!20]{PM: #1}}
\newcommandx{\erici}[1]{\todo[color=blue!20]{EM: #1}}

\newcommandx{\todoi}[1]{\todo[inline,color=red!50]{TODO: #1}}

\def\eqsp{\enspace}
\def\rmd{\mathrm{d}}
\newcommand{\eqdef}{\overset{\Delta}{=}}

\def\Id{\mathrm{Id}}

\def\PE{\mathbb{E}}
\newcommandx{\CPE}[3][1=]{{\mathbb E}^{#1}\left[#2 \middle\vert #3 \right]}
\newcommandx{\bCPE}[3][1=]{{\mathbb E}_{#1}\Big[#2 \Big\vert #3 \Big]}

\def\Vargen{\operatorname{Var}}
\newcommand{\Var}[1]{\Vargen(#1)}
\newcommandx{\Varp}[2][2=]{\ifthenelse{\equal{#2}{}}{\operatorname{H}_{#1}}{\operatorname{H}_{#1}(#2)}}

\newcommandx{\norm}[2][2=]{ \lVert #1 \rVert_{#2} }
\newcommandx{\bnorm}[2][2=]{ \Big\lVert #1 \Big\rVert_{#2} }
\newcommandx{\pscal}[3][3=]{ \langle #1 , #2 \rangle_{#3}}
\newcommandx{\bpscal}[3][3=]{ \Big\langle #1 , #2 \Big\rangle_{#3}}
\newcommandx{\abs}[1]{ | #1 | }
\newcommandx{\babs}[1]{ \Big| #1 \Big| }

\newtheoremstyle{italicClaim} %
  {3pt}                     %
  {3pt}                     %
  {\itshape}                %
  {}                         %
  {\bfseries}               %
  {.}                        %
  {5pt plus 1pt minus 1pt}   %
  {}                         %

\theoremstyle{italicClaim}

\Crefname{assumLHG}{\textbf{A}-\hspace{-3pt}}{\textbf{A}-\hspace{-3pt}}
\crefname{assumLHG}{\textbf{A}}{\textbf{A}\hspace{-2pt}}

\newtheorem{assumFL}{\textbf{FL}-\hspace{-3pt}}
\Crefname{assumFL}{\textbf{FL}-\hspace{-3pt}}{\textbf{FL}-\hspace{-3pt}}
\crefname{assumFL}{\textbf{FL}}{\textbf{FL}\hspace{-2pt}}

\newtheorem{assumPL}{\textbf{PL}-\hspace{-3pt}}
\Crefname{assumPL}{\textbf{PL}-\hspace{-3pt}}{\textbf{PL}-\hspace{-3pt}}
\crefname{assumPL}{\textbf{PL}}{\textbf{PL}\hspace{-2pt}}

\newtheorem{assumLL}{\textbf{PL}-\hspace{-3pt}}
\Crefname{assumLL}{\textbf{PL}-\hspace{-3pt}}{\textbf{PL}-\hspace{-3pt}}
\crefname{assumLL}{\textbf{PL}}{\textbf{PL}\hspace{-2pt}}

\newtheorem{assumQL}{\textbf{QL}-\hspace{-3pt}}
\Crefname{assumQL}{\textbf{QL}-\hspace{-3pt}}{\textbf{QL}-\hspace{-3pt}}
\crefname{assumQL}{\textbf{QL}}{\textbf{QL}\hspace{-2pt}}

\newtheorem{assumQLA}{\textbf{QL}-\hspace{-3pt}}
\Crefname{assumQLA}{\textbf{QL}-\hspace{-3pt}}{\textbf{QL}-\hspace{-3pt}}
\crefname{assumQLA}{\textbf{QL}}{\textbf{QL}\hspace{-2pt}}

\Crefname{assumLLA}{\textbf{LL}-\hspace{-3pt}}{\textbf{LL}-\hspace{-3pt}}
\crefname{assumLLA}{\textbf{LL}}{\textbf{LL}\hspace{-2pt}}

\theoremstyle{plain}
\newtheorem{theorem}{Theorem}[section]

\newtheorem{lemma}[theorem]{Lemma}
\newtheorem{corollary}[theorem]{Corollary}
\theoremstyle{definition}

\newtheoremstyle{italremark}%
  {3pt}{3pt}%
  {\itshape}%
  {}%
  {\bfseries}%
  {.}%
  { }%
  {}%

\theoremstyle{italremark}
\newtheorem{remark}[theorem]{Remark}

\DeclareMathOperator*{\argmax}{arg\,max}

\renewcommandx{\iint}[2]{\{#1,\dots,#2\}}

\newcommand{\fedAVG}{\texttt{FedAVG}}
\newcommand{\projfedAVG}
{\hyperref[algo:FEDAVG]
{\texttt{proj-FedAVG}}}
\newcommand{\fedPG}{\hyperref[algo:FEDPG]{\texttt{FedPG}}}

\newcommand{\SoftfedPG}{\hyperref[algo:FEDPG]{\texttt{S-FedPG}}}
\newcommand{\RegSoftfedPG}{\hyperref[algo:FEDPG]{\texttt{RS-FedPG}}}

\newcommand{\term}[1]{\mathbf{(#1)}}

\newtheorem*{assumMDP}{Assumption $\textbf{A}_{\rho}$}
\def\assumptionmdp{\hyperref[assum:sufficent_exploration]{$\textbf{A}_{\rho}$}}

\newcommandx{\cready}[1]{\textcolor{black}{#1}}

%% file: AISTATS/header/def-letters.tex
\ExplSyntaxOn

\NewDocumentCommand{\definealphabet}{mmmm}
 {%
  \int_step_inline:nnn { `#3 } { `#4 }
   {
    \cs_new_protected:cpx { #1 \char_generate:nn { ##1 }{ 11 } }
     {
      \exp_not:N #2 { \char_generate:nn { ##1 } { 11 } }
     }
   }
 }

\ExplSyntaxOff

\definealphabet{bb}{\mathbb}{A}{Z}
\definealphabet{mc}{\mathcal}{A}{Z}
\definealphabet{mc}{\mathcal}{a}{z}
\definealphabet{mf}{\mathfrak}{A}{Z}
\definealphabet{mf}{\mathfrak}{a}{z}
\definealphabet{ms}{\mathsf}{A}{Z}
\definealphabet{ms}{\mathsf}{a}{z}

\newcommand*{\ie}{i.e.\@\xspace}

\makeatletter
\newcommand*{\etc}{%
    \@ifnextchar{.}%
        {etc}%
        {etc.\@\xspace}%
}
\makeatother

\newtheorem{assum}{\textbf{A-}\hspace{-4pt}}
\crefname{assum}{A-\hspace{-4.5pt}}{A-\hspace{-4.5pt}}

%% file: AISTATS/header/def-policy-gradient.tex
\def\step{\eta}

\def\trounds{R}
\def\lsteps{H}

\newcommand{\policytoparam}{\mathcal{L}}
\newcommand{\policy}{\pi}
\newcommand{\truncpolicy}{\tilde{\pi}}

\newcommand{\initdist}{\rho}
\newcommand{\initdistmin}{\rho_{\min}}
\newcommand{\discount}{\gamma}
\newcommand{\pA}{\mathcal{P}(\A)}
\newcommand{\projball}{B(\temp)}

\newcommandx{\proj}[1][1=]{\text{proj}_{\projball}( #1 )}
\newcommandx{\fedproj}[1][1=]{\text{proj}_{\mathcal{T}}( #1 )}

\newcommand{\kergen}{\mathsf{P}}
\newcommandx{\kerMDP}[1][1=]{\kergen_{#1}}

\newcommandx{\extkerMDP}[1][1=]{\Bar{\kergen}_{#1}}

\newcommandx{\indkerMDP}[1][1=]{\kergen^{ind, \hspace{0.07em}#1}}

\newcommandx{\avgvaluefunc}[1][1=]{
\valuefuncgen^{\hspace{0.07em}\mathsf{avg} \ifthenelse{\equal{#1}{}}{}{,}#1}}

\newcommand{\hgkernel}{\varepsilon_{\kergen}}

\newcommand{\hgreward}{\varepsilon_{\rewardgen}}

\def\projset{\mathcal{T}}

\def\rewardgen{\mathsf{r}}
\newcommandx{\rewardMDP}[1][1=]{\rewardgen_{#1}}
\newcommandx{\extrewardMDP}[1][1=]{\bar{\rewardgen}^{\hspace{0.07em}#1}}

\def\valuefuncgen{V}
\def\qfuncgen{Q}
\def\advfuncgen{A}

\newcommandx{\valuefunc}[2][1=,2=]{\valuefuncgen_{#1}^{#2}}
\newcommandx{\regvaluefunc}[2][1=,2=]{\widetilde{\valuefuncgen}_{#1}^{\hspace{0.07em}#2}}
\newcommandx{\bregvaluefunc}[2][1=, 2=]{\widetilde{\valuefuncgen}_{\mathrm{b}, #1}^{#2}}
\newcommandx{\extregvaluefunc}[2][1=,2=]{\bar{\valuefuncgen}_{#1}^{#2}}

\newcommandx{\qfunc}[2][1=,2=]{\qfuncgen_{#1}^{#2}}
\newcommandx{\regqfunc}[2][1=,2=]{\widetilde{\qfuncgen}_{#1}^{\hspace{0.07em}#2}}
\newcommandx{\advvalue}[2][1=,2=]{\advfuncgen_{#1}^{#2}}
\newcommandx{\regadvvalue}[2][1=, 2=]{\widetilde{\advfuncgen}_{#1}^{\hspace{0.07em}#2}}

\newcommandx{\occupancy}[2][1=,2=]{d_{#1}^{\hspace{0.07em}#2}}

\newcommandx{\baroccupancy}[2][1=,2=]{\bar{d}_{#1}^{\hspace{0.07em}#2}}

\newcommand{\Ind}{\mathsf{1}}
\newcommand{\temp}{\lambda}

\newcommand{\regobjective}{J}

\newcommand{\auxobjective}{\tilde{J}_{\temp}}

\newcommand{\optregobjective}{J^{\star}}

\newcommand{\optauxobjective}{\tilde{J}_{\temp}^{\star}}

\newcommandx{\locfunc}[2][1=,2=]{f_{#1}^{#2}}

\newcommand{\frlobjective}{J}
\newcommand{\optfrlobjective}{J^{\star}}
\newcommandx{\frllocfunc}[1][1=]{J_{#1}}
\newcommand{\regfrlobjective}{\tilde{J}_{\temp}}
\newcommand{\optregfrlobjective}{\tilde{J}^{\star}_{\temp}}
\newcommandx{\regfrllocfunc}[1][1=]{\tilde{J}_{#1,\temp}}

\newcommand{\gfunc}{F}
\newcommandx{\reglocfunc}[1][1=]{J_{#1}}
\newcommandx{\auxlocfunc}[1][1=]{\tilde{J}_{#1,\temp}}

\newcommandx{\optreglocfunc}[1][1=]{J_{#1}^{\star}}
\newcommandx{\optfrllocfunc}[1][1=]{J_{ #1}^{\star}}
\newcommandx{\optregfrllocfunc}[1][1=]{\tilde{J}_{ #1, \temp}^{\star}}
\newcommandx{\optauxlocfunc}[1][1=]{\tilde{J}_{#1, \temp}^{\star}}

\newcommandx{\regularisation}[2]{\mathcal{H}_{#1}^{\hspace{0.07em}#2}}

\def\sizebatch{B}
\def\lentrunc{T}

\def\bias{\beta}

\def\A{\mathcal{A}}
\def\S{\mathcal{S}}
\def\nstates{|\mathcal{S}|}
\def\nactions{|\mathcal{A}|}
\def\nagent{M}
\newcommand{\cM}{\mathcal{M}}

\newcommand{\PiSta}{\Pi_{\mathrm{sta}}}
\newcommand{\PiDet}{\Pi_{\mathrm{det}}}

\def\hgty{\zeta}

\def\softhgty{\zeta_{\mathrm{s}}}
\def\regsofthgty{\tilde{\zeta}_{\temp}}

\def\softsmoooth{L_{2,\mathrm{s}}}
\def\regsoftsmooth{\tilde{L}_{2,\temp}}

\newcommandx{\filtration}[2][2=]{\mathcal{F}_{#1}^{#2}}

\newcommandx{\flstar}[1][1=]{f_{#1}^{\star}}
\def\Fstar{F^{\star}}
\newcommandx{\mufl}[1][1=]{\mu_{#1}}
\def\minmufl{\underline{\mu}}
\newcommandx{\lmufl}[1][1=]{\tilde{\mu}_{#1}}
\def\minlmufl{\underline{\tilde{\mu}}}
\newcommandx{\softmu}[1][1=]{\mu_{#1,\mathrm{s}}}
\def\minsoftmu{\mu_{\mathrm{s}}}
\def\minminsoftmu{\underline{\mu_{\mathrm{s}}}}
\def\regsoftmusglobal{\tilde{\mu}_{\temp}}
\newcommandx{\regsoftmu}[1][1=]{\tilde{\mu}_{#1,\temp}}
\def\minregsoftmu{\tilde{\mu}_{\temp}}
\def\minminregsoftmu{\underline{\tilde{\mu}}_{\temp}}

\def\softbias{\beta_{\operatorname{s}}}
\def\regsoftbias{\tilde{\beta}_{\temp}}

\def\softfourmoment{L_{1,\operatorname{s}}}
\def\regsoftfourmoment{\tilde{L}_{1,\temp}}

\newcommand{\softsampledist}[2]{\nu_{#1}(#2)}

\def\softthirddbound{L_{3,\operatorname{s}}}
\def\regsoftthirddbound{\tilde{L}_{3,\temp}}

\def\param{\theta}
\newcommandx{\globparam}[1]{\param^{#1}}
\newcommandx{\avgparam}[1]{\bar{\param}^{#1}}
\newcommandx{\locparam}[2]{\param^{#2}_{#1}}

\newcommandx{\randState}[2][2=]{Z_{#1}^{#2}}
\newcommandx{\globfiltr}[1]{\mcF^{#1}}
\newcommandx{\locfiltr}[2]{\mcF_{#2}^{#1}}

\newcommand{\cmark}{\ding{51}}%
\newcommand{\xmark}{\ding{55}}%
\newcommand{\yes}{\cmark}
\newcommand{\no}{\xmark}

\newcommand{\matreste}[2]{\mathrm{D}_{#1}^{#2}}

\newcommand{\thirddbound}{L_3}
\newcommand{\regthirddbound}{L_{3,\mathtt{r}}}

\newcommandx{\moments}[2]{\ensuremath{\sigma_{#1}^{#2}}}

\newcommandx{\softmoments}[2]{\ensuremath{\sigma_{ #1,\operatorname{s}}^{#2}}}

\newcommandx{\regsoftmoments}[2]{\ensuremath{\tilde{\sigma}_{ #1, \temp}^{#2}}}

\newcommand{\grb}[2]{\mathrm{g}_{#1}^{#2}}
\newcommand{\egrb}[1]{\mathrm{g}_{#1}}

\newcommand{\ngb}[2]{\mathrm{g}_{#1}(#2)}
\newcommand{\ngbs}[3]{\mathrm{g}_{#1}^{#3}(#2)}

\newcommand{\softgrb}[2]{\mathrm{g}_{#1, \mathrm{s}}^{#2}}
\newcommand{\softegrb}[1]{\mathrm{g}_{#1,\mathrm{s}}}

\newcommand{\regsoftgrb}[2]{\tilde{\mathrm{g}}_{#1, \temp}^{ #2}}
\newcommand{\regsoftegrb}[1]{\tilde{\mathrm{g}}_{#1, \temp}}

\def\fw{F}
\newcommandx{\nfw}[1]{f_{#1}} %
\newcommandx{\nfsw}[2]{f^{#2}_{#1}} %
\newcommandx{\f}[1]{F(#1)} %
\newcommandx{\nf}[2]{f_{#1}(#2)} %
\newcommandx{\nfs}[3]{f^{#3}_{#1}(#2)} %

\def\gfw{\nabla F} 
\newcommandx{\ngfw}[1]{\nabla \nfw{#1}}
\newcommandx{\ngfsw}[2]{\nabla \nfsw{#1}{#2}}

\newcommandx{\gf}[1]{\gfw(#1)}
\newcommandx{\ngf}[2]{\ngfw{#1}(#2)}
\newcommandx{\ngfs}[3]{\nabla \nfs{#1}{#2}{#3}}

\newcommandx{\nhfw}[1]{\nabla^2 \nfw{#1}}
\newcommandx{\nhfsw}[2]{\nabla^2 \nfsw{#1}{#2}}

\newcommandx{\hf}[1]{\nabla^2 \f{#1}}
\newcommandx{\nhf}[2]{\nabla^2 \nf{#1}{#2}}
\newcommandx{\nhfs}[3]{\nabla^2 \nfs{#1}{#2}{#3}}

\newcommandx{\hhf}[1]{\nabla^3 \f{#1}}
\newcommandx{\nhhf}[2]{\nabla^3 \nf{#1}{#2}}
\newcommandx{\nhhfs}[3]{\nabla^3 \nfs{#1}{#2}{#3}}

\newcommand{\sampledist}[2]{\xi_{#1}(#2)}

\def\traj{z}
\def\vartraj{\mathfrak{T}}
\newcommandx{\disttraj}[1]{q^{#1}}

\newcommandx{\varstate}[2]{S_{#1}^{#2}}
\newcommandx{\varaction}[2]{A_{#1}^{#2}}
\newcommandx{\bvarstate}[2]{\bar{S}_{#1}^{#2}}
\newcommandx{\bvarstatebis}[2]{\bar{W}_{#1}^{#2}}
\newcommandx{\bvaraction}[2]{\bar{A}_{#1}^{#2}}

\newcommandx{\realstate}[2]{s_{#1}^{#2}}
\newcommandx{\realaction}[2]{a_{#1}^{#2}}

\newcommandx{\varlochistory}[2]{\mathcal{H}_{#1}^{#2}}
\newcommandx{\varglobhistory}[1]{\mathcal{H}^{#1}}
\newcommandx{\setlochistory}[1]{\mathbb{H}_{\mathrm{loc}}^{#1}}
\newcommandx{\setglobhistory}[2]{\mathbb{H}_{#1}^{#2}}

\newcommand{\entr}{\mathcal{H}}

\def\sigmafour{\sigma_4}
\def\boundgrad{L_1}
\def\secboundgrad{L_2}

\def\softboundgrad{L_{1,\mathrm{s}}}
\def\regsoftboundgrad{\tilde{L}_{1,\temp}}

\def\firsttermmoment{\mathrm{u}_{\mathrm{r}}(Z_{c,1})}
\def\secondtermmoment{\mathrm{u}_{\temp}(Z_{c,1})}
\def\efirsttermmoment{\mathrm{u}_{\mathrm{r}}}
\def\esecondtermmoment{\mathrm{u}_{\temp}}

\newcommandx{\zerotermmoment}[1]{\mathrm{u}(Z_{c, #1})}
\def\bzerotermmoment{\mathrm{u}(Z_{c,b})}
\def\dzerotermmoment{\mathrm{u}}
\def\azerotermmoment{\bar{\mathrm{u}}}
\def\ezerotermmoment{\bar{\mathrm{u}}}

\newcommandx{\seq}[1]{\delta^{#1}}
\newcommand{\biasthm}{B}

\def\hidleq{\lesssim}

\def\logitspace{\rset^{|\S| \times |\A|}}

\def\sthreshold{\tau_{\temp}}

%% file: AISTATS/src/introduction_new.tex
\begin{table*}
\caption{Comparison with prior work in the setting of agents with heterogeneous dynamics. Our results are the first to prove global convergence of $\fedPG$ to a near-optimal policy.}
\label{ta:comparaison}
  \centering
\resizebox{\textwidth}{!}{
  \begin{threeparttable}
\renewcommand{\arraystretch}{1.3}
    \label{tab:comparaison_table}
    \begin{tabular}{>{\centering\arraybackslash}c>{\centering\arraybackslash}c>{\centering\arraybackslash}c> 
    {\centering\arraybackslash}c> 
    {\centering\arraybackslash}c} 
     \toprule
      \multicolumn{1}{c}{Algorithm$^{*}$}  
     & \multicolumn{1}{c}{Global convergence}
     & \multicolumn{1}{c}{Last iterate}
     & \multicolumn{1}{c}{Communication Complexity$^{**}$}
     &
     \multicolumn{1}{c}{Sample Complexity$^{**}$}
     \\ \midrule 
    \texttt{PAvg}~(\citealt{jin2022federated}) &  \no & \no &\no & \no \\
     \texttt{FEDSVRPG-M}~(\citealt{wang2024momentum})   & \no &\no &$\mathcal{O}\left(1/\epsilon\right)$ & $\mathcal{O}\left(1/(\nagent \epsilon^{3/2})\right)$\\
     \texttt{FEDHAPG-M}~(\citealt{wang2024momentum})  & \no &\no &$\mathcal{O}\left(1/\epsilon\right)$  & $\mathcal{O}\left(1/(\nagent \epsilon^{3/2})\right)$
     \\
    \midrule
    \rowcolor{lightgray!50}  $\SoftfedPG$~(our work) & \yes &\yes  &$\mathcal{O}\left(1/\epsilon\right)$ & $\mathcal{O}\left(1/(\nagent \epsilon^3)\right)$
\\
    \rowcolor{lightgray!50}  $\RegSoftfedPG$~(our work) & \yes &\yes  &$\mathcal{O}\left(\log(1/\epsilon)\right)$& $\mathcal{O}\left(\log(1/\epsilon)/(\nagent \epsilon)\right)$
     \\ 
         \bottomrule
    \end{tabular}
\begin{tablenotes}
   \item[*] \cready{Note that all methods optimize the unregularized objective \eqref{def:objective_regularised}, except $\RegSoftfedPG$, which optimizes the entropy-regularized objective \eqref{def:objective_unregularised}}; $^{**}$ \cready{For methods that do not enjoy global convergence, the reported sample and communication complexities correspond to finding an $\epsilon$-stationary point of the objective $F$, \ie, a parameter $\param$ such that $
\PE\!\left[\|\nabla F(\param)\|^{2}\right] \le \epsilon $.
In contrast, for our methods $\SoftfedPG$ and $\RegSoftfedPG$, the complexities are stated for obtaining an $\epsilon$-optimal solution, \ie a $\param$ such that $
\PE\!\left[F^{\star}-F(\param)\right] \le \epsilon$.
These guarantees hold for any target accuracy $\epsilon > \epsilon_{\min}$, where $\epsilon_{\min}$ is the heterogeneity floor (equal to $0$ in the homogeneous case); see
\Cref{thm:complexity-general-softfedpg-main,thm:complexity-regsoftmaxfedpg_main}
for the exact expressions.}
  \end{tablenotes}
  \end{threeparttable}}
  \end{table*}

In Federated Reinforcement Learning (FRL), multiple agents learn collaboratively by interacting with their own independent environments.
As Reinforcement learning (RL) is known to be highly data-hungry \citep{akkaya2019solving}, and generating such training data is very time-consuming \citep{nair2015massively},  FRL constitutes a promising framework that can dramatically reduce the number of samples each agent must collect.
Raw trajectories, \ie, state, action and reward sequences, are never exchanged. 
Instead, agents communicate intermediate computations such as policy gradients to a central server, which aggregates them to update a global policy \citep{qi2021federated,zhuo2023federated,khodadadian2022federated}.
While FRL can accelerate the training, it must overcome two important obstacles: environment heterogeneity \citep{jin2022federated} and limited communications \citep{zhu2022federated,fan2023federated}.
Although these challenges are shared with federated learning \citep{kairouz2021advances}, the solutions developed in the classical federated learning \emph{do not generally apply to FRL}.
Specifically, the convergence of \fedAVG\, applied to the FRL objective under heterogeneity and local training remains poorly understood.
Addressing these obstacles in this context is thus crucial for the large-scale deployment of FRL.

In this paper, we establish the first global convergence analysis of federated policy gradient methods with local training in heterogeneous environments. 
As such, we address a significant gap in the federated policy-gradient literature, which has essentially focused on proving convergence to first‑order stationary points of the average value \citep{wang2024momentum,jin2022federated}. 
More precisely, we analyze federated softmax policy gradient (\fedPG) with and without entropy regularization, and propose algorithmic strategies to learn stochastic stationary policies in tabular environments.

Our analysis leverages a local property of agent-specific value functions, building on recent single-agent results \citep{mei2020global}, which show that these functions satisfy a \emph{non-uniform Łojasiewicz} condition, a generalization of gradient dominance. Remarkably, these local properties do not extend to the global federated RL objective, motivating the development of a tailored theoretical framework. As part of this framework, we provide a novel convergence analysis of \fedAVG\ for non-convex objectives, which also offers new insights into the behavior of Federated Averaging under non-uniform Łojasiewicz conditions.

Finally, we show that the differences between single-agent RL and heterogeneous FRL are intrinsic to FRL: they are not artifacts of gradient methods nor of our analysis.
Specifically, we show that heterogeneity breaks classical RL properties, as the optimal \emph{common} policy can be inherently \emph{stochastic} or even \emph{non‑stationary}.
This contrasts with the single‑agent case, where there always exists an optimal deterministic and stationary policy \citep{agarwal2019reinforcement}. 

Our contributions can be summarized as follows:
\begin{itemize}[itemsep=0pt,leftmargin=*]
\item \cready{We provide a novel analysis of federated averaging for objectives that satisfy a \emph{local} \L{}ojasiewicz-type conditions, showing how client-side regularity can be leveraged even when the global federated objective fails to satisfy any \L{}ojasiewicz inequality.} 
\item We present the \emph{first} global convergence guarantees for entropy‑regularized policy gradients in heterogeneous FRL. Exploiting the non‑uniform Łojasiewicz property of the \emph{local} objective, we show that $\fedPG$ converges to near‑optimal policies and achieves linear speed‑up in the number of agents.
\item We reveal fundamental gaps between federated and classical RL, motivating our new analytical framework. A surprising observation is that, unlike in classical RL, optimal federated policies can be non-deterministic or time-varying.
\item We validate the theory on two FRL benchmarks, showing the predicted scaling of \fedPG\ with heterogeneity and robust empirical performance.
\end{itemize}
We compare to prior work in \Cref{ta:comparaison}, review related literature in \Cref{sec:related-work}, introduce the problem setting in \Cref{sec:bg}, present our main results in \Cref{sec:analysis_fedAVG} and \Cref{sec:fed-pg}, describe specific FRL properties in \Cref{sec:heteregenous_frl}, and provide experiments in \Cref{sec:expe}.

%% file: AISTATS/src/related_works_new.tex
\paragraph{Policy Gradient Methods.}
Policy gradient (PG) methods \citep{williams1992simple, sutton1999policy} are well understood in tabular, single‑agent discounted RL. For softmax policies with exact gradients, recent analyses characterize the optimization landscape of RL via Łojasiewicz-type inequalities, establishing global convergence with sublinear rates, and linear rates with entropy regularization \citep{mei2020global, zhang2020variational, xiao2022convergence, agarwal2020optimality}. Stochastic PG is subtler: early results proved convergence to first‑order stationary points \citep{zhang2021convergence, zhang2021sample, yuan2022general}, and later work clarified when deterministic and stochastic updates align to recover global guarantees \citep{mei2021understanding,ding2025beyond,wang2026modelfree,labbi2026beyond}.%

\vspace{-4pt}
\paragraph{Federated RL.}
FRL theory literature is growing fast \citep{zhuo2023federated}.
Under homogeneous dynamics, federated Q‑learning variants have been shown to reduce sample complexity \citep{salgia2024sample,zheng2025federated,jin2022federated}. With heterogeneous dynamics, non‑asymptotic analyses reveal inherent trade‑offs: collaboration gives speedups, but with an unavoidable bias scaling with heterogeneity \citep{wang2024convergence, zhang2024finite, labbi2025federated,mangold2025convergence}. A notable exception is for federated policy evaluation, where such bias can be mitigated using control variate-type methods \cite{mangold2024scafflsatamingheterogeneityfederated}. Variants of federated PG have been analyzed in homogeneous settings, with global convergence and improved communication \citep{lan2023improved, ganesh2024global,wang2024momentum}, but rely on strong assumptions.
In contrast, we derive global convergence guarantees for federated PG with heterogeneous agents, stochastic gradients, and classical RL assumptions.

\vspace{-4pt}
\paragraph{\fedAVG\, under PL.}
Non-convex federated optimization under Polyak–Łojasiewicz (PL) conditions has received limited attention. \citet{haddadpour2019convergence} proved \fedAVG's convergence with deterministic gradients and controlled gradient diversity, while \citet{haddadpour2019local} proved linear speedups in the stochastic setting. A more recent analysis \citep{demidovich2025methods} still requires the PL condition on the global objective.
Unfortunately, such a PL condition does not hold for the global objective with heterogeneous agents.
By contrast, our analysis relies on \emph{local} client objectives' regularity, showing that local PL‑like structure can be leveraged even when the \emph{global} objective fails to satisfy a Łojasiewicz‑type property.

%% file: AISTATS/src/background.tex
\label{sec:fedrl}
\paragraph{Problem Setting.}%
We consider a FRL setting with $\nagent$ agents, where each agent $c\in[M]$ has its own independent Markov Decision Process (MDP), $\cM_c \eqdef (\S, \A, \discount, \kerMDP[c],\rewardMDP[c], \rho)$, with a state space $\S$, an action space $\A$, and discount factor $\discount < 1$, but distinct rewards $\rewardMDP[c]$ and transition kernels $\kerMDP[c]$. 
Following common practice in FRL, we define kernel and reward heterogeneity as
\begin{align}
\label{def:hg_kernel}
\!\!\!\!\hgkernel &\eqdef \max_{s,a \in \S \times \A }\max_{c,c' \in [\nagent]^{\otimes 2}}\norm{\kerMDP[c](\cdot | s,a) - \kerMDP[c'](\cdot | s,a)}[1]
\eqsp,
\\[-0.2em]
\label{def:hg_reward}
\hgreward &\eqdef \max_{c,c' \in [\nagent]^{\otimes 2}}\norm{ \rewardMDP[c] - \rewardMDP[c']}[\infty]
\eqsp.
\end{align}
We consider policies with \emph{softmax} parameterization 
\begin{equation}
\label{eq:softmax_parameterization}
\!\! \policy_\theta(a|s) \!\eqdef\! \frac{\exp(\theta(s,a))}{\sum_{a' \in \A} \exp(\theta(s,a'))} 
 \text{ for } \theta \in \logitspace
 ~,
\end{equation}
and aim to minimize the averaged objective
\begin{align}
\label{def:objective_unregularised} 
    \max_{ \param \in \logitspace } \frlobjective(\param)  &\textstyle \eqdef \frac{1}{\nagent} \sum_{c=1}^\nagent \frllocfunc[c](\param) \eqsp,
\\
\label{def:local_objective_unregularised} 
\text{where } \frllocfunc[c](\param) \! &\textstyle\eqdef \PE_{c,\initdist}^{\policy_{\param}} \left[ \sum_{t=0}^{\infty} \discount^t \rewardMDP[c](\varstate{c}{t},\varaction{c}{t}) \right]\eqsp,
\end{align}
and where $\PE_{c,\initdist}^{\policy_{\param}}[\cdot]$ is the expectation over random trajectories generated by following the softmax policy $\policy_{\param}$ parametrized by $\param$: the initial state is sampled from a distribution $\varstate{c}{0} \sim \rho(\cdot)$ and for all $ t \geq 0:$  $\varaction{c}{t} \sim \policy_{\param}(\cdot |\varstate{c}{t}), \varstate{c}{t+1} \sim \kerMDP[c](\cdot | \varstate{c}{t},\varaction{c}{t})$.
We define 
\begin{align}
\optfrllocfunc[c] \eqdef \sup_{\theta \in \logitspace}  \frllocfunc[c](\theta)
\eqsp,
\quad 
\optfrlobjective \eqdef \textstyle{\frac{1}{\nagent} \sum_{c=1}^{\nagent}} \optfrllocfunc[c]
\eqsp,
\end{align}
the maximum value and its average over all agents.
Similarly, we define the entropy-regularized FRL objective, for a temperature 
$\temp > 0$, as %
\[
\label{def:objective_regularised} 
 \max_{ \param \in \logitspace } \regfrlobjective(\param) 
 \textstyle
 \eqdef \frac{1}{\nagent} \sum_{c=1}^\nagent \regfrllocfunc[c](\param) \eqsp, 
\]
with  $\regfrllocfunc[c](\param) \!\eqdef \PE_{c,\initdist}^{\policy_{\param}} \big[ \textstyle{\sum_{t=0}^{\infty}} \discount^t \!\!\left(\rewardMDP[c](\varstate{c}{t},\varaction{c}{t}) \!-\! \lambda h_\theta(\varaction{c}{t}, \varstate{c}{t}) \big] \right)
$ and $h_\theta(\varaction{c}{t}, \varstate{c}{t}) \eqdef \log(\policy_\theta(\varaction{c}{t} | \varstate{c}{t}))$.
Finally, let
\begin{equation}
    \optregfrllocfunc[c] \eqdef \smash{\sup_{\theta \in \logitspace} }\regfrllocfunc[c](\theta), 
    \quad \optregfrlobjective \eqdef \textstyle{\frac{1}{\nagent}\sum_{c=1}^\nagent } \optregfrllocfunc[c]
    \eqsp,
\end{equation} 
be the maximum value and its average over agents.

\paragraph{Single-Agent Regularity.}
Taking $\nagent = 1$, we have $\optfrlobjective = \optfrllocfunc[1]$ and $\optregfrlobjective = \optregfrllocfunc[1]$.
In this setting, for a stationary policy $\policy$, we define the discounted state occupancy
\begin{gather}
\label{def:occupancy_measure}
\textstyle
\occupancy[\initdist][\policy](s) \textstyle \eqdef (1- \discount) \sum_{t=0}^{\infty} \discount^t \cdot \initdist \kerMDP[\policy]^t(s) \eqsp, 
\end{gather}
where $\kerMDP[\policy](s'|s) \eqdef \sum_{a\in \A} \policy(a|s) \kerMDP[](s'| s, a)$. We also assume the initial state distribution $\initdist$ has strictly positive coefficients (see \Cref{assum:sufficent_exploration}).

\textit{(Unregularized Case).}
Under these assumptions, in the single-agent setting, \citet{mei2020global} proved that $\frllocfunc$ is smooth with constant $L_{2,s}\eqdef8/(1-\discount)^3$ (see its Lemma 7). 
Moreover, $\frllocfunc$ satisfies a (non‑uniform) \emph{Łojasiewicz inequality} (Lemma 8 in \citealt{mei2020global}) 
{\begin{equation}
\label{eq:nonuniform-lojasiewicz-non-regularized}
\|\nabla \frllocfunc(\theta)\|_2^2 \;\ge\; 2\,\mu_s(\theta)\,\big[\optfrllocfunc - \frllocfunc(\theta)\big]^2
\text{ for } \forall \theta \in \Theta,
\end{equation}
with $\smash{\mu_{\mathrm{s}}(\theta) \eqdef  \min_{s} \policy_{\theta}(a^{\star}(s)|s)^2 \cdot \left\|  \occupancy[\initdist][\policy^{\star}]/ \occupancy[\initdist][\policy_{\theta}] \right\|_{\infty}^{-2}/(2\nstates)}$.}

\textit{(Regularized Case).} 
Under the above assumptions, in the single-agent setting, \citet{mei2020global} proved that $\regfrlobjective$ is $\regsoftsmooth$-smooth with $\regsoftsmooth \eqdef (8 + \temp (4+ 8\log(\nactions))/(1-\gamma)^3$. 
Moreover, the regularized objective satisfies a stronger non‑uniform Łojasiewicz inequality \citep{williams1991function, mnih2016asynchronous, schulman2017equivalence, ahmed2019understanding}, with a \emph{linear} suboptimality gap,
\begin{equation}
\label{eq:nonuniform-lojasiewicz-regularized}
\norm{\nabla \regfrlobjective(\theta)}[2]^2 \geq 2 \regsoftmusglobal(\theta) \left[ \regfrlobjective^{\star} - \regfrlobjective(\theta)\right]
\eqsp,
\end{equation}
for $\theta \in \logitspace$, %
with 
\begin{align*}
\regsoftmusglobal(\theta)
    \eqdef  \frac{\temp \min_{s} \occupancy[\initdist][ \policy_{\theta}](s)\min_{s,a} \policy_{\theta}(a|s)^2 }{\nstates(1- \discount)}  \Big\|\frac{\occupancy[\initdist][ \policy_{\temp}^{\star}]} {\occupancy[\initdist][\policy_{\theta}]} \Big\|_{\infty}^{-1} \eqsp.
\end{align*}    

\paragraph{FRL (Non-)Regularity.}
From the single‑agent case,  for any agent $c\in [\nagent]$, $\frllocfunc[c]$, and $\regfrllocfunc[c]$ are smooth (hence $\frlobjective$ and $\regfrlobjective$ are smooth). 
Furthermore, we show in \Cref{assum:sufficent_exploration} that for any $c \in [\nagent]$, local functions  $\frllocfunc[c]$ and $\regfrllocfunc[c]$ satisfy non‑uniform Łojasiewicz inequalities. 
Unfortunately, averaging such functions \emph{does not preserve} such non‑uniform Łojasiewicz property in general; the federated objective $\frlobjective$ (or $\regfrlobjective$), even when each client objective enjoys single‑agent geometry (see \Cref{lem:counterexample_pl_regularized} where we provide a counter example).
Consequently, analyses of federated averaging methods leveraging global Łojasiewicz conditions (e.g., \citealt{demidovich2025methods}) cannot be directly applied in our context. We remedy this problem by proposing a novel analytical framework for FedAVG‑style methods, tailored to global objectives that are sums of locally Łojasiewicz functions, as arise in heterogeneous FRL.

%% file: AISTATS/src/FedAVG.tex
\paragraph{Federated Averaging.}
In this section, we provide convergence bounds on a general class of distributed non-convex optimization problems of the form
\begin{align}
\label{eq:fedavg_problem_main}
\!\!\!\!\max_{\param \in \rset^d}
\f{\param} \eqdef \frac{1}{\nagent} \sum_{c=1}^{\nagent} \nf{c}{\param} ,
\eqsp
\nf{c}{\param} \eqdef \PE_{\randState{c}}[\nfs{c}{\param}{\randState{c}}],
\end{align}
where each $\randState{c}$ is a random variable sampled from a distribution $\sampledist{c}{\param}$, which may depend on~$\param$, and takes values in a measurable set $(\msZ,\mathcal{Z})$, and where the function $(z, \param) \mapsto \nfs{c}{\param}{z}$ are measurable.
Each function $\nfw{c}$ is only available to the client $c$ through \emph{biased} stochastic gradients $\grb{c}{z}(\param)$, whose expected value is
\begin{align}
\label{eqdef:expected_sto_estimator_main}
\egrb{c}(\param) \eqdef  \PE_{\randState{c} \sim \sampledist{c}{\param}}[  \grb{c}{\randState{c}}(\param) ] \eqsp,
\end{align}
but is typically different from the gradient of $\nfw{c}$. To solve \eqref{eq:fedavg_problem_main}, we use $\projfedAVG$, a variant of federated averaging with a projection-like \emph{improvement operator} (see \Cref{algo:FEDAVG}). Each communication round of $\projfedAVG$ involves the central server distributing global parameters $\globparam{r}$ to all agents. Subsequently, each agent  performs $\lsteps$ stochastic gradient ascent steps on its local objective:
\begin{equation}\label{eq:def_local_step_fedavg}
    \locparam{c}{r,h+1} = \locparam{c}{r,h} + \eta \cdot \grb{c}{\randState{c}[r,h+1]}(\locparam{c}{r,h})\,, \quad \locparam{c}{r,0} = \globparam{r}
\end{equation}
where $\eta > 0$ is a learning rate shared by the agents, and the $ \randState{{c}}[{r},{h}]$ for $\smash{{c} \in [\nagent]}$, $\smash{{r} \in [\trounds]}$, and $\smash{{h} \in [\lsteps]}$ are independent from an agent to another, and are a martingale with respect to the filtration
\begin{align*}
\globfiltr{r} 
\eqdef 
\sigma(\randState{c}[r',h'] : r' < r, h' \in \{0, \dots, \lsteps\}, c' \in [\nagent])
\eqsp.
\end{align*} %
After $\lsteps$ local steps, parameters are averaged as $\bar{\theta}^{r+1} = \frac{1}{\nagent} \sum_{c = 1}^{\nagent} \locparam{c}{r,H}$, and followed by a projection-like step $\globparam{r+1} = \projset(\bar{\theta}^{r+1})$, where $\mathcal{T} \colon \rset^d \rightarrow \rset^d$.

\begin{algorithm}[t]
\caption{\projfedAVG}
\label{algo:FEDAVG}
\begin{algorithmic}
\STATE \textbf{Initialization:} Learning rate $\step > 0$,  parameter $\globparam{0}$, Improvement operator $\mathcal{T}$
\FOR{$r = 0$ to $\trounds - 1$}

    \FOR{$c = 1$ to $\nagent$}       
        \STATE Set $\locparam{c}{r,0} = \globparam{r}$.

        \FOR{$h = 0$ to $\lsteps - 1$}
            \STATE Receive random state $\randState{c}[r,h+1]$
            \STATE Update $\locparam{c}{r,h+1} = \locparam{c}{r,h} + \step \grb{c}{\randState{c}[r,h+1]}(\locparam{c}{r,h})$
        \ENDFOR 
    \ENDFOR

    \STATE Server updates parameter:  $ \globparam{r+1} = \mathcal{T}(\avgparam{r+1} ) $ where $ \avgparam{r+1} = \frac{1}{\nagent} \sum_{c=1}^{\nagent} \locparam{c}{r, \lsteps} $

\ENDFOR
\end{algorithmic}
\end{algorithm}

\paragraph{Descent Lemma.}
We start by establishing a descent lemma, under the following assumption.
We stress that this assumption is not restrictive, as we will show in \Cref{sec:fed-pg} that this assumption is satisfied by PG methods under a standard RL assumption.
\begin{assum}
\label{assum:gen-f-reg}
For all $c \in [\nagent]$, $f_c$ is three times differentiable, and there exists $\boundgrad, \secboundgrad, \thirddbound, \hgty, \bias, \moments{2}{2}, \moments{4}{4} \ge 0$ such that 
\begin{enumerate}[leftmargin=*, labelsep=0.5em,itemsep=-0.2em, topsep=-0.2em]
\item \textit{Smoothness}: for any $(\param,u) \in \rset^d \times \rset^d$,
\begin{gather*}
\norm{\nabla f_c(\param)}\leq \boundgrad
\eqsp, \quad 
\norm{\nabla^2 f_c(\param) u} \leq \secboundgrad \norm{u}
\eqsp,
\\
\norm{\nabla^3 f_c(\param) u^{\otimes 2}} \leq \thirddbound \norm{u}^2
\eqsp,
\end{gather*}
and $\grb{c}{}(\theta) = \PE_{Z_c \sim \nu_c(\theta)}[\grb{c}{Z_c}(\theta)]$ is $\secboundgrad$-Lipschitz. 
\item  \textit{Heterogeneity}:  $\norm{\nabla F(\theta) -  \nabla f_c(\theta)}[2] \leq \hgty$ 
\item \textit{Bias and variance}: for any $\theta \in \rset^d$,\vspace{-4pt}
\begin{align*}
\| \nabla f_c(\theta) - \grb{c}{}(\theta)\|_2 &\leq \bias\,,
\\
\PE_{\randState{c} \sim \softsampledist{c}{\param}} [\norm{ \grb{c}{\randState{c}}(\param) \!-\! \grb{c}{}(\param)}[2]^p ] &\leq \moments{p}{p}
~,
\text{ for } p \in \{2,4\}
\eqsp.
\end{align*}
\end{enumerate}
\end{assum}
\cready{Assumption \cref{assum:gen-f-reg} captures standard regularity conditions on the local objectives, such as smoothness, bounded gradient bias and variance, and bounded gradient heterogeneity, which are commonly used in federated learning. Under this assumption, we derive the following descent lemma for non-convex objectives.} 
\begin{lemma} 
\label{lem:ascent_lemma_main_text}
Assume \Cref{assum:gen-f-reg}. 
Then, for any $\step >0$ such that $\step \lsteps \secboundgrad \leq 1/6$ and $32 \step^2 \lsteps^2 \thirddbound^2 \boundgrad^2 \le \secboundgrad^2$, the iterates of $\projfedAVG$ satisfy
\begin{align*}
\nonumber
&  \f{\globparam{r}} \!-\! \CPE{ \f{\avgparam{r+1}} }{\globfiltr{r}}
\hidleq
- \tfrac{\step \lsteps}{4}\norm{ \gf{\globparam{r}}}[2]^2
+ \tfrac{\step^2 \secboundgrad \lsteps \moments{2}{2}}{\nagent}
\\
& \qquad 
\qquad\qquad 
\quad 
+ \step \lsteps \bias^2
+ \step^3 \secboundgrad^2 \lsteps^3 \hgty^2
+ \step^5 \thirddbound^2 \lsteps^3 \moments{4}{4} 
\eqsp.
\end{align*}
\end{lemma}
\textbf{\emph{Sketch of proof.}}
Let $\kappa = \frac{1}{\sqrt{\step \lsteps}} $. Using the $\secboundgrad$-smoothness of each $f_c$,  taking the expectation conditionally on $\globfiltr{r}$ and using the identity $2\pscal{a}{b} = \norm{a}[2]^2 + \norm{b}[2]^2 - \norm{a-b}[2]^2  $ for $a, b \in \rset^d$, we get
\begin{align*}
\nonumber
&\textstyle - \CPE{ \f{\avgparam{r+1}} }{\globfiltr{r}} + \f{\globparam{r}} \leq 
- \tfrac{1}{2 \kappa^2}\norm{ \gf{\globparam{r}}}[2]^2
\\
&\textstyle + \underbrace{\tfrac{1}{2 \kappa^2} \norm{\gf{\globparam{r}} + \kappa^2 \CPE[]{ \globparam{r} - \avgparam{r+1} }{\globfiltr{r}} }[2]^2  }_{\term{A}} 
\\ & \textstyle 
+ \underbrace{
\tfrac{\secboundgrad}{2} \CPE[]{\norm{\avgparam{r+1} - \globparam{r}}[2]^2}{\globfiltr{r}}
- \tfrac{\kappa^2}{2} \norm{ \CPE[]{\avgparam{r+1} - \globparam{r} }{\globfiltr{r}} }[2]^2}_{\term{B}}.
\end{align*}
The term $\term{A}$ is a drift term, that is due to local updates, and is due to heterogeneity, while the term $\term{B}$ is a second-order term error term and a variance term.
We now bound each of these two terms.

\textbf{Bounding $\term{A}$.} Using Jensen's inequality, combined with Young's inequality and the bound on the bias of the stochastic estimator (\cref{assum:gen-f-reg}), we get
\begin{align*}
\textstyle
\!\term{A}\!\leq\! \! 
\tfrac{2}{\lsteps \nagent}\!\sum\limits_{c=1}^{\nagent}  \!\sum\limits_{h=0}^{\lsteps-1}\! \norm{  \CPE{   \ngf{c}{\globparam{r}}\! -\! \ngf{c}{\locparam{c}{r,h}}  }{\globfiltr{r}} }[2]^2 
\!+\!2 \bias^2.
\end{align*}
The term on the right-hand side captures the expected drift induced by local updates under client heterogeneity. We bound it via a third-order Taylor expansion combined with Burkholder's inequality (Theorem~8.6 in \citealp{oskekowski2012sharp}); see \Cref{lem:bound-expec-drift} for details, which yields
\begin{align*}
& \textstyle\term{A}
\lesssim
\tfrac{\step^3 \secboundgrad^2 \lsteps^2 (\lsteps-1)}{\nagent} \sum_{c=1}^\nagent \norm{ \ngf{c}{\globparam{r}} }[2]^2
\\ 
&\textstyle 
+  \thirddbound^2 \step^5 \lsteps^2(\lsteps-1) \moments{4}{4} 
+ (1 + \step^2 \secboundgrad^2 \lsteps(\lsteps-1)) \step \lsteps \bias^2
\eqsp.
\end{align*}

\textbf{Bounding $\term{B}$.} We decompose $\term{B}$ by writing $\avgparam{r+1} = \CPE[]{ \avgparam{r+1} }{\globfiltr{r}} + \avgparam{r+1} - \CPE[]{ \avgparam{r+1} }{\globfiltr{r}}$, which gives
\begin{align*}
\textstyle \term{B}
& \textstyle =
\tfrac{\secboundgrad}{2}  \CPE[]{ \norm{\CPE[]{ \avgparam{r+1} }{\globfiltr{r}} - \avgparam{r+1} }^2 }{\globfiltr{r}}
\\ \textstyle
& \textstyle+
\Big( \tfrac{\secboundgrad}{2} - \tfrac{\kappa^2}{2} \Big)
\norm{ \CPE[]{\avgparam{r+1} - \globparam{r} }{\globfiltr{r}} }[2]^2 
\eqsp.
\end{align*}
Since $\step \lsteps \secboundgrad \le 1$, we have $\frac{\secboundgrad}{2}-  \frac{\kappa^2}{2} \le 0$, and the second term is negative.
The second term is a variance term, that we bound using \Cref{lem:globparam-variance}, which gives
\begin{align*}
\textstyle
\term{B} \le 3 \step^2 \secboundgrad \lsteps \moments{2}{2}/2 \nagent
\eqsp.
\end{align*}
Combining $\frac{1}{\nagent} \sum_{c=1}^\nagent \norm{ \ngf{c}{\globparam{r}}}^2 \!\le\! \norm{ \gf{\globparam{r}} }^2 + \hgty^2$, with the bounds on $\!\term{A}\!$ and $\!\term{B}$,\! concludes the proof
 \hfill$\square$\par
We provide a full statement and proof of this lemma in \Cref{lem:ascent_lemma}.
This result generalizes Theorem 4.2 of \citet{glasgow2022sharp} to biased oracles and relaxes the previously required pointwise third-order smoothness to only in-expectation third-order smoothness.
In the following, we will use the following quantities
\begin{align}
\label{eq:gen-f-def-fstar-and-avg}
\textstyle \flstar[c] \eqdef \sup_{\param \in \rset^d} \nf{c}{\param} \eqsp, \quad \Fstar \eqdef \frac{1}{\nagent} \sum_{c=1}^{\nagent} \flstar[c] \eqsp.
\end{align}
We now apply \Cref{lem:ascent_lemma_main_text} to two cases where the local functions satisfy a type of non-uniform Łojasiewicz condition. Using this approach, we will derive global convergence bounds, thereby generalizing the analyses of \cite{glasgow2022sharp} and \cite{demidovich2025methods}.

\paragraph{Quadratic Łojasiewicz (QL) inequality.} 
For any $(c, \param) \in [\nagent] \times \rset^{d}$, we define the non-uniform "quadratic-type" Łojasiewicz constant by 
\[
\mufl[c](\param) := \sup\{x \in \rset^{+}, \norm{\ngf{c}{\param}}[2]^2 \geq 2 x \left(\flstar[c] - \nf{c}{\param}\right)^2\}.
\]
We first prove the convergence of \projfedAVG\, under \Cref{assum:gen-f-reg} and the following two assumptions.
\begin{assumQL}
\label{assum:LL1}
For any $(c, \param) \in [\nagent] \times \rset^{d}$, $\mufl[c](\theta)>0$. 
\end{assumQL}
\begin{assumQL}
\label{assum:LL2}
There exists $\minmufl>0$, such that for any $\param \in \rset^{d}$, we have $\min_{c\in [\nagent]}\mufl[c](\projset(\param)) \geq \minmufl$ and 
$\f{\projset(\param)} \geq \f{\param}$.
\end{assumQL}
\cready{Assumption~\Cref{assum:LL1} characterizes the landscape of each local objective by imposing a weaker, \emph{quadratic} \L{}ojasiewicz inequality, as opposed to the standard Polyak--\L{}ojasiewicz condition~\cite{POLYAK1963864}. Such structure arises in unregularized FRL; see~\eqref{eq:nonuniform-lojasiewicz-non-regularized}. Assumption~\Cref{assum:LL2} ensures that applying $\projset$ increases the value of the objective and keeps the iterates away from regions where the objective is ill-conditioned.}
\begin{theorem}
[Convergence rates of $\projfedAVG$] 
\label{theorem:convergence_pl_type_square_main}
Assume \Cref{assum:gen-f-reg}, \Cref{assum:LL1} and \Cref{assum:LL2}. For any $\step >0$ such that $\step \lsteps \secboundgrad \leq 1/6$ and $32 \step^2 \lsteps^2 \thirddbound^2 \boundgrad^2 \le \secboundgrad^2$, the iterates of $\projfedAVG$ satisfy
\begin{align*}
&\PE[\Fstar - \f{\globparam{\trounds}}] 
\lesssim
\tfrac{\Fstar- \f{\globparam{0}}}{1+  \trounds \cdot (\Fstar - \f{\globparam{0}}) \step \lsteps \minmufl} 
 +\tfrac{\hgty + \bias}{\sqrt{\minmufl}}
  + \tfrac{\hgty^2 + \bias^2}{\secboundgrad}
 \\
& 
\qquad \qquad 
+
\sqrt{\smash{\tfrac{\step \secboundgrad  \moments{2}{2}}{\nagent\minmufl}} \vphantom{\frac{U}{u}}}
+ \tfrac{\step \moments{2}{2}}{\nagent}
 + \tfrac{\step^2 \thirddbound \lsteps \moments{4}{2}}{\sqrt{\minmufl} }  
+ \tfrac{\step^4 \thirddbound^2 \lsteps^2 \moments{4}{4}}{\secboundgrad}
\eqsp.
\end{align*}
\end{theorem}
\textbf{\emph{Sketch of proof.}} Firstly, using \Cref{assum:LL1}, we have for any $c \in [\nagent]$ and $\param \in \logitspace$
\begin{align*}
\textstyle
\sqrt{ \min_{c \in [\nagent]}2\mufl[c](\theta)} \left[ \flstar[c]- \nf{c}{\theta}\right] &\leq 
\norm{\ngf{c}{\theta}}[2] 
 \eqsp.
 \end{align*}
We then decompose $\ngf{c}{\theta} = \ngf{c}{\theta} - \nabla \f{\theta} + \nabla \f{\theta}$ and use triangle inequality and \Cref{assum:gen-f-reg} to bound
 \begin{align*}
 \textstyle
 \norm{\ngf{c}{\theta}}[2]  \leq \hgty + \norm{\nabla \f{\theta}}[2]
 \eqsp.
\end{align*}
Averaging this inequality over all the agents, taking the square, and applying Young's inequality allows to derive a quasi-QL inequality for the global objective:
\begin{align}
\textstyle
\label{eq:quasi_ql_global_main_sketch}
\hgty^{2} + \norm{\nabla \f{\theta}}[2]^{2}  \geq \min_{c\in [\nagent]}\mufl[c](\theta) (\Fstar- \f{\theta})^{2}.
\end{align}
Using \cref{assum:LL2}, we obtain $\min_{r \geq 0} \min_{c\in [\nagent]} \mufl[c](\globparam{r}) \geq \minmufl$, as $\globparam{r} =  \mathcal{T}(\avgparam{r})$.
Next, using the fact that $\f{\globparam{r+1}} \geq \f{\avgparam{r+1}}$ in \eqref{eq:quasi_ql_global_main_sketch}, plugging the result in \Cref{lem:ascent_lemma_main_text}, and
unrolling the recursion gives the result. \hfill$\square$\par
We prove this theorem in \Cref{sec:conv-loc-null}.
This theorem shows that under the non-uniform Łojasiewicz inequality, the averaged objective’s optimality gap decays sublinearly, converging to a residual floor determined by the gradient bias $\bias$, the heterogeneity $\hgty$, and stochastic errors of order $\moments{2}{2}/\nagent$, which decrease linearly with the number of agents $\nagent$ (linear speed-up). Higher-order contributions scale as $\step^2$. The homogeneous setting is recovered by setting $\hgty=0$, while unbiased gradients correspond to $\bias=0$. In the unbiased homogeneous case, the residual floor disappears. We now present the sample and communication complexity result.
\begin{corollary}[(Simplified) Sample and Communication Complexity] 
\label{thm:complexity-general-fl_ql_main}
Under the assumptions of \Cref{theorem:convergence_pl_type_square_main},
let $\textstyle \epsilon \gtrsim  \frac{\hgty}{\minmufl^{1/2}}  + \frac{\bias}{\minmufl^{1/2}}  
 + \frac{\hgty^2}{2\secboundgrad}
+ \frac{\bias^2 }{\secboundgrad}$, and $\step \le \min\Big( \frac{1}{6\secboundgrad}, \frac{\minmufl \nagent \epsilon^2}{216 \secboundgrad \moments{2}{2}}, \frac{\mu^{1/2} \epsilon \secboundgrad}{13^2 \thirddbound \moments{4}{2}},\frac{2\epsilon\nagent}{\moments{2}{2}}, \frac{\epsilon^{1/2}\secboundgrad^{3/2}}{24 \thirddbound \moments{4}{2}} \Big)$, then $\projfedAVG$ achieves $\PE[\Fstar - \f{\theta^{\trounds}}] \leq \epsilon$, with
\begin{align*} 
\textstyle \trounds \gtrsim
\tfrac{\left[\Fstar - \f{\theta^{0}}  - \epsilon\right]}{(\Fstar - \f{\theta^{0}}) \minmufl \epsilon} 
\cdot \max\big(
\secboundgrad,
\tfrac{\thirddbound \boundgrad}{\secboundgrad}
\big)
\eqsp,
\end{align*}
communications and a number of samples per agent of
\begin{align*}
\textstyle
\trounds \lsteps 
\gtrsim \tfrac{[\Fstar - \f{\theta^{0}}  - \epsilon]}{(\Fstar - \f{\theta^{0}}) \minmufl \epsilon}
\max\big( 
\secboundgrad, 
\tfrac{\secboundgrad \moments{2}{2}}{\minmufl \nagent \epsilon^2},
\tfrac{\thirddbound \moments{4}{2}}{\mu^{1/2} \epsilon \secboundgrad},
\tfrac{ \thirddbound \moments{4}{2}}{\epsilon^{1/2}\secboundgrad^{3/2}} \big).
\end{align*}
\end{corollary}
See \Cref{sec:conv-loc-null} for a proof of this corollary.
This result shows that \fedAVG\, achieves linear speedup in the number of agents, provided that desired precision is not too large.
Moreover, the number of communication rounds scales with $O(1/\epsilon)$.

\paragraph{Polyak-Łojasiewicz (PL) inequality.}
For any $(c, \param) \in [\nagent] \times \rset^{d}$, we define the non-uniform Polyak-Łojasiewicz constant by:
\[
\lmufl[c](\param) := \sup\{x \in \rset^{+}, \norm{\ngf{c}{\param}}[2]^2 \geq 2 x \left(\flstar[c] - \nf{c}{\param}\right)\}.
\]
Next, we assume the following PL-type conditions. 
\begin{assumPL}
\label{assum:PL1}
     For any $(c, \param) \in [\nagent] \times \rset^{d}$, $\lmufl[c](\param) >0$.
\end{assumPL}
\begin{assumPL}
\label{assum:PL2}
There exists $\minlmufl>0$, such that for any $\param \in \rset^{d}$, we have $\min_{c\in [\nagent]}\lmufl[c](\projset(\param)) \geq \minlmufl$ and 
$\f{\projset(\param)} \geq \f{\param}$.
\end{assumPL}
\cready{Assumption~\Cref{assum:PL1} is closer to the standard Polyak--\L{}ojasiewicz condition~\cite{POLYAK1963864}, yet it remains more general since the \L{}ojasiewicz coefficient is allowed to depend on $\param$, which can create highly ill-conditioned regions. We avoid such regions by assuming that the operator $\projset$ confines the iterates of $\projfedAVG$ to a well-conditioned set; see~\Cref{assum:PL2}.}
\begin{theorem}[Convergence rates of $\projfedAVG$] 
\label{theorem:sto_fedavg_alpha_2_main}
Assume \Cref{assum:gen-f-reg}, \Cref{assum:PL1} and \Cref{assum:PL2}. For any $\step >0$ such that $\step \lsteps \secboundgrad \leq 1/6$ and $32 \step^2 \lsteps^2 \thirddbound^2 \boundgrad^2 \le \secboundgrad^2$, the iterates of $\projfedAVG$ satisfy
\begin{align*}
\textstyle \PE \left[\Fstar - \f{\globparam{\trounds}}\right]
& \textstyle \lesssim
\big(1 -\tfrac{\step \lsteps \minlmufl }{2}\big)^{\trounds} \cdot (\Fstar - \f{\globparam{0}})
\\
& \textstyle + \tfrac{\hgty^2 + \bias^2}{\minlmufl } 
+ \tfrac{\step \secboundgrad \moments{2}{2}}{ \nagent \minlmufl } 
+ \tfrac{\step^4 \thirddbound^2 \lsteps^2 \moments{4}{4}}{\minlmufl \secboundgrad }
\eqsp.
\end{align*}
\end{theorem}
We prove this theorem in \Cref{sec:conv-loc-null}.
This theorem shows that \fedAVG\, converges faster under these assumptions, although a residual floor term remains, depending on the gradient bias and heterogeneity.
We now give a sample and complexity result.
\begin{corollary}[Sample and Communication Complexity of $\projfedAVG$] 
\label{thm:complexity-linear-pl_main}
Under the assumptions of \Cref{theorem:sto_fedavg_alpha_2_main}, let $\epsilon> {4\hgty^2}/{\minlmufl} + {16 \bias^2}/{\minlmufl}$
and $\step \le \min( \frac{1}{6\secboundgrad}, \frac{\minlmufl  \epsilon \nagent}{12 \secboundgrad \moments{2}{2}}, \frac{\minlmufl^{1/2}\secboundgrad^{3/2} \epsilon^{1/2}}{5\thirddbound \moments{4}{2}} )$,
Then $\projfedAVG$ achieves $\PE[\Fstar - \f{\theta^{\trounds}}] \leq \epsilon$, with 
\begin{align*}
\textstyle \trounds \gtrsim 
\tfrac{\secboundgrad}{ \minlmufl }
\max\big( 1, \tfrac{\thirddbound \boundgrad}{\secboundgrad^2}\big)
\log\big(\tfrac{4(\Fstar - \f{\theta^{0}})}{\epsilon}\big) 
\eqsp,
\end{align*}
communications and a number of samples per agent of
\begin{align*}
& \textstyle \trounds \lsteps 
\gtrsim
\tfrac{\secboundgrad}{\minlmufl }
\max\big( 1, \tfrac{12 \secboundgrad \moments{2}{2}}{\minlmufl  \epsilon \nagent}, \tfrac{5\thirddbound \moments{4}{2}}{\minlmufl^{1/2}\secboundgrad^{5/2} \epsilon^{1/2}} \big)
\log\big(
\tfrac{4\Delta^0}{\epsilon}
\big)\eqsp,
\end{align*}
where $\Delta^0 = \Fstar - \f{\theta^{0}}$.
\end{corollary}
We prove this corollary in \Cref{sec:conv-loc-null}.
As in the previous case, this result proves that \fedAVG\, converges with $O(\log(1/\epsilon))$ communication rounds, with linear speedup in the number of agents, up to higher-order terms.
Next, we apply these results to federated policy gradient with heterogeneous agents.

%% file: AISTATS/src/solvingfrl.tex
We introduce two federated extensions of policy gradient, $\SoftfedPG$ and $\RegSoftfedPG$, designed for \eqref{def:objective_unregularised} and \eqref{def:objective_regularised}, respectively \citep[see][]{mei2020global,agarwal2021theory}. These algorithms can be viewed as particular instances of the general $\projfedAVG$ framework (see \Cref{sec:analysis_fedAVG}).
Both $\SoftfedPG$ and $\RegSoftfedPG$ leverage a REINFORCE-like estimator \citep{williams1992simple} that uses a batch of independent $\sizebatch$ trajectories of length $\lentrunc$. 
For completeness, the pseudo-code of $\SoftfedPG$ and $\RegSoftfedPG$, are provided in \Cref{algo:FEDPG}.  Next, we check that \Cref{assum:gen-f-reg},\, \Cref{assum:LL1} and \Cref{assum:LL2} holds for $\SoftfedPG$,  and \Cref{assum:gen-f-reg},\Cref{assum:PL1} and \Cref{assum:PL2}, holds for $\RegSoftfedPG$. All proofs of subsequent results are carried in \Cref{secappendix:softfedpg} and \Cref{secappendix:regsoftfedpg}.

\subsection{Convergence Analysis of \texorpdfstring{$\SoftfedPG$}{S-FedPG}} \label{subsec:main_sfedpg}

\begin{algorithm}[t]
\caption{\texttt{(S, RS)-FedPG}}
\label{algo:FEDPG}
\begin{algorithmic}
\STATE \textbf{Initialization:} Learning rate $\step > 0$,  parameter $\globparam{0}$, improvement projector $\projset$

\FOR{$r = 0$ to $\trounds - 1$}

    \FOR{$c = 1$ to $\nagent$}       
        \STATE Set $\locparam{c}{r,0} = \globparam{r}$.

        \FOR{$h = 0$ to $\lsteps - 1$}
            \STATE Collect $\sizebatch$ trajectories of length  $\lentrunc$: $Z_{c}^{r,h+1} \eqdef (\varstate{c,b}{r,h,1:\lentrunc}, \varaction{c,b}{r,h,1:\lentrunc})_{b=1}^{\sizebatch}$ using $\policy_{\locparam{c}{r,h}}$
            \STATE Update $\locparam{c}{r,h+1} = \locparam{c}{r,h} + \step \grb{c}{Z_{c}^{r,h+1}}(\locparam{c}{r,h})$ where $ \grb{c}{Z_{c}^{r,h+1}}(\locparam{c}{r,h})$ is computed using \eqref{eq:expression_of_stochastic_gradient_softmax_fedpg} for $\SoftfedPG$, and \eqref{eq:expression_of_stochastic_gradient_regularised_softmax_fedpg_main} for $\RegSoftfedPG$.
        \ENDFOR 
    \ENDFOR

    \STATE Server updates parameter:  $ \theta^{r+1} = \projset(\bar{\theta}^{r+1} ) $ where $ \bar{\theta}^{r+1} = \frac{1}{\nagent} \sum_{c=1}^{\nagent} \theta^{r,\lsteps}_{c} $

\ENDFOR
\end{algorithmic}
\end{algorithm}
For a batch of $\sizebatch $ trajectories $z \in \left( \S \times \A\right)^{\lentrunc \cdot\sizebatch}$, the REINFORCE estimator is 
\begin{align}
\nonumber
\textstyle \softgrb{c}{z}(\param)
& \textstyle \eqdef \frac{1}{\sizebatch}  \sum_{b=1}^{\sizebatch} \sum_{t=0}^{\lentrunc-1} \discount^{t}  \\ \label{eq:expression_of_stochastic_gradient_softmax_fedpg}
& \textstyle \times \left( \sum_{\ell =0}^t \nabla \log \policy_{\theta}(a_{b}^{\ell} \mid s_{b}^{\ell}) \right) \rewardMDP[c](s_{b}^{t}, a_{b}^{t})  \eqsp.
\end{align}
Condition \Cref{assum:gen-f-reg} holds with the following constants
$\boundgrad \simeq     (1- \discount)^{-2}$, $\secboundgrad \simeq (1-\gamma)^{-3}$, $\thirddbound \simeq
    (1-\discount)^{-4}$, 
    $\hgty^2 \simeq \hgkernel^2
    (1- \discount)^{-6} + \hgreward^{2} (1- \discount)^{-4}$, $\bias \simeq
    \frac{\discount^{\lentrunc} \lentrunc}{1-\discount} + \frac{ \discount^{\lentrunc}}{(1-\discount)^{2}}$, $\moments{2}{2}
    \simeq (1- \discount)^{-4} \sizebatch^{-1}$, and $
    \moments{4}{4} 
    \simeq 
    (1- \discount)^{-8} \sizebatch^{-2}$.
Consider the sufficient exploration condition 
\begin{assumMDP}
\label{assum:sufficent_exploration}
\(\initdist\) satisfies $ \initdistmin \eqdef \min_{s \in \S}\initdist(s)> 0$.
\end{assumMDP}
Under \assumptionmdp, and assuming the uniform minorization condition $\minminsoftmu \in (0,1)$ such that $\inf_{r \in \nset} \minsoftmu(\globparam{r}) \geq \minminsoftmu$ almost surely, assumptions \Cref{assum:LL1} and \Cref{assum:LL2} are satisfied. While restrictive, this condition is unavoidable in our setting, as we neither employ projection methods nor consider entropy-regularized objectives. Similar requirements arise in the non-federated case ($\nagent=1$), e.g., \cite[Theorem~3]{lu2024towards}. We will later show that entropy regularization removes this requirement.

Applying \Cref{thm:complexity-general-fl_ql_main}, we obtain the following sample complexity and Communication Complexity bound:
\begin{corollary}[(Simplified) Sample and Communication Complexity] 
\label{thm:complexity-general-softfedpg-main}
Assume \assumptionmdp\, and set $\mathcal{T} = \Id$. Additionally,  assume that there exists $1>\minminsoftmu>0 $ such that $ \inf_{r \in [\nset]} \minsoftmu (\globparam{r}) \geq \minminsoftmu $,  for any , $\lentrunc \geq 4(1-\discount)^{-2}$, and $\nagent \cdot \sizebatch \geq (1-\discount)^{-1}$. Let $\epsilon \gtrsim \tfrac{\hgkernel}{(1-\discount)^3\minminsoftmu^{1/2}} + \tfrac{\hgreward}{(1-\discount)^2\minminsoftmu^{1/2}}  + \tfrac{ \discount^{\lentrunc}\lentrunc}{(1-\discount)\minminsoftmu^{1/2}}$
and $\step \lesssim \min\Big( (1-\discount)^3, (1-\discount)^7\minminsoftmu \sizebatch\nagent \epsilon^2, \minminsoftmu^{1/2} \epsilon (1-\discount)^5 \sizebatch\Big)$.
In this case $\SoftfedPG$ achieves $\PE[\optregobjective - \regobjective(\globparam{\trounds})] \leq \epsilon$, with a number of communication
\begin{align*}
\trounds \gtrsim 
\tfrac{\left[\optregobjective - \regobjective(\globparam{0})  - \epsilon/6\right]}{(\optregobjective - \regobjective(\globparam{0})) \minminsoftmu \epsilon} \cdot\tfrac{1}{(1-\discount)^3}
\eqsp,
\end{align*}
for a total number of sampled trajectories per agent of
\begin{align*}
\trounds \lsteps \sizebatch
\gtrsim \!
\tfrac{[\optregobjective - \regobjective(\globparam{0})  - \epsilon/6]}{(\optregobjective - \regobjective(\globparam{0})) \minminsoftmu \epsilon}  \max\Big( \tfrac{\sizebatch}{(1-\discount)^3}, \tfrac{(1-\discount)^{-7}}{\minminsoftmu \nagent \epsilon^2}, \tfrac{(1-\discount)^{-5}}{\minminsoftmu^{1/2} \epsilon}\Big).
\end{align*}
\end{corollary}
This shows that $\SoftfedPG$ has linear speedup as long as $\nagent \lesssim \min\left(\tfrac{1}{\minminsoftmu^{1/2} \epsilon (1- \discount)^2}, \tfrac{1}{\minminsoftmu\sizebatch (1- \discount)^4 \epsilon^2}\right)$.

\subsection{Convergence Analysis of \texorpdfstring{$\RegSoftfedPG$}{RS-FedPG}} \label{subsec:regsoftfedpg}
The stochastic gradient estimator of $\RegSoftfedPG$ is given by \eqref{eq:expression_of_stochastic_gradient_softmax_fedpg}, with the distinction that the reward is replaced with the entropy penalized reward
\begin{align}
\nonumber
\textstyle \softgrb{c}{z}(\param)
 &\textstyle \eqdef \frac{1}{\sizebatch}  \sum_{b=1}^{\sizebatch} \sum_{t=0}^{\lentrunc-1} \discount^{t} \big( \sum_{\ell =0}^t \nabla \log \policy_{\theta}(a_{b}^{\ell} | s_{b}^{\ell}) \big)  \\ \label{eq:expression_of_stochastic_gradient_regularised_softmax_fedpg_main}
& \textstyle \times \big[\rewardMDP[c](s_{b}^{t}, a_{b}^{t}) - \temp \log(\policy_{\theta}(a_{b}^{t}|s_{b}^{t}))\big].
\end{align}
We introduce on the central server side a projection-like operator $\projset_{\tau}$, analogous to the projection step in projected gradient descent \citep{bertsekas2003goldstein}. This operator restricts optimization to a region of interest by excluding policies with excessively large entropy penalization. For a policy $\policy$ and $s \in \S$, define $ a_{\max}^{\policy}(s) \allowdisplaybreaks = \argmax_{a\in \A} \{ \policy(a|s)\}$, choosing at random in the $\argmax$ in case of ties. For $\smash{0 < \tau < 1/(2\nactions^2)}$, set
\begin{align*}
\vspace{-0.2em}
\textstyle
\A_{\tau}^{\policy}(s) \eqdef \left\{ a \in \A, \policy(a|s) 
\leq \tau/2 \right\} \eqsp.
\vspace{-0.2em}
\end{align*}
We define the operator $\mathcal{U}_\tau$ which acts as a projection in the policy space as follows: for each $(s,a) \in \S \times \A$:
$\mathcal{U}_\tau(\policy)(a|s)
=
\tau$, if $\policy(a|s) \leq \tau/2$, 
$\mathcal{U}_\tau(\policy)(a|s)
= \policy(a|s) - \sum\limits_{b \in \A_{\tau}^{\policy}(s)}(\tau - \policy(b|s))$, if $a = a_{\max}^{\policy}(s)$, $\mathcal{U}_\tau(\policy)(a|s)\policy(a|s)$, otherwise.  This operator prevents policies from becoming too deterministic: for any $s, a \in \S \times \A$, if $\policy(a|s)$ approaches zero, it is raised above a $\tau$-dependent threshold. The operator $\mathcal{U}_\tau$ acts in policy space; we denote by $\projset_\tau$ the associated operator in the logit space, \ie\ for all $\theta$, $\policy_{\projset_\tau(\theta)} \eqdef \mathcal{U}_\tau(\policy_\theta)$ (see \Cref{subsec:monotone_improvement_operator} for details). With a suitable choice of $\tau$, applying $\projset_\tau$ yields logits with a higher regularized value.
\begin{lemma}
\label{lem:choice-tau}
Assume that $\initdist$ satisfies \assumptionmdp.
Let $\sthreshold \eqdef \min \left(  \tfrac{1}{3} \exp\left( -\tfrac{16 + 8 \discount\temp \log(\nactions)}{\temp(1- \discount)^2 \initdistmin}\right), \frac{1}{3^8 \nactions^4}\right)$.
Then, for any $\theta \in \logitspace$ and for $\widetilde{\theta} = \projset_{\sthreshold}(\theta)$, it holds that  $\auxobjective(\widetilde{\theta}) \geq \auxobjective(\theta)$ and that for any $(s,a) \in \S \times\A,  \policy_{\widetilde{\param}}(a|s) \geq \sthreshold$.
\end{lemma}
In this setting, Condition \Cref{assum:gen-f-reg} holds with:
$\boundgrad \simeq 
    \tfrac{1+ \temp \log(\nactions)}{(1- \discount)^{2}}$, 
$\secboundgrad \simeq \tfrac{1+ \temp\log(\nactions)}{(1-\gamma)^{3}}$,  $\thirddbound \simeq
    \frac{1 + \temp \log |\A|}{(1-\discount)^{4}}$, 
$\hgty^2 \simeq 
    \tfrac{(1+\temp \log(\nactions))^2\hgkernel^2}{(1- \discount)^6}  + \tfrac{\hgreward^2}{(1-\discount)^4}$,   
$\bias \simeq 
    \tfrac{1+ \temp \log(\nactions)\discount^{\lentrunc} \lentrunc}{1-\discount} + \tfrac{1+ \temp, \log(\nactions)\discount^{\lentrunc}}{(1-\discount)^{2}}$,
$\moments{2}{2}\!
    \simeq\! \frac{1 + \temp^2\log(\nactions)^2 }{(1- \discount)^{4} \sizebatch}$,
$\moments{4}{4} 
    \!\simeq \!
    \frac{1 +  \temp^4\log(\nactions)^4}{(1- \discount)^{8} \sizebatch^2}$.
    
By \eqref{eq:nonuniform-lojasiewicz-regularized}, each client satisfies the condition \Cref{assum:PL1}. 
Moreover, when using $\projset_{\sthreshold}$ as the improvement operator, \Cref{lem:choice-tau} shows that the corresponding PL constant is uniformly bounded from below by 
\[
\vspace{-0.2em}
\textstyle
\minlmufl = \minminregsoftmu := \temp (1-\discount)\,\initdistmin^2 \sthreshold^2 / \nstates,
\vspace{-0.2em}
\]
where $\sthreshold$ is defined in \Cref{lem:choice-tau}. 
Thus, \Cref{assum:PL2} also holds. 
We can therefore apply \Cref{thm:complexity-linear-pl_main}:
\begin{corollary}[(Simplified) Sample and Communication Complexity of $\RegSoftfedPG$] 
\label{thm:complexity-regsoftmaxfedpg_main}
Assume \assumptionmdp\, and that the projection operator is $\projset_{\sthreshold}$. Let
$
\textstyle
\epsilon \gtrsim \frac{(1+ \temp \log(\nactions))^2\hgkernel^2}{\minminregsoftmu(1- \discount)^{6}} + \frac{\hgreward^2}{(1-\discount)^4 \minminregsoftmu} + \frac{(1+\lambda \log(\nactions))^2\discount^{2\lentrunc}\lentrunc^2}{(1-\discount)^2}$
and $\step \lesssim \min\Big( \allowdisplaybreaks \frac{(1-\discount)^3}{1+\temp\log(\nactions)}, \frac{\minminregsoftmu  \epsilon \nagent \sizebatch(1-\discount)^7}{(1+\temp \log(\nactions))^3}, \frac{\minminregsoftmu^{1/2} \sizebatch (1-\discount)^{7/2} \epsilon^{1/2}}{(1+ \temp \log(\nactions))^{3/2}} \Big)$.
Then $\RegSoftfedPG$ achieves $\optauxobjective - \PE[\auxobjective{\theta^{\trounds}}] \leq \epsilon$, with a number of communication
\begin{align*}
\textstyle
\trounds \gtrsim 
\tfrac{1}{ \minminregsoftmu }
\log\Big(
\tfrac{4(\optauxobjective - \PE[\auxobjective{\theta^{0}}])}{\epsilon}
\Big) \tfrac{1+ \temp \log(\nactions)}{(1-\discount)^3}
\eqsp,
\end{align*}
for a total number of sampled trajectories per agent of
\begin{align*}
\textstyle \trounds \lsteps \sizebatch
 \!\gtrsim\! 
\tfrac{1}{\minminregsoftmu } 
\log& \textstyle\!\Big(
\tfrac{\optauxobjective - \PE[\auxobjective{\theta^{0}}]}{\epsilon}
\Big)\!\max\Big( \!\tfrac{(1\!+\!\temp\log(\nactions)) \sizebatch}{(1-\discount)^3},
\\
& \textstyle \tfrac{(1+\temp \log(\nactions))^3}{\minminregsoftmu  \epsilon \nagent (1-\discount)^7}, \tfrac{ (1+ \temp \log(\nactions))^{3/2}}{\minminregsoftmu^{1/2} (1-\discount)^{7/2} \epsilon^{1/2}}\Big).
\end{align*}
\end{corollary}
This proves that $\RegSoftfedPG$ achieve a logarithmic communication complexity in the desired accuracy while guaranteeing linear speedup as long as $\nagent \lesssim \min\left( \tfrac{(1+ \temp \log(\nactions))^{3/2}}{\minminregsoftmu^{1/2}\epsilon^{1/2}(1- \discount)^{3/2}},\frac{(1+ \temp \log(\nactions))^2}{\minminregsoftmu \epsilon \sizebatch (1- \discount)^4} \right)$.

%% file: AISTATS/src/Structure_FRL.tex
In this section, we examine the structure of optimal policies in tabular FRL under heterogeneity. We provide only a brief overview here; a detailed analysis and proof of the theorem below are given in \Cref{sec:proof_theorem_upset}.
In the non-federated setting ($\nagent=1$), it is classical that optimal policies can always be chosen deterministic \citep[Theorem~1.7]{agarwal2019reinforcement}. 
In contrast, for the federated case ($\nagent > 1$), this property no longer holds: optimal policies may need to be stochastic, and in some cases even non-stationary.  We provide an example of such an FRL instance in \Cref{fig:local_det_vs_local_stat_main}.

A \emph{history-dependent policy} is a sequence of mappings $(\policy_c^t)_{t \in \nset}$, where each $\policy_c^t$ selects actions based on the entire trajectory observed by agent $c$ up to time $t$. 
The class of such policies is denoted by $\Pi_\ell$. 
A \emph{stationary stochastic policy} is a mapping $\policy\colon \S \to \pA$, assigning to each state a distribution over actions; this class is denoted by $\PiSta$. 
A \emph{deterministic policy} is a mapping $\policy\colon \S \to \A$, choosing a single action for each state; these form the class $\PiDet$. 
By construction, $\PiDet \subset \PiSta \subset \Pi_\ell$. Restricting the policy class can only decrease (or preserve) the supremum of the objective $\regobjective$. 
Unlike the single-agent setting, where $\PiDet$ always contains an optimal policy, this property does not extend to FRL with heterogeneous dynamics.
\begin{theorem}
\label{thm:biggest_upset_of_frl}
There exists an FRL instance with two infinite-horizon discounted MDPs 
such that
\[
\max_{\pi \in \PiDet} \frlobjective(\pi) 
< \max_{\pi \in \PiSta} \frlobjective(\pi)
,\quad 
\max_{\pi \in \PiSta} \frlobjective(\pi) 
< \max_{\pi \in \Pi_{\ell}} \frlobjective(\pi).
\]
\end{theorem}
The key difficulty stems from heterogeneity in the transition kernels: 
while homogeneous transitions (even with heterogeneous rewards) reduce 
to a standard RL problem with averaged rewards 
(see \Cref{sec:multitask-RL}), heterogeneity in dynamics 
fundamentally alters the structure of optimal policies. 

\begin{remark}[Algorithmic implications]
The hierarchy of policy classes directly impacts algorithm design. 
Methods restricted to deterministic policies, such as Fed-Q-learning 
\citep{jin2022federated}, are provably suboptimal in heterogeneous FRL 
(see \Cref{thm:biggest_upset_of_frl}). At the other extreme, 
history-dependent policies are too complex for practical implementation. 
Stationary stochastic policies thus strike a natural balance.
\end{remark}

\begin{figure}
    \centering
    \includegraphics[width=0.8\linewidth]{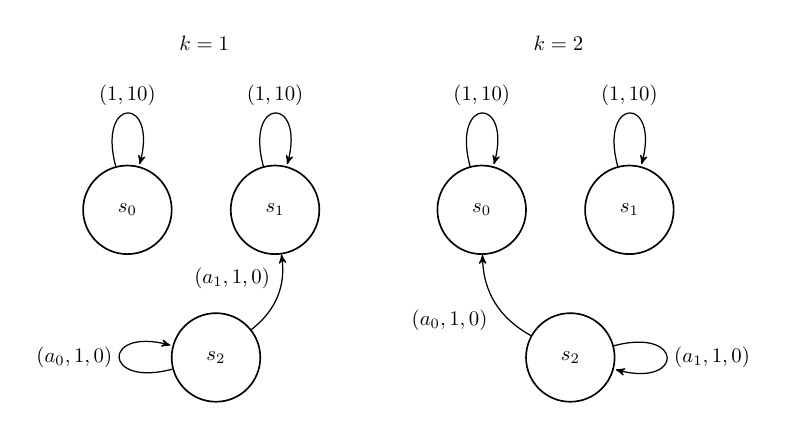}
    \caption{FRL task with no optimal deterministic policy. The triplet means (action, probability, reward). If the action is unspecified, it means that all actions give the same reward and lead to the same state.}
    \label{fig:local_det_vs_local_stat_main}
\end{figure}

\begin{figure*}[t]
    \centering
    \begin{subfigure}[b]{0.32\textwidth}
        \centering
        \includegraphics[width=\textwidth]{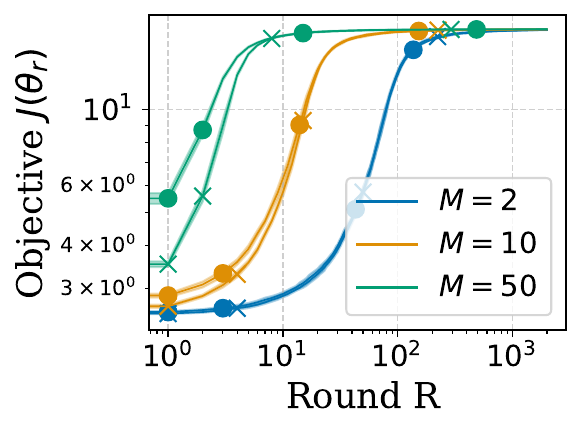}
        \caption{\texttt{GridWorld}, $\lsteps=5$, $\hgkernel = 0.0$}
        \label{subfig:speedup_homogenous_grdiword}
    \end{subfigure}
    \begin{subfigure}[b]{0.32\textwidth}
        \centering
        \includegraphics[width=\textwidth]{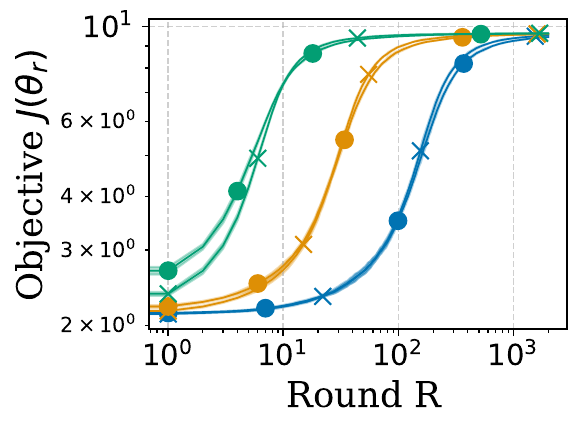}
        \caption{\texttt{GridWorld}, $\lsteps=5$, $\hgkernel = 0.3$}
        \label{subfig:speedup_hetregenous_gridword}
    \end{subfigure}
    \begin{subfigure}[b]{0.32\textwidth}
        \centering
        \includegraphics[width=\textwidth]{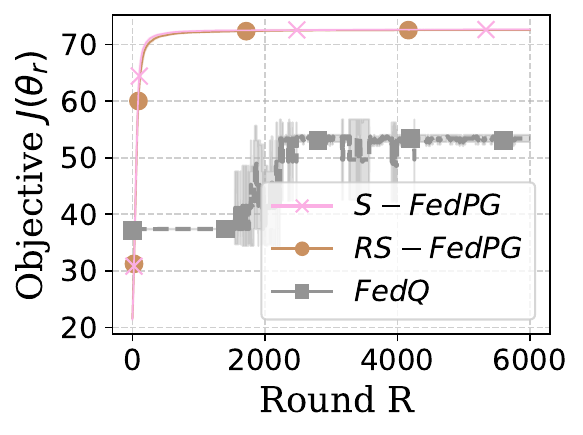}
        \caption{\texttt{GridWorld}, $\lsteps=5$, $\hgkernel \gg 0.3$}
        \label{subfig:win_fed_q_learning_main_gridword}
    \end{subfigure}
    \begin{subfigure}[b]{0.32\textwidth}
        \centering
        \includegraphics[width=\textwidth]{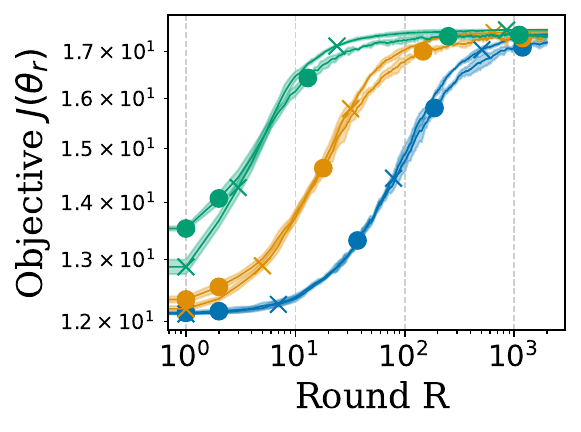}
        \caption{\texttt{Synthetic}, $\lsteps=5$, $\hgkernel = 0.0$}
        \label{subfig:speedup_homogenous_synthetic}
    \end{subfigure}
    \begin{subfigure}[b]{0.32\textwidth}
        \centering
        \includegraphics[width=\textwidth]{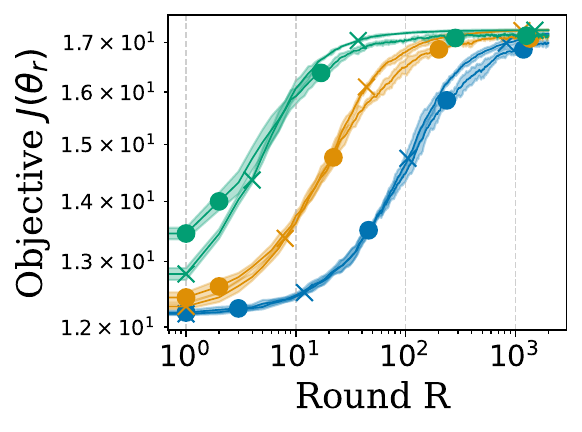}
        \caption{\texttt{Synthetic}, $\lsteps=5$, $\hgkernel = 0.3$}
        \label{subfig:speedup_heteregenous_synthetic}
    \end{subfigure}
    \begin{subfigure}[b]{0.32\textwidth}
        \centering
        \includegraphics[width=\textwidth]{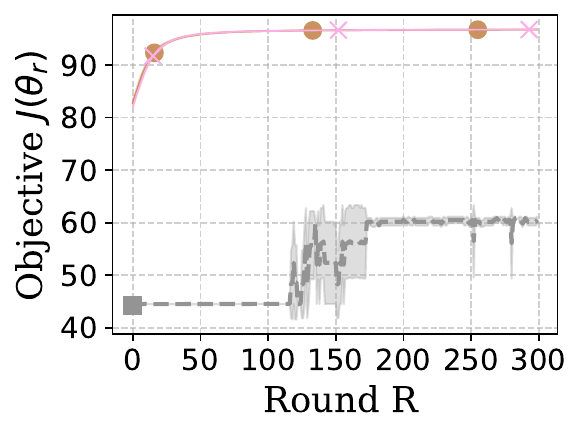}
        \caption{\texttt{Synthetic}, $\lsteps=5$, $\hgkernel \gg 0.3$}
        \label{subfig:win_fed_q_learning_main_synthetic}
    \end{subfigure}
    \hfill
    \caption{Comparison of $\SoftfedPG$ (crosses), $\RegSoftfedPG$ (circles), and \texttt{Fed-Q-learning} (squares) in two environments: \texttt{Synthetic} (below) and \texttt{Gridword} (above). From left to right: Value of the global objective $\frlobjective(\theta^r)$ achieved by $\SoftfedPG$, and  $\RegSoftfedPG$, for $\hgkernel=0$, and for different numbers of agents $\nagent \in \{2, 10, 50\}$ as a function of the number of rounds $\trounds$; Value of the global objective $\frlobjective(\theta^r)$ achieved by $\SoftfedPG$, and  $\RegSoftfedPG$, for $\hgkernel=0.3$, and for different numbers of agents $\nagent \in \{2, 10, 50\}$ as a function of the number of rounds $\trounds$; (c) Value of $\frlobjective(\theta^r)$, comparing all three algorithms.}
    \vspace{-1em}
    \label{fig:linearspeedup}
\end{figure*}

%% file: AISTATS/src/experiments.tex
\label{sec:experiments}

We study the empirical performance of the two proposed methods on two environments, and illustrate their advantage over \texttt{Fed-Q-learning} \citep{jin2022federated} in heterogeneous settings. Experiments were conducted on a computer with an Intel Xeon 6534 CPU with 196GB RAM. We report the average over $4$ runs and the standard deviation in all the plots. Our code is available on \url{https://github.com/Labbi-Safwan/FedPolicy-gradient}

In the two environments, the transition kernel of agent $c$ is modeled as a mixture of two components: $\kerMDP[c] = (1- \hgkernel) \kerMDP^{\mathsf{com}}+ \hgkernel \kerMDP[c]^{\mathsf{ind}}$ where $\kerMDP^{\mathsf{com}}$ is a common kernel shared by all agents, and $\kerMDP[c]^{\mathsf{ind}}$ is agent-specific. In both environments, the agent starts randomly from a uniformly sampled position, and $\discount = 0.95$.
\vspace{-0.5em}
\paragraph{Synthetic.} In the \texttt{synthetic} environment \citet{zheng2023federated}, all agents share a common reward function  $\rewardMDP$, where each reward value $\rewardMDP(s,a)$ for $(s,a) \in \S\times \A$ is independently sampled from the uniform distribution over $[0,1]$.  For each $(s, a)$, the transition kernels  $\kerMDP^{\mathsf{com}}(\cdot|s,a)$ and $\kerMDP[c]^{\mathsf{ind}}(\cdot|s,a)$ are drawn uniformly and randomly from the \(\nstates\)-dimensional simplex. In the experiments with $\hgkernel=0$ or $0.3$, we consider environments with $\nstates=5$ states and 
$\nactions = 4$ actions. The highly heterogeneous \texttt{synthetic} FRL instance ($\hgkernel\gg 0.3$) extends the previous setup by adding two states, each reachable from one of the original five states. Once reached, these states yield a reward of \(+1\) in every timestep, and the agents remain there indefinitely. This instance includes two types of MDPs, differing in which high-reward state is accessible: in the first type, the first two actions deterministically lead to the rewarding state, while the last two deterministically make the agent stay in the same state; in the second type, this mapping is reversed. As a result, agents must take opposing actions, similar to \Cref{fig:local_det_vs_local_stat_main}, to maximize their rewards.
\vspace{-0.5em}
\paragraph{GridWorld \citep{rlberry}.} In the \texttt{GridWorld} environment, an agent navigates a \(3 \times 3\) grid to reach a goal state at \((2, 2)\), receiving a reward of \(+1\) upon arrival and \(0\) otherwise. The agent can move in four directions, with intended actions succeeding with probability \(0.8\) under the shared dynamics \(\kerMDP^{\mathsf{com}}\), and failing to a random neighbor with probability \(0.2\). The individual transition kernel $\kerMDP[c]^{\mathsf{ind}}$ moves the agents to a neighboring cell with random probabilities that are specific to each agent. A wall at \((1, 1)\) results in \(\nstates = 8\) reachable states. We use this setup for experiments with \(\hgkernel = 0\) and \(\hgkernel = 0.3\). In the highly heterogeneous FRL instance, the target position is connected to two additional states, similarly to what has been described in the heterogeneous \texttt{synthetic} FRL instance.

\paragraph{$\fedPG$ has linear speedup.}
In \Cref{subfig:speedup_homogenous_grdiword,subfig:speedup_hetregenous_gridword,subfig:speedup_homogenous_synthetic,subfig:speedup_heteregenous_synthetic}, we illustrate the \textit{linear speedup} property by evaluating $\SoftfedPG$ and $\RegSoftfedPG$ in both homogeneous and slightly heterogeneous environments. Specifically, we report the global objective $\frlobjective$ during the learning process for various numbers of agents, using the theoretically motivated step size, and $\temp = 0.05$ for $\RegSoftfedPG$. Both algorithms show that using a larger number of agents in the federation reduces the number of iterations per agent to reach convergence, highlighting the benefits of collaboration even among heterogeneous agents. 

\paragraph{$\fedPG$ outperforms \texttt{Fed-Q-learning}.} In \Cref{subfig:win_fed_q_learning_main_gridword,subfig:win_fed_q_learning_main_synthetic}, we compare the performance of \texttt{Fed-Q-learning} with $\SoftfedPG$ and $\RegSoftfedPG$ on two highly heterogeneous FRL problems.  $\SoftfedPG$ and $\RegSoftfedPG$ \textit{learn better policies}, demonstrating, as suggested by \Cref{thm:biggest_upset_of_frl}, the advantage of leveraging methods that learn a stochastic policy.

%% file: AISTATS/src/Conclusion.tex
\vspace{-10pt}
This work extends the theoretical foundations of FRL in heterogeneous environments. 
Our main contribution is the first global convergence guarantee for both non-regularized and entropy-regularized policy gradient methods in heterogeneous FRL. 
We also highlight structural differences that challenge classical RL properties, showing in particular that deterministic and stationary policies can be suboptimal. An important direction for future work is to address the \emph{heterogeneity bias} that arises in federated objectives. 
A natural candidate is to adapt bias-reduction techniques developed in the convex setting, such as \textsc{Scaffold}~\citep{karimireddy2020scaffold}, to the non-convex FRL landscape.

%% file: AISTATS/appendix/fedavg.tex
In this section, we study the following federated stochastic optimization problem
\begin{align}
\label{eq:fedavg_problem}
\max_{\param \in \rset^d}
\f{\param} = \frac{1}{\nagent} \sum_{c=1}^{\nagent} \nf{c}{\param} \eqsp,
\quad \text{ where }
\nf{c}{\param} \eqdef \PE_{\randState{c} \sim \sampledist{c}{\param}}[\nfs{c}{\param}{\randState{c}}]\eqsp,
\end{align}
where each $\randState{c}$ is a random variable which takes its value from a distribution $\sampledist{c}{\param}$, which may depend on $\param$, and takes values in a measurable set $(\msZ,\mathcal{Z})$, and where the function $(z, \param) \mapsto \nfs{c}{\param}{z}$ are measurable functions.
Each function $\nfw{c}$ is only available to the client $c$ through \emph{biased} stochastic gradients $\grb{c}{z}(\param)$, whose expected value is
\begin{align}
\label{eqdef:expected_sto_estimator}
\egrb{c}(\param) \eqdef  \PE_{\randState{c} \sim \sampledist{c}{\param}}[  \grb{c}{\randState{c}}(\param) ] \eqsp,
\end{align}
but is typically different from the gradient of $\nfw{c}$. 

To solve \eqref{eq:fedavg_problem}, we use $\projfedAVG$, an extension of projected gradient ascent to the federated setting, which performs local stochastic gradient updates at the client level with step size $\step$, aggregates the locally updated model, and projects the resulting model using a projection-like operator $\projset \colon \rset^{d} \rightarrow \rset^{d}$.
For completeness, we give the pseudo-code for this algorithm in \Cref{algo:FEDAVG}.

\subsection{Descent Lemma}

\paragraph{Assumptions.}
We derive our new ascent lemma for this algorithm under the following assumptions, which slightly differ from the classical setting, but are typical in reinforcement learning.
First, we assume that both the true gradient and its biased estimator are Lispchitz-continuous, that the true gradient is bounded, and that the objective functions' third derivates are uniformly bounded. 
\begin{assumFL}
\label{assumFL:uniform_grad_bound}
There exists $\boundgrad > 0$, such that for all $c \in [\nagent]$ and $\param \in \rset^d$,
    \begin{align}
    \label{eqdef:uniform_grad_bound}
        \norm{ \ngf{c}{\param} } 
        \le
        \boundgrad
        \eqsp,
        \quad \text{ for all } c \in [\nagent]~, \eqsp \param \in \rset^d~.
    \end{align}
\end{assumFL}
\begin{assumFL}
\label{assumFL:smoothness}
For any $c \in [\nagent]$, the functions $\ngfw{c}$ and the biased gradients $\egrb{c}$ are $\secboundgrad$-Lipschitz, that is
\begin{align}
\label{eqdef:smoothness}
    \norm{ \nhf{c}{\param} u}
    & \le
    \secboundgrad \norm{u}
    \eqsp,
    \quad \text{ for all } \param \in \rset^d, u \in \rset^d
    \eqsp,
    \\
    \label{eqdef:biased-gradient-smoothness}
    \norm{ \egrb{c}(\param) - \egrb{c}(\param') }
    & \le
    \secboundgrad \norm{ \param - \param' }
    \eqsp,
    \quad \text{ for all } \param, \param' \in \rset^d
    \eqsp.
\end{align}
\end{assumFL}
\begin{assumFL}
\label{assumFL:bounded-third-derivative}
For any $c \in [\nagent]$, the function $\nfw{c}$ is three times differentiable and has bounded third derivative tensor, that is, there exists $\thirddbound < \infty$ such that
    \begin{align}
    \label{eqdef:bounded-third-derivative}
        \norm{ \nhhf{c}{\param} u^{\otimes 2} }
        \le
        \thirddbound \norm{ u }^2
        \eqsp,
        \quad \text{ for all } \param \in \rset^d ~, \eqsp u \in \rset^d
        \eqsp.
    \end{align}
\end{assumFL}
\begin{assumFL}
\label{assumFL:heterogeneity}
For any $c \in [\nagent]$, the gradient gradient heterogeneity is uniformly bounded, that is, there exists $\hgty\geq 0$ such that
    \begin{align}
\label{eqdef:heterogeneity}
        \norm{ \ngf{c}{\param} - \gf{\param}}
        \le 
        \hgty
        \eqsp,
        \quad \text{ for all } c \in [\nagent] ~, \eqsp \param \in \rset^d
        \eqsp.
    \end{align}
\end{assumFL}
\begin{assumFL}
\label{assumFL:variance}
For $p \in \{2, 4\}$, $c \in [\nagent]$, there exists $\moments{p}{p} \geq 0$ such that
\begin{align}
\label{eqdef:variance}
    \PE_{\randState{c} \sim \sampledist{c}{\param}}[ \norm{ \grb{c}{\randState{c}}(\param) - \egrb{c}(\param) }^p ]
    \le
    \moments{p}{p} 
    \eqsp,
    \quad \text{ for all } c \in [\nagent]~, \eqsp \param \in \rset^d~
    \eqsp.
\end{align}
\end{assumFL}
\begin{assumFL}
\label{assumFL:bias}
For any $c \in [\nagent]$,  there exists $\bias \geq  0$ such that
    \begin{align}
    \label{eqdef:bias}
        \norm{ \egrb{c}(\param) - \ngf{c}{\param}}
        \le
        \bias
        \eqsp,
        \quad \text{ for all } c \in [\nagent]~, \eqsp \param \in \rset^d~.
    \end{align}
\end{assumFL}

\paragraph{Proof of Descent lemma.}
To establish an descent lemma for \Cref{algo:FEDAVG}, we first provide two lemmas: in \Cref{lem:bound-expec-drift}, we give a bound on the expected drift, and in \Cref{lem:globparam-variance}, we provide a bound on the variance of the global averaged parameters.
We then use these two lemmas to prove \Cref{lem:ascent_lemma}, which is our main result.

In the following, we define the filtration adapted to the global and local iterates of \Cref{algo:FEDAVG} as
\begin{align*}
\globfiltr{r} 
\eqdef 
\sigma\Big(\randState{c}[r',h'] : r' < r, h' \in \{0, \dots, \lsteps\}, c' \in \{1, \dots, \nagent\}\Big)
\eqsp.
\end{align*} 
We now prove our first lemma on the expected drift of \Cref{algo:FEDAVG}.
\begin{lemma}[Bound on Expected Drift]
\label{lem:bound-expec-drift}
Assume \Cref{assumFL:uniform_grad_bound} to \Cref{assumFL:bias}. Let $\step > 0$ such that $\step \lsteps \secboundgrad \le 1/6$ and $32 \step^2 \lsteps^2 \thirddbound^2 \boundgrad^2 \le \secboundgrad^2$, where $\secboundgrad$ and $\boundgrad$ are defined in \Cref{assumFL:smoothness} and \Cref{assumFL:uniform_grad_bound}. respectively.
Then the iterates of $\projfedAVG$ satisfy
\begin{align*}
& \frac{1}{\nagent \lsteps}
\sum_{c=1}^{\nagent}
\sum_{h=1}^{\lsteps-1}
 \norm{  \CPE{   \ngf{c}{\globparam{r}} -    \ngf{c}{\locparam{c}{r,h}}  }{\mcF^r} }[2]^2 
\\
& \qquad \le
\frac{8 \step^2 \secboundgrad^2 \lsteps (\lsteps\!\!-\!1)}{\nagent} \sum_{c=1}^\nagent \norm{ \ngf{c}{\globparam{r}} }[2]^2
+ 8 \step^2 \secboundgrad^2 \lsteps(\lsteps\!\!-\!1) \bias^2
+ 4 \cdot 12^3 \step^4 \thirddbound^2 \lsteps(\lsteps-1) \moments{4}{4} 
\eqsp.
\end{align*}
\end{lemma}
\begin{proof}
\textit{(Definition of Drift Error Terms.)}
To prove this lemma, we will bound each term of the sum
\begin{align*}
\boldsymbol{\mathrm{U}_c^h}
\eqdef 
 \norm{  \CPE{ \ngf{c}{\globparam{r}} -    \ngf{c}{\locparam{c}{r,h}}  }{\globfiltr{r}}}[2]^2
 \eqsp.
\end{align*}

\textit{(Bound on Drift Error Terms.)}
First, we use Taylor expansion to expand
\begin{align*}
\CPE{ \ngf{c}{\locparam{c}{r, h}}}{\globfiltr{r}}
- 
\ngf{c}{\globparam{r}} 
& =
\nhf{c}{\globparam{r}} \CPE{ \locparam{c}{r, h} - \globparam{r} }{\globfiltr{r}}
+ \CPE{ \matreste{3,c}{r}(\locparam{c}{r, h}) (\locparam{c}{r,h} - \globparam{r})^{\otimes 2} }{\globfiltr{r}}
\eqsp,
\end{align*}
where we defined the integral remainder as
\begin{align}
\matreste{3,c}{r}(\locparam{c}{r, h})
& =
\int_{0}^1 (1-t) \nhhf{c}{\globparam{r} + t(\locparam{c}{r, h} - \globparam{r})} \rmd t
\eqsp.
\end{align}
We thus obtain the following bound, using Jensen's inequality and the bound on the third derivatives tensor of $\nfw{c}$,
\begin{align}
\nonumber
\boldsymbol{\mathrm{U}_c^h}
& \le
2 \norm{ 
\nhf{c}{\globparam{r}} \CPE{ \locparam{c}{r, h} - \globparam{r} }{\globfiltr{r}}
}[2]^2
+ 2 \norm{ 
 \CPE{ \matreste{3,c}{r}(\locparam{c}{r, h}) (\locparam{c}{r,h} - \globparam{r})^{\otimes 2} }{\globfiltr{r}}
}[2]^2
\\
\label{eq:bound-drift-bound-bhc}
& \le
2 \secboundgrad^2 \norm{  \CPE{ \locparam{c}{r, h} \!-\! \globparam{r} }{\globfiltr{r}} }[2]^2
+ 2 \thirddbound^2
\CPE{ \norm{ \locparam{c}{r,h} \!-\! \globparam{r} }[2]^4 }{\globfiltr{r}}
\eqsp,
\end{align}
We now use the fact that $\locparam{c}{r,h} = \globparam{r} - \step \sum_{\ell=0}^{h-1} \grb{c}{\randState{c}[r,\ell+1]}(\locparam{c}{r, \ell})$ and \eqref{eqdef:bias} to write
\begin{align*}
& 2 \secboundgrad^2\norm{  \CPE{ \locparam{c}{r, h} - \globparam{r} }{\globfiltr{r}} }[2]^2
\\
& \quad =
2 \step^2 \secboundgrad^2 \bnorm{ \bCPE{ \sum_{\ell=0}^{h-1} 
\ngf{c}{\globparam{r}}
+ \ngf{c}{\locparam{c}{r, \ell}} - \ngf{c}{\globparam{r}} 
+ \grb{c}{\randState{c}[r,\ell+1]}(\locparam{c}{r, \ell}) - \ngf{c}{\locparam{c}{r, \ell}}  }{\globfiltr{r}} }[2]^2
\\[-0.5em]
& \quad \le
6 \step^2 \secboundgrad^2 h^2 \norm{ \ngf{c}{\globparam{r}} }[2]^2
+ 6 \step^2 \secboundgrad^2 h  \sum_{\ell=0}^{h-1}\bnorm{ \CPE{ \ngf{c}{\globparam{r}} - \ngf{c}{\locparam{c}{r, \ell}} }{\globfiltr{r}} }[2]^2
+ 6 \step^2 \secboundgrad^2 h^2 \bias^2
\eqsp,
\end{align*}
Completing the sum until $\ell = \lsteps-1$ and plugging this inequality in \eqref{eq:bound-drift-bound-bhc}, we obtain
\begin{align*}
\boldsymbol{\mathrm{U}_c^h}
& \le
6 \step^2 \secboundgrad^2 h^2  \norm{ \ngf{c}{\globparam{r}} }[2]^2
+ 6 \step^2 \secboundgrad^2 h \sum_{\ell=0}^{\lsteps-1}  \bnorm{  \CPE{ \ngf{c}{\globparam{r}} - \ngf{c}{\locparam{c}{r, \ell}} }{\globfiltr{r}} }[2]^2
\\
& \qquad
+ 6 \step^2 \secboundgrad^2 h^2 \bias^2
+ 2 \thirddbound^2
\CPE{ \norm{ \locparam{c}{r,h} \!-\! \globparam{r} }[2]^4 }{\globfiltr{r}}
\eqsp.
\end{align*}
Now, we average the above inequality for $h = 0$ to $\lsteps-1$ and $c = 1$ to $\nagent$, which gives
\begin{align*}
\frac{1}{\nagent \lsteps}
\sum_{c=1}^{\nagent}
\sum_{h=1}^{\lsteps-1}
\boldsymbol{\mathrm{U}_c^h}
& \le
\frac{3 \step^2 \secboundgrad^2 \lsteps (\lsteps\!-\!1)}{\nagent} \sum_{c=1}^\nagent \norm{ \ngf{c}{\globparam{r}} }[2]^2
+ \frac{3 \step^2 \secboundgrad^2 \lsteps(\lsteps\!-\!1)}{\nagent \lsteps } \sum_{c=1}^\nagent \sum_{h=0}^{\lsteps-1} 
\boldsymbol{\mathrm{U}_c^h}
\\ 
& \quad 
+ 3 \step^2 \secboundgrad^2 \lsteps(\lsteps-1) \bias^2
+ \frac{2  \thirddbound^2}{\nagent \lsteps  } \sum_{c=1}^\nagent \sum_{h=0}^{\lsteps-1}
\CPE{ \norm{ \locparam{c}{r,h} - \globparam{r} }[2]^4 }{\globfiltr{r}}
\eqsp,
\end{align*}
where we used $\sum_{h=0}^{\lsteps-1} h^2 \le \lsteps \sum_{h=0}^{\lsteps-1} h = \frac{\lsteps^2 (\lsteps-1)}{2}$.
Using that $3 \step^2 \secboundgrad^2 \lsteps (\lsteps-1) \le 1/2$, reorganizing the terms, and multiplying the resulting inequality by $2$, we obtain
\begin{align}
\nonumber
\frac{1}{\nagent \lsteps}
\sum_{c=1}^{\nagent}
\sum_{h=1}^{\lsteps-1}
\boldsymbol{\mathrm{U}_c^h}
& \le
\frac{6 \step^2 \secboundgrad^2 \lsteps (\lsteps\!\!-\!1)}{\nagent} \sum_{c=1}^\nagent \norm{ \ngf{c}{\globparam{r}} }[2]^2
\\ & \quad
\label{eq:bound-sum-drift-with-fourth-moment}
\!+\! \frac{4  \thirddbound^2}{\nagent \lsteps } \sum_{c=1}^\nagent \sum_{h=0}^{\lsteps-1}
\CPE{ \norm{ \locparam{c}{r,h} \!\!-\! \globparam{r} }[2]^4 }{\globfiltr{r}}
\!+\! 6 \step^2 \secboundgrad^2 \lsteps(\lsteps\!\!-\!1) \bias^2
\eqsp.
\end{align}

\textit{(Fourth Order Drift Terms.)}
We now bound the fourth moment of the drift.
To this end, we define 
\begin{align*}
\boldsymbol{\mathrm{V}_c^h} 
\eqdef
\CPE{ \norm{ \locparam{c}{r,h} - \globparam{r} }^4 }{\globfiltr{r}}
\eqsp,
\end{align*}
and we write $\locparam{c}{r,h} = \globparam{r} +
\step \sum_{\ell=0}^{h-1} \grb{c}{\randState{c}[r,\ell+1]}(\locparam{c}{r,\ell})$, and we decompose each update as 
\begin{align*}
\grb{c}{\randState{c}[r,\ell+1]}(\locparam{c}{r,\ell})
& =
\grb{c}{\randState{c}[r,\ell+1]}(\locparam{c}{r,\ell})
- \egrb{c}(\locparam{c}{r,\ell})
+ \egrb{c}(\locparam{c}{r,\ell})
- \ngf{c}{\locparam{c}{r,\ell}}
+ \ngf{c}{\locparam{c}{r,\ell}}
- \ngf{c}{\globparam{r}}
+ \ngf{c}{\globparam{r}}
\eqsp.
\end{align*}
This gives the bound
\begin{align}
\nonumber
\boldsymbol{\mathrm{V}_c^h}
& \le
4^3 \step^4 \underbrace{\CPE{ \bnorm{ \sum_{\ell=0}^{h-1}
\grb{c}{\randState{c}[r,\ell+1]}(\locparam{c}{r,\ell})
- \egrb{c}(\locparam{c}{r,\ell}) }^4 }{\globfiltr{r}}}_{T_1}
+ 4^3 \step^4 \underbrace{\CPE{ \bnorm{ \sum_{\ell=0}^{h-1} \egrb{c}(\locparam{c}{r,\ell})
- \ngf{c}{\locparam{c}{r,\ell}} }^4 }{\globfiltr{r}}}_{T_2}
\\[-0.5em]
\label{eq:bound-moment-four-decomposition}
& \quad 
+ 4^3 \step^4 \underbrace{\CPE{ \bnorm{ \sum_{\ell=0}^{h-1}\ngf{c}{\locparam{c}{r,\ell}}
- \ngf{c}{\globparam{r}} }^4 }{\globfiltr{r}}}_{T_3}
+ 4^3 \step^4 \underbrace{ h^4 \norm{ \ngf{c}{\globparam{r}} }^4 }_{T_4}
\eqsp.
\end{align}
We bound $T_1$ using Burkholder's inequality (Theorem 8.6, \citealp{oskekowski2012sharp}), which gives
\begin{align}
T_1 
& \le
3^4 \Big\{ \sum_{\ell=0}^{h-1} \CPE[1/2]{ \norm{ 
\grb{c}{\randState{c}[r,\ell+1]}(\locparam{c}{r,\ell})
- \egrb{c}(\locparam{c}{r,\ell}) }^4}{\globfiltr{r}} \Big\}^2
\label{eq:bound-moment-four-decomposition-T1}
\le 
3^4 h^2 \moments{4}{4} 
\eqsp.
\end{align}
The term $T_2$ is a bias term, which we bound using \eqref{eqdef:bias},
\begin{align}
\label{eq:bound-moment-four-decomposition-T2}
T_2
\le h^3  \sum_{\ell=0}^{h-1} \CPE{ \norm{ \egrb{c}(\locparam{c}{r,\ell})
- \ngf{c}{\locparam{c}{r,\ell}} }^4 }{\globfiltr{r}}
\le h^4 \bias^4
\eqsp.
\end{align}
Then, we bound $T_3$ using \eqref{eqdef:smoothness}
\begin{align}
\label{eq:bound-moment-four-decomposition-T3}
T_3 
\le h^3 \sum_{\ell=0}^{h-1} \CPE{ \norm{ \ngf{c}{\locparam{c}{r,\ell}}
- \ngf{c}{\globparam{r}} }^4 }{\globfiltr{r}}
\le \secboundgrad^4 h^3 \sum_{\ell=0}^{h-1} \CPE{ \norm{ \locparam{c}{r,\ell} - \globparam{r} }^4 }{\globfiltr{r}}
\eqsp.
\end{align}
Finally, we bound $T_4$ using gradient's boundedness \eqref{eqdef:uniform_grad_bound},
\begin{align}
\label{eq:bound-moment-four-decomposition-T4}
T_4 
& \le
\boundgrad^2 h^4 \norm{ \ngf{c}{\globparam{r}} }^2
\eqsp.
\end{align}
Plugging 
\eqref{eq:bound-moment-four-decomposition-T1}, 
\eqref{eq:bound-moment-four-decomposition-T2}, 
\eqref{eq:bound-moment-four-decomposition-T3}, 
\eqref{eq:bound-moment-four-decomposition-T4} in 
\eqref{eq:bound-moment-four-decomposition}, we obtain
\begin{align*}
\boldsymbol{\mathrm{V}_c^h}
& \le
4^3 \step^4 h^4 \boundgrad^2 \norm{ \ngf{c}{\globparam{r} }}^2
+ 4^3 \step^4 h^4 \bias^4
+ 4^3 \step^4 \secboundgrad^4 h^3 \sum_{\ell=0}^{h-1} \boldsymbol{\mathrm{V}_c^\ell}
+ 3 \cdot 12^3 \step^4 h^2 \moments{4}{4} 
\eqsp.
\end{align*}
Like for the terms $\boldsymbol{\mathrm{U}_c^h}$, we complete the sum and average over $h = 0$ to $\lsteps - 1$, which gives
\begin{align*}
\frac{1}{\lsteps} \sum_{h=0}^{\lsteps - 1}
\boldsymbol{\mathrm{V}_c^h}
& \le
\frac{4^3 \step^4 \secboundgrad^4 \lsteps^3(\lsteps-1)}{5 \lsteps} \sum_{h=0}^{\lsteps-1} \boldsymbol{\mathrm{V}_c^h}
+ \frac{3 \cdot 12^3 \step^4 \lsteps(\lsteps-1)}{3} \moments{4}{4} 
\\
& \quad
+ \frac{4^3 \step^4 \lsteps^2(\lsteps-1)^2}{5} \Big( \boundgrad^2 \norm{ \ngf{c}{\globparam{r} }}^2 + \bias^4 \Big)
\eqsp.
\end{align*}
Using $\step \lsteps \secboundgrad \le 1/6$, averaging over $c = 1$ to $\nagent$, collecting the terms in $\boldsymbol{\mathrm{V}_c^h}$ on the left hand side, and multiplying by $2$, we obtain
\begin{align}
\label{eq:bound-on-sum-of-V}
\frac{1}{\nagent \lsteps}
\sum_{c=1}^{\nagent}
\sum_{h=0}^{\lsteps - 1}
\boldsymbol{\mathrm{V}_c^h}
& \le
\frac{2 \cdot 4^3 \step^4 \lsteps^2(\lsteps-1)^2}{5} \Big\{ \bias^4
+ \boundgrad^2 \norm{ \ngf{c}{\globparam{r} }}^2 \Big\}
+ \frac{6 \cdot 12^3 \step^4 \lsteps(\lsteps-1)}{3} \moments{4}{4} 
\eqsp.
\end{align}

\textit{(Final Result.)}
Plugging \eqref{eq:bound-on-sum-of-V} back in \eqref{eq:bound-sum-drift-with-fourth-moment} and using $\thirddbound^2 \step^2 \lsteps^2 \boundgrad^2 \le \secboundgrad^2/32$ and $\bias \le \boundgrad$ gives
\begin{align*}
\nonumber
\frac{1}{\nagent \lsteps}
\sum_{c=1}^{\nagent}
\sum_{h=1}^{\lsteps-1}
\boldsymbol{\mathrm{U}_c^h}
& \le
\Big( 6 \step^2 \secboundgrad^2 \lsteps (\lsteps\!\!-\!1) 
+ \frac{4^4 \step^4 \thirddbound^2 \boundgrad^2 \lsteps^2(\lsteps-1)^2  }{5} \Big) \frac{1}{\nagent} \sum_{c=1}^\nagent \norm{ \ngf{c}{\globparam{r}} }[2]^2
\\
\nonumber
& \quad
+ \frac{12^4 \step^4 \thirddbound^2 \lsteps(\lsteps-1)}{3} \sigmafour^4
+ 6 \step^2 \secboundgrad^2 \lsteps(\lsteps\!\!-\!1) \bias^2
+ \frac{4^4 \step^4 \thirddbound^2  \lsteps^2(\lsteps-1)^2}{5} \bias^4
\\
\nonumber
& \le
\Big( 6 \step^2 \secboundgrad^2 \lsteps (\lsteps\!\!-\!1) 
+ 2 \step^2 \secboundgrad^2 (\lsteps-1)^2 \Big) \frac{1}{\nagent} \sum_{c=1}^\nagent \norm{ \ngf{c}{\globparam{r}} }[2]^2
\\ 
& \quad
+ 4 \cdot 12^3 \thirddbound^2 \step^4 \lsteps(\lsteps-1) \moments{4}{4} 
+ \Big( 6 \step^2 \secboundgrad^2 \lsteps(\lsteps\!\!-\!1)
+ 2 \step^2 \secboundgrad^2 (\lsteps-1)^2 \Big) \bias^2
\eqsp,
\end{align*}
and the result follows.
\end{proof}

\begin{lemma}[Bound on global iterates variance]
\label{lem:globparam-variance}
Assume \Cref{assumFL:uniform_grad_bound} to \Cref{assumFL:bias}.
Assume that  $\step \lsteps \secboundgrad \le 1/6$
Then the iterates of $\projfedAVG$ satisfy
\begin{align*}
\PE\Big[ \norm{ \avgparam{r+1} - \CPE{\avgparam{r+1}}{\globfiltr{r}} }^2 ]
& \le
\frac{3 \step^2 \lsteps \moments{2}{2}}{\nagent}
\eqsp.
\end{align*}
\end{lemma}
\begin{proof}
Since $\avgparam{r+1} = 1/\nagent \sum_{c=1}^\nagent \locparam{c}{r, \lsteps}$ and $\{\locparam{c}{r,\lsteps}\}_{c=1}^{\nagent}$ are independent conditional to $\globfiltr{r}$,
\begin{align*}
\PE\Big[ \bnorm{ \avgparam{r+1} - \CPE{\avgparam{r+1}}{\globfiltr{r}} }^2 \Big]
& =
\frac{1}{\nagent^2} 
\sum_{c=1}^\nagent
\PE\Big[ \norm{  \locparam{c}{r, \lsteps} - \CPE{ \locparam{c}{r, \lsteps} }{\globfiltr{r}} }^2 \Big]
\eqsp.
\end{align*}
Then, we have, for $h \in \{0, \dots, \lsteps-1\}$, using that $\CPE{ \grb{c}{\randState{c}[r,h+1]}(\locparam{c}{r,h}) }{\globfiltr{r}}= \CPE{ \egrb{c}(\locparam{c}{r,h}) }{\globfiltr{r}}$,
\begin{align*}
\boldsymbol{\mathrm{A}_c^{r,h+1}}\!
& \eqdef \PE\Big[ \norm{  \locparam{c}{r,h+1} - \CPE{ \locparam{c}{r,h+1} }{\globfiltr{r}} }^2 \Big]
\\
& =
\PE\Big[ \bnorm{  \locparam{c}{r,h} \!-\! \CPE{ \locparam{c}{r,h} }{\globfiltr{r}} 
\!+\! \step \big( 
 \grb{c}{\randState{c}[r,h+1]}(\locparam{c}{r,h}) 
\!-\! \egrb{c}(\locparam{c}{r,h}) 
\!+\! \egrb{c}(\locparam{c}{r,h}) 
\!-\! \CPE{ \egrb{c}(\locparam{c}{r,h}) }{\globfiltr{r}} \big) }^2 \Big]
\eqsp.
\end{align*}
Since $\{ \randState{c}[r,h] \}_{h=1}^{\lsteps}$ are independent conditional to $\globfiltr{r}$, we have, using \eqref{eqdef:variance},
\begin{align}
\boldsymbol{\mathrm{A}_c^{r,h+1}}=
\PE\Big[ \bnorm{  \locparam{c}{r,h} - \CPE{ \locparam{c}{r,h} }{\globfiltr{r}} 
+ \step \big( \egrb{c}(\locparam{c}{r,h}) 
- \CPE{ \egrb{c}(\locparam{c}{r,h}) }{\globfiltr{r}} \big) }^2 \Big]
 +
\step^2 \moments{2}{2} .
\label{eq:proof-variance-theta-rplusone-decomp-hplus1}
\end{align}
Then, by Young's inequality, we have
\begin{align*}
& \PE\Big[ \bnorm{  \locparam{c}{r,h} - \CPE{ \locparam{c}{r,h} }{\globfiltr{r}} 
+ \step \Big(  \egrb{c}(\locparam{c}{r,h})
- \CPE{ \egrb{c}(\locparam{c}{r,h}) }{\globfiltr{r}} \Big) }^2 \Big]
\\
& \le 
(1 + \step \secboundgrad)
\PE\Big[ \bnorm{  \locparam{c}{r,h} - \CPE{ \locparam{c}{r,h} }{\globfiltr{r}} }^2 \Big] 
+ (\step^2 + \step/\secboundgrad)
\PE\Big[ \bnorm{ \egrb{c}(\locparam{c}{r,h})
- \CPE{ \egrb{c}(\locparam{c}{r,h}) }{\globfiltr{r}}}^2 \Big]
\end{align*}
Finally, we have, by Young's inequality and  \eqref{eqdef:smoothness},
\begin{align*}
& 
\PE\Big[ \bnorm{ \egrb{c}(\locparam{c}{r,h}) 
- \CPE{ \egrb{c}(\locparam{c}{r,h}) }{\globfiltr{r}}}^2 \Big]
\\
& \le
2 
\PE\Big[ \bnorm{ \egrb{c}(\locparam{c}{r,h})
- \egrb{c}(\CPE{\locparam{c}{r,h}}{\globfiltr{r}}) }^2 \Big]
+ 2 
\PE\Big[ \bnorm{\egrb{c}(\CPE{\locparam{c}{r,h}}{\globfiltr{r}}) 
- \CPE{ \egrb{c}(\locparam{c}{r,h}) }{\globfiltr{r}}}^2 \Big]
\\
& \le
2 \secboundgrad^2
\PE\Big[ \bnorm{\locparam{c}{r,h} 
- \CPE{\locparam{c}{r,h}}{\globfiltr{r}} }^2 \Big]
+ 2  \secboundgrad^2
\PE\Big[ \bnorm{ \CPE{ \locparam{c}{r,h} }{\globfiltr{r}} - \locparam{c}{r,h} }^2 \Big]
\\
& \le
4 \secboundgrad^2
\PE\Big[ \bnorm{\locparam{c}{r,h} 
- \CPE{\locparam{c}{r,h}}{\globfiltr{r}} }^2 \Big]
\eqsp,
\end{align*}
where we used Jensen's inequality in the last inequality.
Then, notice that $4 (\step^2 + \step/\secboundgrad) \secboundgrad^2 = 4 (\step^2 \secboundgrad^2 + \step \secboundgrad) \le 5 \step \secboundgrad$ since $\step \secboundgrad \le 1/4$.
Plugging this in \eqref{eq:proof-variance-theta-rplusone-decomp-hplus1},  we obtain
\begin{align*}
\boldsymbol{\mathrm{A}_c^{r,h+1}}\le (1 + 6 \step \secboundgrad)
\boldsymbol{\mathrm{A}_c^{r,h}}+
\step^2 \moments{2}{2} 
\eqsp.
\end{align*}
And unrolling this inequality gives
\begin{align*}
& \PE\Big[ \norm{  \locparam{c}{r, \lsteps} - \CPE{ \locparam{c}{r, \lsteps} }{\globfiltr{r}} }^2 \Big]
\le
\step^2 
\sum_{h=0}^\lsteps (1 + 6 \step \secboundgrad)^h \moments{2}{2} 
\le
3 \step^2 \lsteps \moments{2}{2} 
\eqsp,
\end{align*}
where the second inequality comes from $\step \lsteps \secboundgrad \le 1/6$, which gives $(1 + 6 \step \secboundgrad)^h \le (1 + 1/\lsteps)^\lsteps \le 3$, and the lemma follows.
\end{proof}

\begin{lemma}[Descent Lemma] 
\label{lem:ascent_lemma}
Assume \Cref{assumFL:uniform_grad_bound} to \Cref{assumFL:bias}.
For any $\step >0$ such that $\step \lsteps \secboundgrad \leq 1/6$ and $32 \step^2 \lsteps^2 \thirddbound^2 \boundgrad^2 \le \secboundgrad^2$, the iterates of $\projfedAVG$ satisfy
\begin{align*}
\nonumber
- \CPE{ \f{\avgparam{r+1}} }{\globfiltr{r}}
& \leq
- \f{\globparam{r}}
- \frac{\step \lsteps}{4}\norm{ \gf{\globparam{r}}}[2]^2
+ \frac{3 \step^2 \secboundgrad \lsteps \moments{2}{2}}{2 \nagent}
\\
& \quad 
+ 2 \step \lsteps \bias^2
+ 8 \step^3 \secboundgrad^2 \lsteps^2(\lsteps-1) \hgty^2
+ 4 \cdot 12^3 \step^5 \thirddbound^2 \lsteps^2(\lsteps-1) \moments{4}{4} 
\eqsp.
\end{align*}
\end{lemma}
\begin{proof}
Smoothness of $\nfw{c}$ gives $\abs{ \f{\avgparam{r+1}} - \f{\globparam{r}} - \pscal{ \gf{\globparam{r}}}{\avgparam{r+1}-\globparam{r}}} \le (\secboundgrad/2) \norm{ \avgparam{r+1} - \globparam{r} }^2$, which implies that
\begin{align*}
- \f{\avgparam{r+1}} 
& 
\leq 
- \f{\globparam{r}}
- \pscal{\gf{\globparam{r}}}{\avgparam{r+1} - \globparam{r}} 
+ \frac{\secboundgrad}{2} \norm{\avgparam{r+1} - \globparam{r}}[2]^2 \eqsp.
\end{align*}
Let $\kappa = \frac{1}{\sqrt{\step \lsteps}} $. 
Taking the expectation conditionally on $\globfiltr{r}$ and using the polarization identity $2\pscal{a}{b} = \norm{a}[2]^2 + \norm{b}[2]^2 - \norm{a-b}[2]^2  $ for $a, b \in \rset^d$, we get
\begin{align}
\nonumber
- \CPE{ \f{\avgparam{r+1}} }{\globfiltr{r}} + \f{\globparam{r}}
& \leq 
- \pscal{\kappa^{-1}\gf{\globparam{r}}}{\kappa \CPE[]{\avgparam{r+1} - \globparam{r}}{\globfiltr{r}} } 
+ \frac{\secboundgrad}{2} \CPE[]{\norm{\avgparam{r+1} - \globparam{r}}[2]^2 }{\globfiltr{r}}
\\
\nonumber
& = 
- \frac{1}{2 \kappa^2}\norm{ \gf{\globparam{r}}}[2]^2
+ \underbrace{\frac{1}{2 \kappa^2} \norm{\gf{\globparam{r}} + \kappa^2 \CPE[]{ \globparam{r} - \avgparam{r+1} }{\globfiltr{r}} }[2]^2  }_{\term{A}} 
\\[-0.5em] \label{eq:main_general_alpha=1}
& \quad 
+ \underbrace{
\frac{\secboundgrad}{2} \CPE[]{\norm{\avgparam{r+1} - \globparam{r}}[2]^2}{\globfiltr{r}}
- \frac{\kappa^2}{2} \norm{ \CPE[]{\avgparam{r+1} - \globparam{r} }{\globfiltr{r}} }[2]^2}_{\term{B}} 
\eqsp.
\end{align}
The term $\term{A}$ is a drift term, that is due to local updates, and is due to heterogeneity, while the term $\term{B}$ is a second order term error term and a variance term.
We now bound each of these two terms.

\textbf{Bounding $\term{A}$.} 
Using the fact that $\fw = \frac{1}{\nagent} \sum_{c=1}^\nagent \nfw{c}$, the definition $\kappa^2 = 1/{\step \lsteps}$, the definition of $\avgparam{r+1}$ and Jensen's inequality, we have
\begin{align*}
\bnorm{\gf{\globparam{r}} + \kappa^2 \CPE{ \globparam{r} - \avgparam{r+1} }{\globfiltr{r}} }[2]^2
& = 
\bnorm{
\bCPE{\frac{1}{\nagent} \sum_{c=1}^{\nagent} \Big(\ngf{c}{\globparam{r}} - \frac{1}{\lsteps} \sum_{h=0}^{\lsteps-1} \ngbs{c}{\locparam{c}{r,h}}{\randState{c}[r,h+1]}\Big) }{\globfiltr{r}} }[2]^2
\\
& \leq 
\frac{1}{\lsteps \nagent} \sum_{c=1}^{\nagent}  \sum_{h=0}^{\lsteps-1} \bnorm{ \CPE{  \ngf{c}{\globparam{r}} -    \ngbs{c}{\locparam{c}{r,h}}{\randState{c}[r,h]} }{\globfiltr{r}} }[2]^2 \\
& = 
\frac{1}{\lsteps \nagent} \sum_{c=1}^{\nagent}  \sum_{h=0}^{\lsteps-1} \bnorm{ \CPE{  \ngf{c}{\globparam{r}} - \ngb{c}{\locparam{c}{r,h}} }{\globfiltr{r}} }[2]^2\eqsp,
\end{align*}
where the last equality holds by independence of $\randState{c}[r,h+1]$ and $\locfiltr{c}{r,h}$. By decomposing 
\begin{align*}
\ngf{c}{\globparam{r}} 
- \ngb{c}{\locparam{c}{r,h}}
& = 
\ngf{c}{\globparam{r}} - \ngf{c}{\locparam{c}{r,h}}  
+ \ngf{c}{\locparam{c}{r,h}} -  \ngb{c}{\locparam{c}{r,h}} 
\eqsp.
\end{align*}
Using Young's inequality and bounding the bias using \eqref{eqdef:bias}, we obtain
\begin{align*}
\bnorm{\gf{\globparam{r}} + \kappa^2 \CPE{ \globparam{r} - \avgparam{r+1} }{\globfiltr{r}} }[2]^2
& \leq  
\frac{2}{\lsteps \nagent} \sum_{c=1}^{\nagent}  \sum_{h=0}^{\lsteps-1} \norm{  \CPE{   \ngf{c}{\globparam{r}} -    \ngf{c}{\locparam{c}{r,h}}  }{\globfiltr{r}} }[2]^2 
+ 2 \bias^2 \eqsp.
\end{align*}
Using \Cref{lem:bound-expec-drift} to bound the first term, and multiplying by $1/(2\kappa^2) = \step \lsteps / 2$, we obtain
\begin{equation}
\label{eq:bound-termA-first}
\begin{aligned}
\term{A}
&\leq
\frac{8 \step^3 \secboundgrad^2 \lsteps^2 (\lsteps-1)}{\nagent} \sum_{c=1}^\nagent \norm{ \ngf{c}{\globparam{r}} }[2]^2
\\ 
& \quad
+ 4 \cdot 12^3 \cdot 2 \thirddbound^2 \step^5 \lsteps^2(\lsteps-1) \moments{4}{4} 
+ (1 + 8 \step^2 \secboundgrad^2 \lsteps(\lsteps-1)) \step \lsteps \bias^2
\eqsp.
\end{aligned}
\end{equation}

\paragraph{Bounding $\term{B}$.}
We decompose $\term{B}$ by writing $\avgparam{r+1} = \CPE[]{ \avgparam{r+1} }{\globfiltr{r}} + \avgparam{r+1} - \CPE[]{ \avgparam{r+1} }{\globfiltr{r}}$, which gives
\begin{align*}
\term{B}
& =
\frac{\secboundgrad}{2}  \CPE[]{ \norm{\CPE[]{ \avgparam{r+1} }{\globfiltr{r}} - \avgparam{r+1} }^2 }{\globfiltr{r}}
+
\frac{\secboundgrad}{2} \norm{\CPE[]{\avgparam{r+1} - \globparam{r}}{\globfiltr{r}}}[2]^2
- \frac{\kappa^2}{2} \norm{ \CPE[]{\avgparam{r+1} - \globparam{r} }{\globfiltr{r}} }[2]^2 
\\
& =
\frac{\secboundgrad}{2}  \CPE[]{ \norm{\CPE[]{ \avgparam{r+1} }{\globfiltr{r}} - \avgparam{r+1} }^2 }{\globfiltr{r}}
+
\Big( \frac{\secboundgrad}{2} - \frac{\kappa^2}{2} \Big)
\norm{ \CPE[]{\avgparam{r+1} - \globparam{r} }{\globfiltr{r}} }[2]^2 
\eqsp.
\end{align*}
Since $\step \lsteps \secboundgrad \le 1$, we have $\frac{\secboundgrad}{2}-  \frac{\kappa^2}{2} \leq  \frac{\secboundgrad}{2} - \frac{1}{2 \step \lsteps} \le 0$, and the second term is negative.
The second term is a variance term, that we bound using \Cref{lem:globparam-variance}, which gives
\begin{align}
\label{eq:bound-termb-in-ascent-lemma}
\term{B} \le \frac{3 \step^2 \secboundgrad \lsteps \moments{2}{2}}{2 \nagent}
\eqsp.
\end{align}

\textbf{Bound on \eqref{eq:main_general_alpha=1}. }
Plugging in the bounds \eqref{eq:bound-termA-first} and \eqref{eq:bound-termb-in-ascent-lemma} on $\term{A}$ and $\term{B}$ in \eqref{eq:main_general_alpha=1} yields
\begin{align}
\nonumber
- \CPE{ \f{\avgparam{r+1}} }{\globfiltr{r}} + \f{\globparam{r}}
& \leq
- \frac{\step \lsteps}{2}\norm{ \gf{\globparam{r}}}[2]^2
 + \frac{8 \step^3 \secboundgrad^2 \lsteps^2 (\lsteps-1)}{\nagent} \sum_{c=1}^\nagent \norm{ \ngf{c}{\globparam{r}} }[2]^2
\\ 
\label{eq:bound-lemma-ascent-last-ineq}
& \quad
+ 4 \cdot 12^3 \step^5 \thirddbound^2  \lsteps^2(\lsteps-1) \moments{4}{4} 
+ 2 \step \lsteps \bias^2
+ \frac{3 \step^2 \secboundgrad \lsteps \moments{2}{2}}{2 \nagent}
\eqsp,
\end{align}
where we used $\step \lsteps \secboundgrad \le 1/6$ to bound $(1 + 8 \step^2 \secboundgrad^2 \lsteps(\lsteps-1)) \step \lsteps \bias^2 \le 2 \step \lsteps \bias^2$.
Moreover, we have $8 \step^2 \secboundgrad^2 \lsteps^2 \le 1/4$ and $\frac{1}{\nagent} \sum_{c=1}^\nagent \norm{ \ngf{c}{\globparam{r}}}^2 \le \norm{ \gf{\globparam{r}} }^2 + \hgty^2$, which gives the bound
\begin{align*}
 \frac{8 \step^3 \secboundgrad^2 \lsteps^2 (\lsteps-1)}{\nagent} \sum_{c=1}^\nagent \norm{ \ngf{c}{\globparam{r}} }[2]^2
 \le 
 \frac{\step \lsteps}{4}\norm{ \gf{\globparam{r}}}[2]^2
 + 8 \step^3 \secboundgrad^2 \lsteps^2 (\lsteps-1) \hgty^2
 \eqsp,
\end{align*}
and the result of the lemma follows from plugging this inequality in \eqref{eq:bound-lemma-ascent-last-ineq}.
\end{proof}
\subsection{Convergence under Local Non-Uniform Łojasiewicz inequality} 
\label{sec:conv-loc-null}
Firstly, define
\begin{align*}
\flstar[c] = \sup_{\param \in \rset^d} \nf{c}{\param} \eqsp, \quad \Fstar = \frac{1}{\nagent} \sum_{c=1}^{\nagent} \flstar[c] \eqsp. 
\end{align*}

\paragraph{Convergence under local 'quadratic' Non-Uniform Łojasiewicz inequalities}
For any $(c, \param) \in [\nagent] \times \rset^{d}$ , define 
\begin{align*}
\mufl[c](\param) := \sup\{x \in \rset^{+}, \norm{\ngf{c}{\param}}[2]^2 \geq 2 x \left(\flstar[c] - \nf{c}{\param}\right)^2\} \eqsp.
\end{align*}
We assume the following additional condition
\begin{assumQLA}
\label{assum:local_quadratic_pl_appendix}
For any $c \in [\nagent]$, we have $\flstar[c]<\infty$. Additionnally, for any $c\in [\nagent], $and $\param \in \rset^{d}$, there exists $\mufl[c](\theta)>0$ such that
$\norm{\ngf{c}{\param}}[2]^2 \geq 2\mufl[c](\theta) \left(\flstar[c] - \nf{c}{\param}\right)^2$.
\end{assumQLA}
For any parameter $\theta \in \rset^d$, define
$\mufl(\theta) \eqdef \min_{c \in [\nagent]}\mufl[c](\theta)$.
\begin{assumQLA}
\label{assum:improvement_quadratic_pl_appendix}
For any $\param \in \rset^{d}$, it holds that $\f{\projset(\param)} \geq \f{\param}$.
Additionally, there exists $\minmufl>0$, such that we have $\mufl(\projset(\param)) \geq \minmufl$.
\end{assumQLA}
Under these two additional assumptions, we can derive global convergence rates for $\projfedAVG$. We preface the proof with two elementary Lemmas. 
\begin{lemma}
\label{lem:convergence_special_seq}
Let $(w_r)_{r=0}^\infty$ be a sequence of positive real numbers, and let $\kappa>0$, $B>0$.  Assume that for all $r\ge0$,
$$
  w_{r+1} \le w_r - \kappa\,w_r^2 + B.
$$
Then for every integer $r\ge0$ one has
$$
  w_r \le \sqrt{\frac{B}{\kappa}} + B + \frac{w_0}{1 + \kappa\,r\,w_0}.
$$
\end{lemma}

\begin{proof}
Set $M = \sqrt{B/\kappa}$ and fix $r \in \nset$. We split into two cases:

\textbf{Case 1:} $w_k > M$ for all $ k \in \{0, \dots r \}$. Define $v_k \eqdef w_k - M $ which is positive as  $w_k > M$.  Then for any $k \in \{0, \dots r \}$, it holds that
$$
  v_{k+1} = w_{k+1} - M \le w_k - M - \kappa (w_k-M + M)^2 + B
  \leq  v_k - \kappa v_k^2 \eqsp,
$$
where in the last inequality, we used that for any $a,b \geq 0$, we have $(a+b)^2 \geq a^2 + b^2$. Dividing the preceding inequality by $v_k^{2}$ yields
\begin{align}
\label{eq:linearise_v_r_1_2}
\frac{v_{k+1} -v_k}{v_k^{2}} \leq -\kappa \eqsp.
\end{align}
For $x>0$, define $g(x) = x^{-1}$. By convexity of $g$ on $\rset_{+}^{\star}$, we have $g(v_{k+1}) \geq g(v_{k})+ (v_{k+1} - v_{k}) g'(v_{k})$ which can be rewritten as
\begin{align*}
v_{k+1}^{-1} \geq v_{k}^{-1}- \left(v_{k+1}- v_{k}\right) \frac{1}{v_{k}^{2}} \eqsp,
\end{align*}
and which implies, after using \eqref{eq:linearise_v_r_1_2}
\begin{align*}
 v_{k}^{-1} - v_{k+1}^{-1} \leq  \frac{v_{k+1}- v_{k}}{v_{k}^{2}} \leq -\kappa \eqsp.
\end{align*}
Summing up both sides over $k = 0 \dots r$ and rearranging the terms yields
\begin{align*}
(w_{r} - M)^{-1} \geq \kappa r + w_{0}^{-1}  \eqsp.
\end{align*}
Finally, we get
\begin{align*}
w_r \leq  M + \frac{w_{0}}{1 + \kappa r w_{0}} \eqsp.
\end{align*}
\textbf{Case 2:} There exists some $0\le r_0\le r$ with $w_{r_0}\le M$. Let us prove that for any $0 \leq x \leq M+B $, it holds that $0 \leq x-\kappa x^2 + B \leq M+ B$. We distinguish two sub-cases. First, if $x \leq M$ then it holds that $x-\kappa x^2 + B \le x + B \le M + B$. Alternatively, if $M \le x\le M + B$ then $x-\kappa x^2 + B \le x-\kappa M^2 + B  = x \le M+ B$. Finally, using the preceding inequality combined with an immediate recursion proves that for all $k\ge r_0$, we have $w_k \le B + M$.
\end{proof}
\begin{lemma}
\label{lem:pl_structure_global_quadratic}
Assume \Cref{assumFL:heterogeneity} and \Cref{assum:local_quadratic_pl_appendix}. For any $\theta \in \rset^d$, it holds that
\begin{align*}
\hgty^{2} + \norm{\nabla \f{\theta}}[2]^{2}  \geq \mufl(\theta) (\Fstar- \f{\theta})^{2}
\eqsp.
\end{align*}
\end{lemma}
\begin{proof}
Let $\theta \in \rset^{d}$. Using \Cref{assum:local_quadratic_pl_appendix}, we have for any $c \in [\nagent] $
\begin{align*}
\sqrt{ \min_{c \in [\nagent]}2\mufl[c](\theta)} \left[ \flstar[c]- \nf{c}{\theta}\right] &\leq \sqrt{ 2\mufl[c](\param)} \left[ \flstar[c] - \nf{c}{\theta}\right] 
 \leq \norm{\ngf{c}{\theta}}[2] 
 \eqsp.
 \end{align*}
We then decompose $\ngf{c}{\theta} = \ngf{c}{\theta} - \nabla \f{\theta} + \nabla \f{\theta}$ and use triangle inequality and \Cref{assumFL:heterogeneity} to bound
 \begin{align*}
 \norm{\ngf{c}{\theta}}[2] 
 \leq \norm{\ngf{c}{\theta} - \nabla \f{\theta}}[2] + \norm{\nabla \f{\theta}}[2] \leq \hgty + \norm{\nabla \f{\theta}}[2]
 \eqsp.
\end{align*}
Averaging the resulting inequality over all the agents, taking the square, and applying Young's inequality concludes the proof.
\end{proof}

\begin{theorem}
[Convergence rates of $\projfedAVG$] 
\label{theorem:convergence_pl_type_square}
Assume \Cref{assumFL:uniform_grad_bound} to \Cref{assumFL:bias}, \Cref{assum:local_quadratic_pl_appendix} and \Cref{assum:improvement_quadratic_pl_appendix}. For any $\step >0$ such that $\step \lsteps \secboundgrad \leq 1/6$ and $32 \step^2 \lsteps^2 \thirddbound^2 \boundgrad^2 \le \secboundgrad^2$, the iterates of $\projfedAVG$ satisfy
\begin{align*}
\Fstar - \PE[\f{\globparam{\trounds}}] &\leq \frac{\Fstar- \f{\globparam{0}}}{1+  \trounds \cdot (\Fstar - \f{\globparam{0}}) \cdot \left(\step \lsteps \minmufl /4 \right)}  + \left( \frac{6\step \secboundgrad  \moments{2}{2}}{\nagent\minmufl} \right)^{1/2}  \\
& + \left( \frac{16 \cdot12^3 \step^4 \thirddbound^2 \lsteps(\lsteps-1) \moments{4}{4}}{\minmufl} \right)^{1/2}  +\left(\frac{2\hgty^2}{\minmufl} \right)^{1/2}  +\left( \frac{8\bias^2}{\minmufl} \right)^{1/2}  \\
& + \frac{\hgty^2}{12\secboundgrad}
+ \frac{\step \moments{2}{2}}{4 \nagent}
+ \frac{\bias^2 }{3 \secboundgrad} 
+ \frac{12^3 \step^4 \thirddbound^2 \lsteps(\lsteps-1) \moments{4}{4}}{\secboundgrad}\eqsp.
\end{align*}
\end{theorem}
\begin{proof}
Firstly, using \Cref{assum:improvement_quadratic_pl_appendix} note by an immediate recursion that
\begin{align*}
\inf_{r\geq 0} \mufl(\globparam{r}) \geq \minmufl \eqsp.
\end{align*}
Applying \Cref{lem:ascent_lemma} yields
\begin{align*}
- \CPE{ \f{\avgparam{r+1}} }{\globfiltr{r}}
& \leq
- \f{\globparam{r}}
- \frac{\step \lsteps}{4}\norm{ \gf{\globparam{r}}}[2]^2
+ \frac{3 \step^2 \secboundgrad \lsteps \moments{2}{2}}{2 \nagent}
\\
& \quad 
+ 2 \step \lsteps \bias^2
+ 8 \step^3 \secboundgrad^2 \lsteps^2(\lsteps-1) \hgty^2
+ 4 \cdot 12^3 \step^5 \thirddbound^2 \lsteps^2(\lsteps-1) \moments{4}{4} \eqsp.
\end{align*}
Adding $\Fstar$, and using \Cref{lem:pl_structure_global_quadratic} combined with \Cref{assum:improvement_quadratic_pl_appendix} yields
\begin{align*}
\Fstar - \CPE{ \f{\globparam{r+1}} }{\globfiltr{r}}
& \leq
\Fstar - \f{\globparam{r}}
- \frac{\step \lsteps \minmufl}{4}\left(\Fstar - \f{\globparam{r}}\right)^2 + \frac{\step \lsteps}{4}\hgty^2
+ \frac{3 \step^2 \secboundgrad \lsteps \moments{2}{2}}{2 \nagent}
\\
& \quad 
+ 2 \step \lsteps \bias^2
+ 8 \step^3 \secboundgrad^2 \lsteps^2(\lsteps-1) \hgty^2
+ 4 \cdot 12^3 \step^5 \thirddbound^2 \lsteps^2(\lsteps-1) \moments{4}{4} \eqsp.
\end{align*}
Taking the expectation with respect to all the stochasticity, applying Jensen's inequality, and using that $\step \lsteps \secboundgrad \leq 1/6$ to simplify the heterogeneity terms gives
\begin{align*}
\nonumber
\seq{r+1}  \leq \seq{r} -  \kappa (\seq{r})^2 +  \biasthm \eqsp,
\end{align*}
where we defined $\seq{r} = \Fstar - \PE[\f{\theta^{r}}]$, $\kappa = \frac{\step \lsteps \minmufl}{4}$, and 
\begin{align*}
\biasthm = \frac{\step \lsteps}{2} \hgty^2
+ \frac{3 \step^2 \secboundgrad \lsteps \moments{2}{2}}{2 \nagent}
+ 2 \step \lsteps \bias^2  
+ 4 \cdot 12^3 \step^5 \thirddbound^2 \lsteps^2(\lsteps-1) \moments{4}{4} \eqsp.
\end{align*}
Finally, applying \Cref{lem:convergence_special_seq} on the sequence $\seq{r}$ concludes the proof.
\end{proof}

\begin{corollary}[Sample and Communication Complexity] 
\label{thm:complexity-general-fl_ql}
Under the assumptions of \Cref{theorem:convergence_pl_type_square},
let 
\begin{align*}
\epsilon >  \frac{12\hgty}{\minmufl^{1/2}}  + \frac{18\bias}{\minmufl^{1/2}}  
 + \frac{\hgty^2}{2\secboundgrad}
+ \frac{2\bias^2 }{\secboundgrad} \eqsp,
\end{align*}
and
\begin{align*}
\step \le \min\Big( \frac{1}{6\secboundgrad}, \frac{\minmufl \nagent \epsilon^2}{216 \secboundgrad \moments{2}{2}}, \frac{\mu^{1/2} \epsilon \secboundgrad}{13^2 \thirddbound \moments{4}{2}},\frac{2\epsilon\nagent}{\moments{2}{2}}, \frac{\epsilon^{1/2}\secboundgrad^{3/2}}{24 \thirddbound \moments{4}{2}} \Big) \eqsp.
\end{align*}
In this case $\projfedAVG$ achieves $\Fstar - \PE[\f{\theta^{\trounds}}]$, with a number of communication
\begin{align*}
\trounds \ge 
\frac{144\left[\Fstar - \f{\theta^{0}}  - \epsilon/6\right]}{(\Fstar - \f{\theta^{0}}) \minmufl \epsilon} 
\max\Big(
\secboundgrad,
\frac{\thirddbound \boundgrad}{\secboundgrad}
\Big)
\eqsp,
\end{align*}
for a total number of samples per agent of
\begin{align*}
\trounds \lsteps 
\ge 
\frac{144[\Fstar - \f{\theta^{0}}  - \epsilon/6]}{(\Fstar - \f{\theta^{0}}) \minmufl \epsilon} \max\Big( \secboundgrad, \frac{36 \secboundgrad \moments{2}{2}}{\minmufl \nagent \epsilon^2}, \frac{29 \thirddbound \moments{4}{2}}{\mu^{1/2} \epsilon \secboundgrad},\frac{\moments{2}{2}}{12\nagent\epsilon}, \frac{4 \thirddbound \moments{4}{2}}{\epsilon^{1/2}\secboundgrad^{3/2}} \Big)
\eqsp.
\end{align*}
\end{corollary}
\begin{proof}
First, we require (i) that $\step \secboundgrad \leq 1/6$, and that (ii) each variance terms to be smaller than $\epsilon/6$, which gives the condition on the step size
\begin{align}
\label{eq:sample-complexity-fl:step-size-condition_ql}
\step \le \min\Big( \frac{1}{6\secboundgrad}, \frac{\minmufl \nagent \epsilon^2}{216 \secboundgrad \moments{2}{2}}, \frac{\mu^{1/2} \epsilon \secboundgrad}{13^2 \thirddbound \moments{4}{2}},\frac{2\epsilon\nagent}{\moments{2}{2}}, \frac{\epsilon^{1/2}\secboundgrad^{3/2}}{24 \thirddbound \moments{4}{2}} \Big)
\eqsp.
\end{align}
Then, $\lsteps$ has to satisfy $\step \lsteps \secboundgrad \le 1/6$ and $32 \step^2 \lsteps^2 \thirddbound^2 \boundgrad^2 \le \secboundgrad^2$, which requires
\begin{align*}
\lsteps
\le
\frac{1}{\step}
\min\Big(
\frac{1}{6\secboundgrad},
\frac{\secboundgrad}{6 \thirddbound \boundgrad}
\Big)
\eqsp.
\end{align*}
Finally, we require that the number of communications is at least
\begin{align*}
\trounds \ge 
\frac{\Fstar - \f{\theta^{0}}  - \epsilon/6}{(\Fstar - \f{\theta^{0}}) \step \lsteps \minmufl \epsilon/24} 
=
\frac{144\left[\Fstar - \f{\theta^{0}}  - \epsilon/6\right]}{(\Fstar - \f{\theta^{0}}) \minmufl \epsilon} 
\max\Big(
\secboundgrad,
\frac{\thirddbound \boundgrad}{\secboundgrad}
\Big)
\eqsp.
\end{align*}
The sample complexity follows from $\trounds \lsteps \ge 
\frac{\Fstar - \f{\theta^{0}}  - \epsilon/6}{(\Fstar - \f{\theta^{0}}) \step \minmufl \epsilon/24} $ and \eqref{eq:sample-complexity-fl:step-size-condition_ql}.
\end{proof}

\paragraph{Convergence under local 'linear' Non-Uniform Łojasiewicz inequalities}
For any $(c, \param) \in [\nagent] \times \rset^{d}$ , define 
\begin{align*}
\lmufl[c](\param) := \sup\{x \in \rset^{+}, \norm{\ngf{c}{\param}}[2]^2 \geq 2 x \left(\flstar[c] - \nf{c}{\param}\right)\} \eqsp.
\end{align*}
We assume the following additional condition
\begin{assumLL}
\label{assum:local_linear_pl_appendix}
For any $c \in [\nagent]$, we have $\flstar[c]<\infty$. Additionnally, for any $c\in [\nagent], $and $\param \in \rset^{d}$, there exists $\lmufl[c](\theta)>0$ such that
\begin{align*}
\norm{\ngf{c}{\param}}[2]^2 \geq 2\lmufl[c](\theta) \left(\flstar[c] - \nf{c}{\param}\right) \eqsp.
\end{align*}
\end{assumLL}
For any parameter $\theta \in \rset^d$, define
\begin{align*}
\lmufl(\theta) \eqdef \min_{c \in [\nagent]}\lmufl[c](\theta)\eqsp. 
\end{align*}
\begin{assumLL}
\label{assum:improvement_linear_pl_appendix}
For any $\param \in \rset^{d}$, it holds that
\begin{align*}
\f{\projset(\param)} \geq \f{\param}\eqsp.
\end{align*}
Additionally, there exists $\minlmufl>0$, such that we have $\lmufl(\projset(\param)) \geq \minlmufl$.
\end{assumLL}
Under these two additional assumptions, we can derive global convergence rates for $\projfedAVG$. We preface the proof with an elementary Lemma. 
\begin{lemma}
\label{lem:pl_structure_global_linear} Assume \Cref{assumFL:heterogeneity} and \Cref{assum:local_linear_pl_appendix}. For any $\theta \in \rset^{d}$, it holds that
\begin{align*}
\hgty^{2} + \norm{\nabla \f{\theta}}[2]^{2}  \geq 2 \lmufl(\theta) (\Fstar- \f{\theta}) \eqsp.
\end{align*}
\end{lemma}
\begin{proof}
Let $\theta \in \rset^{d}$. Using \Cref{assum:local_linear_pl_appendix} and the triangle inequality, we have for any $c \in [\nagent] $
\begin{align*}
 \min_{c \in [\nagent]}&2\lmufl[c](\theta)[ \flstar[c] - \nf{c}{\param}] \\
 &\leq  2\lmufl[c](\theta)[ \flstar[c] - \nf{c}{\param}] 
 \leq \norm{\ngf{c}{\param}}[2]^2 \\
 &= \norm{\ngf{c}{\param} - \nabla \f{\theta}+ \nabla \f{\theta}}[2]^2 \\
 & =  \norm{\ngf{c}{\param} - \nabla \f{\theta}}[2]^2 + 2\pscal{\ngf{c}{\param} -\nabla \f{\theta}}{\nabla \f{\theta}} + \norm{\nabla \f{\theta}}[2]^{2}
 \\
 & \leq \hgty^2 + 2\pscal{\ngf{c}{\param} -\nabla \f{\theta}}{\nabla \f{\theta}} + \norm{\nabla \f{\theta}}[2]^{2} \eqsp,
\end{align*}
where in the last inequality we used \Cref{assumFL:heterogeneity}. Finally, averaging the preceding inequality over all the agents concludes the proof.
\end{proof}
\begin{theorem}[Convergence rates of $\projfedAVG$] 
\label{theorem:sto_fedavg_alpha_2}
Assume \Cref{assumFL:uniform_grad_bound} to \Cref{assumFL:bias}, \Cref{assum:local_linear_pl_appendix} and \Cref{assum:improvement_linear_pl_appendix}. For any $\step >0$ such that $\step \lsteps \secboundgrad \leq 1/6$ and $32 \step^2 \lsteps^2 \thirddbound^2 \boundgrad^2 \le \secboundgrad^2$, the iterates of $\projfedAVG$ satisfy
\begin{align*}
\Fstar - \PE \left[\f{\globparam{\trounds}}\right] &\leq \left(1 -\frac{\step \lsteps \minlmufl }{2}\right)^{\trounds} \! (\Fstar - \f{\globparam{0}}) + \frac{3 \step \secboundgrad \moments{2}{2}}{ \nagent \minlmufl } + \frac{\hgty^2}{\minlmufl } \\
& + 4 \frac{\bias^2}{\minlmufl }
+ \frac{8^2 \cdot 12 \step^4 \thirddbound^2 \lsteps(\lsteps-1) \moments{4}{4}}{\minlmufl \secboundgrad }
\eqsp.
\end{align*}
\end{theorem}
\begin{proof}
Firstly, using \Cref{assum:improvement_linear_pl_appendix} note by an immediate recursion that
\begin{align*}
\inf_{r\geq 0} \lmufl(\globparam{r}) \geq \minlmufl \eqsp.
\end{align*}
Applying \Cref{lem:ascent_lemma} yields
\begin{align*}
- \CPE{ \f{\avgparam{r+1}} }{\globfiltr{r}}
& \leq
- \f{\globparam{r}}
- \frac{\step \lsteps}{4}\norm{ \gf{\globparam{r}}}[2]^2
+ \frac{3 \step^2 \secboundgrad \lsteps \moments{2}{2}}{2 \nagent}
\\
& \quad 
+ 2 \step \lsteps \bias^2
+ 8 \step^3 \secboundgrad^2 \lsteps^2(\lsteps-1) \hgty^2
+ 8 \cdot 12^2 \step^5 \thirddbound^2 \lsteps^2(\lsteps-1) \moments{4}{4} \eqsp.
\end{align*}
Adding $\Fstar$, and using \Cref{lem:pl_structure_global_linear} combined with \Cref{assum:improvement_linear_pl_appendix} yields
\begin{align*}
\Fstar - \CPE{ \f{\globparam{r+1}} }{\globfiltr{r}}
& \leq
\Fstar - \f{\globparam{r}}
- \frac{\step \lsteps\minlmufl}{2} (\Fstar- \f{\theta}) + \frac{\step \lsteps}{4} \hgty^2 +
+ \frac{3 \step^2 \secboundgrad \lsteps \moments{2}{2}}{2 \nagent}
\\
& \quad 
+ 2 \step \lsteps \bias^2
+ 8 \step^3 \secboundgrad^2 \lsteps^2(\lsteps-1) \hgty^2
+ 8 \cdot 12^2 \step^5 \thirddbound^2 \lsteps^2(\lsteps-1) \moments{4}{4} \eqsp.
\end{align*}
Taking the expectation with respect to all the stochasticity, yield
\begin{align*}
\Fstar - \PE \left[\f{\globparam{r+1}}\right]
& \leq
\left( 1 - \frac{\step\lsteps \minlmufl}{2}\right)\left(\Fstar - \PE \left[\f{\globparam{r}}\right]\right)
+ \frac{\step \lsteps}{4} \hgty^2
+ \frac{3 \step^2 \secboundgrad \lsteps \moments{2}{2}}{2 \nagent}
\\
& \quad 
+ 2 \step \lsteps \bias^2
+ 8 \step^3 \secboundgrad^2 \lsteps^2(\lsteps-1) \hgty^2
+ 8 \cdot 12^2 \step^5 \thirddbound^2 \lsteps^2(\lsteps-1) \moments{4}{4} \eqsp.
\end{align*}
The result follows from unrolling the recursion.
\end{proof}
\begin{corollary}[Sample and Communication Complexity of $\projfedAVG$] 
\label{thm:complexity-linear-pl}
Under the assumptions of \Cref{theorem:sto_fedavg_alpha_2}, let
\begin{align*}
\epsilon> \frac{4\hgty^2}{\minlmufl} + \frac{16 \bias^2}{\minlmufl} \eqsp,
\end{align*}
and
\begin{align*}
\step \le \min\Big( \frac{1}{6\secboundgrad}, \frac{\minlmufl  \epsilon \nagent}{12 \secboundgrad \moments{2}{2}}, \frac{\minlmufl^{1/2}\secboundgrad^{3/2} \epsilon^{1/2}}{5\thirddbound \moments{4}{2}} \Big)
\eqsp.
\end{align*}
Then $\projfedAVG$ achieves $\Fstar - \PE[\f{\theta^{\trounds}}]$, with a number of communication
\begin{align*}
\trounds \ge 
\frac{12}{ \minlmufl }
\log\Big(
\frac{4(\Fstar - \f{\theta^{0}})}{\epsilon}
\Big) \max\left( \secboundgrad, \frac{\thirddbound \boundgrad}{\secboundgrad}\right)
\eqsp,
\end{align*}
for a total number of samples per agent of
\begin{align*}
\trounds \lsteps 
\ge 
\frac{2}{\minlmufl } 
\log\Big(
\frac{4\Fstar - \f{\theta^{0}})}{\epsilon}
\Big)\max\Big( 6\secboundgrad, \frac{12 \secboundgrad \moments{2}{2}}{\minlmufl  \epsilon \nagent}, \frac{5\thirddbound \moments{4}{2}}{\minlmufl^{1/2}\secboundgrad^{3/2} \epsilon^{1/2}} \Big)
\eqsp.
\end{align*}
\end{corollary}
\begin{proof}
Firstly, we require (i) that $\step \secboundgrad\le 1/6$, and that (ii) each variance terms to be smaller than $\epsilon/4$, which gives the condition on the step size
\begin{align}
\label{eq:sample-complexity-regsoft:step-size-condition_ll}
\step \le \min\Big( \frac{1}{6\secboundgrad}, \frac{\minlmufl  \epsilon \nagent}{12 \secboundgrad \moments{2}{2}}, \frac{\minlmufl^{1/2}\secboundgrad^{3/2} \epsilon^{1/2}}{5\thirddbound \moments{4}{2}} \Big)
\eqsp,
\end{align}
Then, $\lsteps$ has to satisfy $\step \lsteps \leq 1/6\secboundgrad$ and $32 \step^2 \lsteps^2 \thirddbound^2 \boundgrad^2 \le \secboundgrad^2$.
This requires
\begin{align*}
\lsteps
\le
\frac{1}{\step}
\min\Big(
\frac{1}{6\secboundgrad},
\frac{\secboundgrad}{6 \thirddbound \boundgrad}
\Big)
\eqsp.
\end{align*}
Finally, we require that the number of communications is at least
\begin{align*}
\trounds \ge 
\frac{2}{\step \lsteps \minlmufl } 
\log\Big(
\frac{4(\Fstar - \f{\theta^{0}})}{\epsilon}
\Big) 
=
\frac{12}{ \minlmufl }
\log\Big(
\frac{4(\Fstar - \f{\theta^{0}})}{\epsilon}
\Big) \max\left( \secboundgrad, \frac{\thirddbound \boundgrad}{\secboundgrad}\right)
\eqsp.
\end{align*}
The sample complexity follows from $\trounds \lsteps \ge 
\frac{2}{\step\minlmufl } 
\log\Big(
\frac{4\Fstar - \f{\theta^{0}})}{\epsilon}
\Big) $ and \eqref{eq:sample-complexity-regsoft:step-size-condition_ll}.
\end{proof}

%% file: AISTATS/appendix/S_FedPG.tex
\label{secappendix:softfedpg}
$\SoftfedPG$ can be interpreted as a specific instance of $\projfedAVG$, where the projection set is the identity function, \ie $\mathcal{T} \colon \param \rightarrow \param$, the local objective is defined as $f_c = \reglocfunc[c]$, and the agent data distribution $\sampledist{c}{\param}$ corresponds to $\left[\softsampledist{c}{\theta}\right]^{\otimes \sizebatch}$, where $\softsampledist{c}{\theta}$ is the distribution induced by sampling $\sizebatch$ truncated trajectories from the policy $\policy_\theta$, defined by 
\begin{align}
\label{eq:truncated_trajectory_dist_def}.
\softsampledist{c}{\param;z} = \initdist(s^{0}) \policy_{\param}(a^{0} | s^{0})  \prod_{t=0}^{\lentrunc-1} \kerMDP(s^{t}\mid s^{t-1}, a^{h-1}) \policy(a^{h} |s^{h}) \eqsp.
\end{align}

Given a parameter $\theta \in \logitspace$ and an observation $Z_c \sim \softsampledist{c}{\theta}$, we recall the form of the biased estimator (defined in \eqref{eq:expression_of_stochastic_gradient_softmax_fedpg}) for the stochastic gradient:
\begin{align}
\label{eq:expression_of_stochastic_gradient_softmax_fedpg_appendix}
\softgrb{c}{Z_c}(\theta)
\eqdef \frac{1}{\sizebatch}  \sum_{b=1}^{\sizebatch} \sum_{t=0}^{\lentrunc-1} \discount^{t}  \left( \sum_{\ell =0}^t \nabla \log \policy_{\theta}(\varaction{c,b}{\ell} \mid \varstate{c,b}{\ell}) \right)   \rewardMDP[c](\varstate{c,b}{t}, \varaction{c,b}{t}) \eqsp.
\end{align}
Define also
\begin{align}
\label{eq:expression_of_expected_stochastic_gradient_softmax_fedpg_appendix}
\softegrb{c}(\theta) = \PE_{Z_c \sim \softsampledist{c}{\param}}[\softgrb{c}{Z_c}(\theta)].
\end{align}
To apply the convergence results of \Cref{secapp:ascent_lemma}, it remains to verify that Assumptions~\Cref{assumFL:uniform_grad_bound} to \Cref{assumFL:bias} and \Cref{assum:local_quadratic_pl_appendix} hold (\Cref{assum:improvement_quadratic_pl_appendix} will be assumed to hold for $\SoftfedPG$). We establish these conditions in the following.

\subsection{Checking the assumptions}
\label{subsec:satisfying_assumptions_ascent_softfedpg}

For a given policy $\policy$ and agent $c \in [\nagent]$, the value function $\valuefunc[c][\policy] \colon \S \to \rset$, is defined as:
\begin{align}
\label{def:value_func}
\valuefunc[c][\policy](s) \!= \!\CPE[\policy]{\sum_{t=0}^{\infty}\discount^t \rewardMDP[c](\varstate{c}{t},\varaction{c}{t})}{\varstate{c}{0}=s} ,
\end{align}
where for all $ t \geq 0$, $\varaction{c}{t} \sim \policy(\cdot | \varstate{c}{t}) $ is chosen using the shared policy, and $\varstate{c}{t+1} \sim \kerMDP[c](. | \varstate{c}{t}, \varaction{c}{t})$ follows the local dynamics of agent $c$'s environment. We define $\valuefunc[c][\policy](\initdist)$ as the value function when the initial distribution is $\initdist$. Similarly, the Q-function of a policy $ \policy$ for agent $c$ is
\begin{align}
\label{def:q_func}
\qfunc[c][\policy](s,a) \eqdef  \rewardMDP(s,a) + \discount \sum_{s' \in \S}\kerMDP[c](s'| s, a) \valuefunc[c][\policy](s')\eqsp.
\end{align}
This allows to define the advantage function $\advvalue[c][\policy](s,a) =  \qfunc[c][\policy](s,a) - \valuefunc[c][\policy](s) $.  Define $\reglocfunc[c](\theta) \eqdef \frllocfunc[c](\policy_{\theta})$. 
We define the advantage function of a policy $\policy_{\theta}$ as 
\begin{align}
\label{def:advantage_function}
\advvalue[c][\expandafter{\policy_{\theta}}](s,a) \eqdef \qfunc[c][\expandafter{ \policy_{\theta}}](s,a) -  \valuefunc[c][\expandafter{ \policy_{\theta}}](s)\eqsp, \quad \text{ for all } (s,a) \in \S \times \A \eqsp.
\end{align}
The occupancy measure of agent $c \in [\nagent]$, is defined as
\begin{align*}
\occupancy[c][\initdist, \policy](s) &\textstyle\eqdef (1- \discount) \sum_{t=0}^{\infty} \discount^t \initdist \kerMDP[c,\policy]^t(s) \eqsp, \text{where}
&\textstyle\quad \kerMDP[c,\policy](s'|s) \eqdef \sum_{a\in \A} \policy(a|s) \kerMDP[c](s'| s, a) 
\end{align*}
Following \citet{mei2020global}, we will use the following expression of the gradient.
\begin{lemma}[Lemma 10 from \cite{mei2020global}]
\label{lem:gradient_objective}
We have
\begin{align}
\frac{\partial \reglocfunc[c](\theta)}{\partial \theta(s,a)}  = \frac{1}{1- \gamma}\cdot \occupancy[c][\initdist, \policy_{\theta}](s) \policy_{\theta}(a|s) \advvalue[c][\expandafter{\policy_{\theta}}](s,a)\eqsp,
\end{align}
where $\advvalue[c][\expandafter{\policy_{\theta}}]$ is defined in \eqref{def:advantage_function}.
\end{lemma}

First, we establish the smoothness of $\softegrb{c}(\theta)$. 
\begin{lemma}
\label{lem:smoothness_estimator_soft}
For any $c\in [\nagent]$, the function $\softegrb{c}$ is $ \softsmoooth \eqdef 8/(1-\gamma)^3$-smooth, that is for all $\param, \param' \in \logitspace$, it holds that
\begin{align*}
\| \softegrb{c}(\theta') - \softegrb{c}(\theta)  \| \leq \softsmoooth \| \theta' - \theta \|_{2} \eqsp.
\end{align*}
\end{lemma}
\begin{proof}
The result follows from setting $\temp = 0$ in the bound of \Cref{lem:smoothness_estimator_regsoft}.

\end{proof}

\begin{lemma} %
\label{lem:smoothness_value_func}
For $c \in [\nagent]$, the function $\reglocfunc[c]$ is $ \softsmoooth = 8/(1-\gamma)^3$-smooth and $\regobjective$ is also $\softsmoooth$-smooth.
\end{lemma}
\begin{proof}
The result follows from Lemma~7 of \cite{mei2020global} and the fact that a mean of smooth functions is a smooth functions with the same smoothness coefficient.
\end{proof}

\begin{lemma}
\label{lem:soft_bound_grad}
For all $c \in [\nagent]$ and $\param \in \logitspace$, it holds
\begin{align*}
    \norm{ \nabla \reglocfunc[c](\param)}[2] 
    \le
    \softboundgrad \eqsp, \quad \text{ where } \softboundgrad \eqdef \frac{1}{(1- \discount)^2} \eqsp.
\end{align*}
\end{lemma}
\begin{proof}
By norm comparisons, and \Cref{lem:gradient_objective}, it holds that
\begin{align*}
\norm{ \reglocfunc[c](\param)}[2] 
\leq
\norm{ \reglocfunc[c](\param)}[1]
\leq 
\frac{1}{1- \discount} \sum_{s,a} \occupancy[c][\initdist, \policy_{\theta}](s) \policy_{\theta}(a|s) |\advvalue[c][\expandafter{\policy_{\theta}}](s,a)|
\eqsp.
\end{align*}
Finally, using that for any $(s,a) \in \S \times \A$, we have $|\advvalue[c][\expandafter{\policy_{\theta}}](s,a)| \leq 1/(1- \discount)$ concludes the proof.
\end{proof}

\begin{lemma}\label{lem:third_derivative_bound_nonref}
The spectral norm of the third derivative tensor is bounded by $\softthirddbound \eqdef 480 \cdot (1-\discount)^{-4}$, i.e., for any $u,v,w \in \logitspace$ it holds
\begin{align*}
    | \rmd^3 \reglocfunc[c](\theta)[u,v,w]| = |\nabla^3 \reglocfunc[c](\theta) u \otimes v \otimes w|    \leq \frac{480}{(1-\gamma)^4} \norm{u}[2] \norm{v}[2] \norm{w}[2]\,.    
\end{align*}
\end{lemma}
\begin{proof}
    By \Cref{lem:derivatives_value} with $\temp = 0$ we have for any $u,v,w \in \logitspace$
    \[
        \norm{\rmd^3 \valuefunc[c][\policy_\theta][u,v,w]}[\infty] \leq \frac{480}{(1-\discount)^4} \norm{u}[2] \norm{v}[2] \norm{w}[2]\eqsp.
    \]
    Next, we notice that
    \[
        \rmd^3 \reglocfunc[c](\theta)[u,v,w] = \rho^\top \rmd^3 \valuefunc[c][\policy_\theta][u,v,w]\eqsp,
    \]
    and the result follows from the fact that $\rho$ is a probability distribution.
\end{proof}

\begin{lemma}
\label{lem:bounded_gradient_frl}
Let $c\in [\nagent]$ and $\theta \in \logitspace$. It holds that
\begin{align*}
\norm{\nabla \regobjective(\theta) -  \nabla \reglocfunc[c](\theta)}[2]^{2} \leq \softhgty^2 \eqsp, \quad \text{ where } \softhgty^2\eqdef \frac{56\hgkernel^2}{(1- \discount)^6}  + \frac{36 \hgreward^2}{(1-\discount)^4} \eqsp.
\end{align*}
\end{lemma}
\begin{proof}
The result follows from setting $\temp = 0$ in the bound of \Cref{lem:bounded_gradient__regularised_frl}.

\end{proof}

The following lemma bounds the bias and the variance of the estimator of this stochastic gradient.
\begin{lemma}[Lemmas 6 and 7 from \cite{ding2025beyond}]
\label{lem:bias_and_variance_stochastic_gradient}
Consider the stochastic gradient defined in \eqref{eq:expression_of_stochastic_gradient_softmax_fedpg_appendix}. For any $\theta \in \logitspace$, we have
\begin{align*}
\| \nabla\reglocfunc[c](\theta) - \softegrb{c}(\theta)\|_2 \leq \softbias &\eqdef \frac{2 \discount^{\lentrunc}}{1- \discount}\left(\lentrunc + \frac{1}{1- \discount}\right) \eqsp,  \\
\Var{\softgrb{c}{Z_c}(\theta)} \leq \softmoments{2}{2} &\eqdef \frac{12}{\sizebatch(1- \discount)^4} \eqsp.
\end{align*}
\end{lemma}

Finally, we show that the fourth-order moment of our biased estimator is bounded.
\begin{lemma}
\label{lem:bound_fourth_moment_soft}
For any $c \in [\nagent]$, for any $\theta \in \logitspace$, the fourth central moment of $\softgrb{c}{\randState{c}}$ is bounded, that is
\begin{align}
    \PE_{\randState{c} \sim \softsampledist{c}{\param}} \left[\norm{ \softgrb{c}{\randState{c}}(\param) - \softegrb{c}(\param)}[2]^4 \right]
    \le
    \softmoments{4}{4} \eqdef \frac{1120}{\sizebatch^2(1- \discount)^8}
    \eqsp.
\end{align}
\end{lemma}
\begin{proof}
The result follows from setting $\temp = 0$ in the bound of 
\Cref{lem:bound_fourth_moment_regsoft}.
\end{proof}
The preceding lemmas conclude to prove that \Cref{assumFL:uniform_grad_bound} to \Cref{assumFL:bias} hold. The following lemma establishes that \Cref{assum:local_quadratic_pl_appendix} hold.
\begin{lemma}[Lemma 8 of \cite{mei2020global}]
\label{lem:pl_structure}
Assume \assumptionmdp. For all $c\in [\nagent]$, for any $\theta \in \logitspace$, it holds
\begin{align*}
\norm{\nabla \reglocfunc[c](\theta)}[2]^2 \geq 2\softmu[c](\theta) \cdot \left[ \optreglocfunc[c] - \reglocfunc[c](\theta)\right]^2 \eqsp,
\end{align*}
where
\begin{align*}
\softmu[c](\theta) \eqdef \frac{1}{2\nstates} \cdot \min_{s} \policy_{\theta}(a^{\star}(s)|s)^2 \cdot \left\| \frac{ \occupancy[c][\initdist, \policy_{c}^{\star}]}{ \occupancy[c][\initdist, \theta]} \right\|_{\infty}^{-2} \eqsp,
\end{align*}
and where $\policy_{c}^{\star}$ is an optimal deterministic policy of agent $c$, and $a^{\star}(s)$ is the action picked by this policy when the agent is in state $s$.
\end{lemma}

\subsection{Convergence rates, sample, and communication complexities}
Using \Cref{theorem:convergence_pl_type_square}, and \Cref{thm:complexity-general-fl_ql}, we derive the following convergence rates.
\begin{theorem}
[Convergence rates of $\SoftfedPG$] 
\label{theorem:convergence_softfedpg}
Assume \assumptionmdp\, and no projection ($\mathcal{T}\colon \param \rightarrow \param$). Additionally,  assume that there exists $1>\minminsoftmu>0 $ such that such that $ \inf_{r \in [\nset]} \minsoftmu (\globparam{r}) \geq \minminsoftmu $. For any $\step >0$ such that $\step \lsteps \softsmoooth \leq 1/74$ the iterates of $\SoftfedPG$ satisfy
\begin{align*}
\optfrlobjective - \PE[\frlobjective(\globparam{\trounds})] &\leq \frac{\optfrlobjective- \frlobjective(\globparam{0})}{1+  \trounds \cdot (\optfrlobjective- \frlobjective(\globparam{0})) \cdot \left(\step \lsteps \minminsoftmu /4 \right)}  + \left( \frac{6\step \softsmoooth  \softmoments{2}{2}}{\nagent\minminsoftmu} \right)^{1/2}  \\
& + \left( \frac{16 \cdot12^3 \step^4 \softthirddbound^2 \lsteps(\lsteps-1) \softmoments{4}{4}}{\minminsoftmu} \right)^{1/2}  +\left(\frac{2\softhgty^2}{\minminsoftmu} \right)^{1/2}  +\left( \frac{8\softbias^2}{\minminsoftmu} \right)^{1/2}  \\
& + \frac{\softhgty^2}{12\softsmoooth}
+ \frac{\step \softmoments{2}{2}}{4 \nagent}
+ \frac{\softbias^2 }{3 \softsmoooth} 
+ \frac{12^3 \step^4 \softthirddbound^2 \lsteps(\lsteps-1) \softmoments{4}{4}}{\softsmoooth}\eqsp.
\end{align*}
\end{theorem}
\begin{proof}
First note that the combination of lemmas of \Cref{subsec:satisfying_assumptions_ascent_softfedpg} and the assumption made on the trajectory of the iterates implies that  Assumptions~\Cref{assumFL:uniform_grad_bound} to \Cref{assumFL:bias}, \Cref{assum:local_quadratic_pl_appendix}, and \Cref{assum:improvement_quadratic_pl_appendix} hold. Importantly note that if $\step \lsteps \softsmoooth \leq 1/74$ then it holds that $32 \step^2 \lsteps^2 \softthirddbound^2 \softfourmoment^2 \le \softsmoooth^2$ (as $32\softthirddbound^2 \softfourmoment^2 \leq 74^2 \softsmoooth^4$ by \Cref{lem:smoothness_value_func}, \Cref{lem:soft_bound_grad}, and \Cref{lem:third_derivative_bound_nonref}). Thus, applying \Cref{theorem:convergence_pl_type_square} concludes the proof.
\end{proof}

Recall that
\begin{gather*}
    \softboundgrad
    = 
    \frac{1}{(1- \discount)^{2}}
    ~,~~
    \softsmoooth = \frac{8}{(1-\gamma)^{3}}
    ~,~~
    \softthirddbound =
    \frac{480}{(1-\discount)^{4}}
    ~,~~
    \softhgty^2 = 
    \frac{56\hgkernel^2}{(1- \discount)^6}  + \frac{36 \hgreward^2}{(1-\discount)^4}
    \eqsp,
    \\
    \softbias = 
    \frac{2\discount^{\lentrunc} \lentrunc}{1-\discount} + \frac{2\discount^{\lentrunc}}{(1-\discount)^{2}}
    ~,~~ 
    \softmoments{2}{2}
    = \frac{12}{(1- \discount)^{4} \sizebatch}
    ~,~~ 
    \softmoments{4}{4} 
    = 
    \frac{1120}{(1- \discount)^{8} \sizebatch^2}
    \eqsp,
\end{gather*}
which are defined respectively in \Cref{lem:smoothness_estimator_soft,lem:soft_bound_grad,lem:third_derivative_bound_nonref,lem:bounded_gradient_frl,lem:bias_and_variance_stochastic_gradient,lem:bound_fourth_moment_soft}. We obtain the following simplified result.
\begin{corollary}
[Simplified convergence rates of $\SoftfedPG$] 
\label{theorem:convergence_softfedpg_appendix_simplified}
Under the assumptions of \Cref{theorem:convergence_softfedpg}, for any $\step >0$ such that $\step \lsteps \leq (1-\discount)^3/592$, $T \geq 4(1-\discount)^2$, and $M \cdot B \geq (1-\discount)^{-1}$, the iterates of $\SoftfedPG$ satisfy 
\begin{align*}
\optregobjective- \PE[\regobjective(\globparam{\trounds})] &\leq \frac{\optregobjective - \regobjective(\globparam{0}) }{1+  \trounds \cdot (\optregobjective - \regobjective(\globparam{0})) \cdot \step \lsteps \minsoftmu/4}  + \frac{24 \step^{1/2}}{\minsoftmu^{1/2} \nagent^{1/2}\sizebatch^{1/2} \cdot (1-\discount)^{3.5}} 
\\ \quad
& + \frac{14^7\step^2 \lsteps^{1/2}(\lsteps-1)^{1/2} }{\minsoftmu^{1/2}(1- \discount)^{8} \sizebatch} + \frac{13\lentrunc \discount^\lentrunc}{\minsoftmu^{1/2} (1-\discount)} + \frac{13\hgkernel}{\minsoftmu^{1/2} (1-\discount)^3} + \frac{4\hgreward}{\minsoftmu^{1/2}(1-\discount)^2}
 \eqsp.
\end{align*}
\end{corollary}

\begin{corollary}[Sample and Communication Complexity] 
\label{thm:complexity-general-softfedpg-appendix}
Under the assumptions of \Cref{theorem:convergence_softfedpg}, for any $\lentrunc \geq 4(1-\discount)^{-2}$, and $\nagent \cdot \sizebatch \geq (1-\discount)^{-1}$
let 
\begin{align*}
\epsilon >  \frac{94\hgkernel}{(1-\discount)^3\minminsoftmu^{1/2}} + \frac{75\hgreward}{(1-\discount)^2\minminsoftmu^{1/2}}  + \frac{90 \discount^{\lentrunc}\lentrunc}{(1-\discount)\minminsoftmu^{1/2}} \eqsp,
\end{align*}
and
\begin{align*}
\step \le \min\Big( \frac{(1-\discount)^3}{48}, \frac{(1-\discount)^7\minminsoftmu \sizebatch\nagent \epsilon^2}{20736}, \frac{\minminsoftmu^{1/2} \epsilon (1-\discount)^5 \sizebatch}{2008 }\Big) \eqsp.
\end{align*}
In this case $\SoftfedPG$ achieves $\optregobjective - \regobjective(\globparam{0}) \leq \epsilon$, with a number of communication
\begin{align*}
\trounds \ge 
\frac{8640\left[\optregobjective - \regobjective(\globparam{0})  - \epsilon/6\right]}{(\optregobjective - \regobjective(\globparam{0})) \minminsoftmu \epsilon} \cdot\frac{1}{(1-\discount)^3}
\eqsp,
\end{align*}
for a total number of trajectories sampled per agent of
\begin{align*}
\trounds \lsteps \sizebatch
\ge 
\frac{144[\optregobjective - \regobjective(\globparam{0})  - \epsilon/6]}{(\optregobjective - \regobjective(\globparam{0})) \minminsoftmu \epsilon} \max\Big(  \frac{8\sizebatch}{(1-\discount)^3}, \frac{3464}{(1-\discount)^7\minminsoftmu \nagent \epsilon^2}, \frac{335}{\minminsoftmu^{1/2} \epsilon (1-\discount)^5}\Big)
\eqsp.
\end{align*}
\end{corollary}

%% file: AISTATS/appendix/RS_FedPG.tex
$\RegSoftfedPG$ is a special instance of $\projfedAVG$ in which, the local objective function is $f_c = \auxlocfunc[c] =: \reglocfunc[c] + \temp \regularisation{c}{\initdist}$, where for any $\theta \in \logitspace$ we have 
\begin{align}
\label{def:entropic_regularisation}
\regularisation{c}{\initdist}(\param) \eqdef -\PE_{\policy} \left[ \sum_{t=0}^{\infty} \discount^{t} \log(\policy_{\theta}(\varaction{c}{t}|\varstate{c}{t}))  \mid \varstate{c}{0} \sim \initdist \right] \eqsp.
\end{align}
We additionally define the global objective of the algorithm as $\auxobjective \eqdef \frac{1}{\nagent} \sum_{c=1}^{\nagent} \auxlocfunc[c]$. The client-specific data distribution $\sampledist{c}{\param}$ corresponds to $\softsampledist{c}{\theta}$, as defined in Eq.~\eqref{eq:truncated_trajectory_dist_def}. For a given parameter $\theta \in \logitspace$ and an observation $Z_c \sim \softsampledist{c}{\theta}$, we define the biased stochastic estimator of the gradient of the local objective $\auxlocfunc[c]$ as:
\begin{align}
\label{eq:expression_of_stochastic_regularised_gradient_softmax_fedpg_appendix}
\!\!\!\!\!\ \regsoftgrb{c}{Z_c}(\theta)
\eqdef \frac{1}{\sizebatch}  \sum_{b=1}^{\sizebatch} \sum_{t=0}^{\lentrunc-1} \discount^{t}  \left( \sum_{\ell =0}^t \nabla \log \policy_{\theta}(a_{c,b}^{\ell} \mid \varstate{c,b}{\ell}) \right)   \left[ \rewardMDP[c](\varstate{c,b}{t}, \varaction{c,b}{t}) - \temp \log(\policy_{\theta}(\varaction{c,b}{t}, \varstate{c,b}{t}))\right] \eqsp.
\end{align}
We also define
\begin{align*}
\regsoftegrb{c}(\theta) = \PE_{Z_c \sim \softsampledist{c}{\param}}[\regsoftgrb{c}{Z_c}(\theta)] \eqsp.
\end{align*}
To apply the convergence results of \Cref{secapp:ascent_lemma}, it remains to verify that Assumptions~\Cref{assumFL:uniform_grad_bound} to \Cref{assumFL:bias} , \Cref{assum:local_linear_pl_appendix}, and \Cref{assum:improvement_linear_pl_appendix} hold. We establish these conditions in the following.

\subsection{Checking the assumptions}
\label{subsec:applying_ascent_lemma_regsoft}

For convenience, we recall the definitions of the regularised value function, the regularised Q-function, and the regularised advantage function defined in \cite{geist2019theory}:
\begin{align}
\label{def:regularised_value_function}
\regvaluefunc[c][\policy_{\theta}](s) &\eqdef  \valuefunc[c][\policy_{\theta}](s) + \temp \regularisation{c}{s}(\theta) \eqsp, \eqsp \text{where } \regularisation{c}{s}(\param) = -\PE \left[ \sum_{t=0}^{\infty} \discount^{t} \log(\policy_{\theta}(\varaction{c}{t}|\varstate{c}{t}))  \mid \varstate{c}{0} = s \right]
\\
\label{def:regularised_q_function}
\regqfunc[c][\policy_{\theta}](s,a) &\eqdef \rewardMDP(s,a) + \discount \sum_{s' \in \S} \kerMDP[c](s'|s,a) \regvaluefunc[c][\policy_{\theta}](s') \eqsp,\\
\label{def:regularised_advantage_function}
\regadvvalue[c][\policy_{\theta}](s,a) &\eqdef \regqfunc[c][\policy_{\theta}](s,a) - \temp \log(\policy_{\theta}(a\mid s))-  \regvaluefunc[c][\policy_{\theta}](s)\eqsp, \quad \text{ for all } (s,a) \in \S \times \A \eqsp.
\end{align}
Following \citet{mei2020global}, we will use the following expression of the gradient.
\begin{lemma}[Lemma 10 from \cite{mei2020global}]
\label{lem:regularised_gradient_fc}
We have
\begin{align}
\frac{\partial \auxlocfunc[c](\theta)}{\partial \theta(s,a)}  = \frac{1}{1- \gamma}\cdot \occupancy[c][\initdist, \policy_{\theta}](s) \policy_{\theta}(a|s) \regadvvalue[c][\policy_{\theta}](s,a)\eqsp,
\end{align}
where $\regadvvalue[c][\policy_{\theta}]$ is defined in \eqref{def:regularised_advantage_function}.
\end{lemma}

First, we establish the smoothness of $\regsoftegrb{c}(\theta)$. 
\begin{lemma}
\label{lem:smoothness_estimator_regsoft}
For any $c\in [\nagent]$, the function $\regsoftegrb{c}$ is $ \regsoftsmooth \eqdef (8+ \temp(4+ 8\log(\nactions))/(1-\gamma)^3$-smooth, that is for all $\param, \param' \in \logitspace$, it holds that
\begin{align*}
\| \regsoftegrb{c}(\theta') - \regsoftegrb{c}(\theta)  \| \leq \regsoftsmooth  \| \theta' - \theta \|_{2} \eqsp.
\end{align*}
\end{lemma}
\begin{proof}
Fix any $\theta \in \logitspace$ and $c\in [\nagent]$. Let $\vartraj \eqdef (\varstate{}{0}, \varaction{}{0}, \dots, \varstate{}{\lentrunc-1}, \varaction{}{\lentrunc-1})$ be a random variable distributed according to $\softsampledist{c}{\param}$, as defined in \eqref{eq:truncated_trajectory_dist_def}. Then, $\regsoftegrb{c}(\theta)$ can be equivalently expressed as
\begin{align*}
\regsoftegrb{c}(\param) &=  \sum_{t=0}^{\lentrunc-1} \discount^t\sum_{\ell=0}^{t}  \underbrace{\PE_{\vartraj \sim \softsampledist{c}{\param}} \left[ \nabla \log \policy_{\theta}(\varaction{}{\ell} \mid \varstate{}{\ell}) \left(\rewardMDP[c](\varstate{}{t}, \varaction{}{t}) - \temp \log(\policy_{\theta}(\varaction{}{t} \mid \varstate{}{t} )) \right) \right]}_{\mathrm{E}_{\ell}^t(\theta)} \eqsp.
\end{align*}
Denote by $\mathrm{E}_{\ell}^t(\theta, s,a)$ the coefficient at coordinate $(s,a)$ of $\mathrm{E}_{\ell}^t(\theta)$. Using the REINFORCE formula (\Cref{lem:reinforce}), for any $(\bar{s},\bar{a})$, we can express the partial derivative of $\mathrm{E}_{\ell}^t(\theta,s ,a )$ with respect to $\theta(\bar{s},\bar{a})$ as
\begin{align*}
\frac{\partial \mathrm{E}_{\ell}^t(\theta,s,a)}{\partial \theta(\bar{s},\bar{a})} &= \frac{\partial}{\partial \theta(\bar{s},\bar{a})}\left[ \PE_{\vartraj \sim \softsampledist{c}{\param} }\left[ \frac{\partial \log \policy_{\theta}(\varaction{}{\ell} \mid \varstate{}{\ell})}{\partial \theta(s,a)} \left(\rewardMDP[c](\varstate{}{t}, \varaction{}{t}) - \temp \log(\policy_{\theta}(\varaction{}{t} \mid \varstate{}{t} )) \right)  \right] \right]  \\
& = \underbrace{\PE_{\vartraj \sim \softsampledist{c}{\param} }\left[  \frac{\partial \log( \softsampledist{c}{\param;\vartraj})}{\partial \theta(\bar{s},\bar{a})} \cdot \frac{\partial \log \policy_{\theta}(\varaction{}{\ell} \mid \varstate{}{\ell})}{\partial \theta(s,a)} \left(\rewardMDP[c](\varstate{}{t}, \varaction{}{t}) - \temp \log(\policy_{\theta}(\varaction{}{t} \mid \varstate{}{t} )) \right) \right]}_{\mathrm{F}_{\ell}^t(s,a, \bar{s}, \bar{a})} \\
&+ \underbrace{\PE_{\vartraj \sim \softsampledist{c}{\param} }\left[ \frac{ \partial^2 \log \policy_{\theta}(\varaction{}{\ell} \mid \varstate{}{\ell})}{\partial \theta(s,a) \partial \theta(\bar{s},\bar{a})} \left(\rewardMDP[c](\varstate{}{t}, \varaction{}{t}) - \temp \log(\policy_{\theta}(\varaction{}{t} \mid \varstate{}{t} )) \right) \right]}_{\mathrm{G}_{\ell}^t(s,a, \bar{s}, \bar{a})} \\
& - \temp \underbrace{\left[ \PE_{\vartraj \sim \softsampledist{c}{\param} }\left[ \frac{\partial \log \policy_{\theta}(\varaction{}{\ell} \mid \varstate{}{\ell})}{\partial \theta(s,a)} \cdot \frac{\partial \log(\policy_{\theta}(\varaction{}{t} \mid \varstate{}{t} ))}{\partial \theta(\bar{s},\bar{a})}   \right] \right]}_{\mathrm{H}_{\ell}^t(s,a, \bar{s}, \bar{a})} \eqsp.
\end{align*}
We now bound each of these three terms separately. Beforehand, recall that for any $(s,a, \bar{s}, \bar{a})$, we have
\begin{align}
\label{eq:partial_derivative_softmax}
\frac{\partial \policy_{\theta}(a \mid s)}{\partial \theta(\bar{s}, \bar{a})} = \Ind_{\bar{s}}(s)(\Ind_{\bar{a}}(a) \policy_{\theta}(a | s) - \policy_{\theta}(a|s) \policy_{\theta}(\bar{a}|s) ) \eqsp.
\end{align}

\paragraph{Bounding $\mathrm{F}_{\ell}^t(s,a, \bar{s}, \bar{a})$.} Using \eqref{eq:partial_derivative_softmax}, note that, for any $(s,a,s^{\ell}, a^{\ell})$, we have
\begin{align*}
\frac{ \partial \log(\policy_{\theta}(a^{\ell} \mid s^{\ell}))}{\partial \theta(s,a)}  = \Ind_{s}(s^{\ell}) \left( \Ind_{a}(a^{\ell}) - \policy_{\theta}(a|s) \right) \eqsp.
\end{align*}
Now, consider a trajectory $\traj = (s^0, a^0, \dots s^{\lentrunc-1}, a^{\lentrunc-1} )$. It holds that
\begin{align*}
\frac{\partial \log( \softsampledist{c}{\param;\traj})}{\partial \theta(\bar{s},\bar{a})} = \sum_{k=0}^{\lentrunc-1} \Ind_{\bar{s}}(s^k) \left( \Ind_{\bar{a}}(a^k) - \policy_{\theta}(\bar{a}|\bar{s}) \right) \eqsp.
\end{align*}
Additionally, note that for $k \geq \max(t, \ell)$, we have
\begin{align*}
\PE_{\vartraj }\left[  \Ind_{\bar{s}}(\varstate{}{k}) \left( \Ind_{\bar{a}}(\varaction{}{k}) - \policy_{\theta}(\bar{a}|\bar{s}) \right) \cdot \frac{ \partial \log(\policy_{\theta}(\varaction{}{\ell} \mid \varstate{}{\ell}))}{\partial \theta(s,a)} \left(\rewardMDP[c](\varstate{}{t}, \varaction{}{t}) - \temp \log(\policy_{\theta}(\varaction{}{t} \mid \varstate{}{t} )) \right)  \right] = 0 \eqsp.
\end{align*}
Combining the three previous identities, the triangle inequality and the fact that the reward is bounded by $1$ yields
\begin{align*}
\left|\mathrm{F}_{\ell}^t(s,a, \bar{s}, \bar{a}) \right| &\leq \sum_{k=0}^{t}   \PE\left[ \Ind_{\bar{s}}(\varstate{}{k}) \Ind_{s}(\varstate{}{\ell}) \left( \Ind_{\bar{a}}(\varaction{}{k}) \Ind_{a}(\varaction{}{\ell}) + \policy_{\theta}(\bar{a}|\bar{s}) \Ind_{a}(\varaction{}{\ell}) \right)  
\right] \\
&+ \sum_{k=0}^{t}   \PE\left[ \Ind_{\bar{s}}(\varstate{}{k}) \Ind_{s}(\varstate{}{\ell}) \left(  \Ind_{\bar{a}}(\varaction{}{k}) \policy_{\theta}(a \mid s)  + \policy_{\theta}(a \mid s) \policy_{\theta}(\bar{a}|\bar{s}) \right)  
\right] \\
&- \temp \sum_{k=0}^{t} 
 \PE\left[   \Ind_{\bar{s}}(\varstate{}{k}) \Ind_{s}(\varstate{}{\ell}) \left( \Ind_{\bar{a}}(\varaction{}{k}) \Ind_{a}(\varaction{}{\ell})  + \policy_{\theta}(\bar{a}|\bar{s})  \Ind_{a}(\varaction{}{\ell})\right)\log\policy_{\theta}(\varaction{}{t} \mid \varstate{}{t} )  
\right]  \\
 &- \temp \sum_{k=0}^{t} 
 \PE\left[   \Ind_{\bar{s}}(\varstate{}{k}) \Ind_{s}(\varstate{}{\ell}) \left( \Ind_{\bar{a}}(\varaction{}{k})\policy_{\theta}(a\mid s)  + \policy_{\theta}(\bar{a}|\bar{s})  \policy_{\theta}(a \mid s) \right)\log\policy_{\theta}(\varaction{}{t} \mid \varstate{}{t} )  
\right]  \eqsp.
\end{align*}

\paragraph{Bounding $\mathrm{G}_{\ell}^t(s,a, \bar{s}, \bar{a})$.} Consider a trajectory $\traj = (s^0, a^0, \dots s^{\lentrunc-1}, a^{\lentrunc-1} )$. It holds that
\begin{align*}
\frac{ \partial \log(\policy_{\theta}(a^{\ell} \mid s^{\ell}))}{\partial \theta(s,a)}  = \Ind_{s}(s^{\ell}) \left( \Ind_{a}(a^{\ell}) - \policy_{\theta}(a|s) \right) \eqsp.
\end{align*}
Next, deriving with respect to $\param(\bar{s}, \bar{a})$ yields
\begin{align*}
\frac{ \partial^2 \log \policy_{\theta}(a^{\ell} \mid s^{\ell})}{\partial \theta(s,a) \partial \theta(\bar{s},\bar{a})} &= - \Ind_{\bar{s}}(s^{\ell}) \Ind_{\bar{s}}(s) \left[ \Ind_{a}(\bar{a}) \policy_{\theta}(a|s) - \policy_{\theta}(a|s) \policy_{\theta}(\bar{a}|\bar{s}) \right] \eqsp.
\end{align*}
Combining the previous equality, the triangle inequality, and using that the reward is bounded by $1$ yields
\begin{align*}
\left| \mathrm{G}_{\ell}^t(s,a, \bar{s}, \bar{a}) \right|
&\leq \Ind_{\bar{s}}(s)  \Ind_{\bar{a}}(a) \PE[\Ind_{\bar{s}}(\varstate{}{\ell})] \policy_{\theta}(a \mid s) + \Ind_{\bar{s}}(s)  \PE[\Ind_{\bar{s}}(\varstate{}{\ell})]
\policy_{\theta}(a\mid s) \policy_{\theta}(\bar{a}\mid \bar{s}) \\
& - \temp \Ind_{\bar{s}}(s)  \Ind_{\bar{a}}(a) \PE[\Ind_{\bar{s}}(\varstate{}{\ell})\log\policy_{\theta}(\varaction{}{t} \mid \varstate{}{t})] \policy_{\theta}(a \mid s) \\
&- \temp \Ind_{\bar{s}}(s)  \PE[\Ind_{\bar{s}}(\varstate{}{\ell}) \log\policy_{\theta}(\varaction{}{t} \mid \varstate{}{t})]
\policy_{\theta}(a\mid s) \policy_{\theta}(\bar{a}\mid \bar{s})
\eqsp.
\end{align*}
\paragraph{Bounding $\mathrm{H}_{\ell}^t(s,a, \bar{s}, \bar{a})$.} 
Applying the triangle inequality yields
\begin{align*}
|\mathrm{H}_{\ell}^t(s,a, \bar{s}, \bar{a})|  &=  \temp \left| \PE_{\vartraj}\left[ \Ind_{s}(\varstate{}{\ell}) \left( \Ind_{a}(\varaction{}{\ell}) - \policy_{\theta}(a\mid s) \right) 
\Ind_{\bar{s}}(\varstate{}{t}) \left( \Ind_{\bar{a}}(\varaction{}{t}) - \policy_{\theta}(\bar{a} \mid \bar{s}) \right) 
\right] \right| \\
&\leq \temp \PE_{\vartraj}\left[ \Ind_{s}(\varstate{}{\ell}) \Ind_{\bar{s}}(\varstate{}{t}) \left( \Ind_{a}(\varaction{}{\ell}) \Ind_{\bar{a}}(\varaction{}{t}) +  \Ind_{a}(\varaction{}{\ell}) \policy_{\theta}(\bar{a} \mid \bar{s}) \right) 
\right]  \\
&+ \temp  \PE_{\vartraj}\left[ \Ind_{s}(\varstate{}{\ell}) \Ind_{\bar{s}}(\varstate{}{t}) \left( \policy_{\theta}(a\mid s)  \Ind_{\bar{a}}(\varaction{}{t}) + \policy_{\theta}(a\mid s)  \policy_{\theta}(\bar{a} \mid \bar{s}) \right) 
\right] \eqsp.
\end{align*}

Denote by $\regsoftegrb{c}(\param,s,a)$ the coefficient at coordinate $(s,a)$ of $\regsoftegrb{c}(\param)$. Applying the triangle inequality yields
\begin{align*}
\left| \frac{ \partial \softegrb{c}(\param,s,a)}{\partial \theta(\bar{s},\bar{a})} \right| &\leq \sum_{t=0}^{\lentrunc-1} \discount^t \sum_{\ell=0}^t \left[ \left|\mathrm{F}_{\ell}^t(s,a, \bar{s}, \bar{a}) \right| + \left| \mathrm{G}_{\ell}^t(s,a, \bar{s}, \bar{a}) \right| + \left| \mathrm{H}_{\ell}^t(s,a, \bar{s}, \bar{a}) \right| \right]
\end{align*}
Using that for any $s' \in \S$,  $\sum_{s\in \S} \Ind_{s}(s') =1 $, and that for any $a' \in \A$,  $\sum_{a\in \A} \Ind_{s}(a') =1 $ gives
\begin{align*}
\sum_{s,a,\bar{s}, \bar{a}} \left| \frac{ \partial \softegrb{c}(\param,s,a)}{\partial \theta(\bar{s},\bar{a})} \right| &\leq \sum_{t=0}^{\lentrunc-1} \discount^t \sum_{\ell=0}^t \left[ 2 t +2 +4\temp - (2 \temp + 2t \temp) \PE_{\vartraj} \left[ \log\policy_{\theta}(\varaction{}{t} \mid \varstate{}{t} ) \right]   \right]
\end{align*}
Now using that for any $x \in [0,1[$, $\sum_{k=0}^{\infty}k^2x^k \leq 2/(1-x)^3$, $\sum_{k=0}^{\infty}k x^k \leq 1/(1-x)^2$, and that $-\PE_{\vartraj} \left[ \log\policy_{\theta}(\varaction{}{t} \mid \varstate{}{t} ) \right] \leq \log(\nactions)$ yields
\begin{align*}
\big\|\frac{\partial \softegrb{c}}{\partial \theta} \big\|_{2} \leq \big\|\frac{\partial \softegrb{c}}{\partial \theta} \big\|_{1} \leq \frac{8 + \temp(4+ 8\log(\nactions)))}{(1- \discount)^3} \eqsp,
\end{align*}
which concludes the proof.
\end{proof}

\begin{lemma}
\label{lem:smoothness_regularised_value_func}
For any $c\in [\nagent] $,  $\auxlocfunc[c]$ and $\auxobjective$ are $\regsoftsmooth \eqdef (8 + \temp (4+ 8\log(\nactions))/(1-\gamma)^3$-smooth.
\end{lemma}
\begin{proof}
Follows from \cite[Lemma 14]{mei2020global} and the properties of averaging of smooth functions.
\end{proof}

\begin{lemma}
\label{lem:regsoft_bound_grad}
For all $c \in [\nagent]$ and $\param \in \logitspace$, it holds
\begin{align*}
    \norm{ \nabla \auxlocfunc[c](\param)}[2] 
    \le
    \regsoftboundgrad \eqsp, \quad \text{ where } \regsoftboundgrad \eqdef \frac{1+ \temp \log(\nactions)}{(1- \discount)^2} \eqsp.
\end{align*}
\end{lemma}
\begin{proof}
By norm comparisons, and \Cref{lem:regularised_gradient_fc}, it holds that
\begin{align*}
\norm{ \auxlocfunc[c](\param)}[2] \leq         \norm{ \auxlocfunc[c](\param)}[1] =  \frac{1}{1- \discount} \sum_{s,a} \occupancy[c][\initdist, \policy_{\theta}](s) \policy_{\theta}(a|s) |\regadvvalue[c][\expandafter{\policy_{\theta}}](s,a)|\eqsp.
\end{align*}
Now, using that for any $(s,a) \in \S \times \A$, we have $|\regadvvalue[c][\expandafter{\policy_{\theta}}](s,a)| \leq (1+ \temp \log(\nactions))/(1- \discount)$ yields
\begin{align*}
\norm{ \auxlocfunc[c](\param)}[2] \leq       \norm{ \auxlocfunc[c](\param)}[1] \leq \frac{1+ \temp \log(\nactions)}{(1- \discount)^2} \sum_{s,a}\occupancy[c][\initdist, \policy_{\theta}](s) \policy_{\theta}(a|s)  = \frac{1+ \temp \log(\nactions)}{(1- \discount)^2} \eqsp.
\end{align*}
which concludes the proof.
\end{proof}

\begin{lemma}\label{lem:third_derivative_bound_reg}
The spectral norm of the third derivative tensor of $\auxlocfunc[c]$ is bounded by $\regthirddbound \eqdef (480 + 832 \temp \log |\A|) \cdot (1-\discount)^{-4}$, i.e., for any $u,v,w \in \logitspace$ it holds
\begin{align*}
    | \rmd^3 \auxlocfunc[c](\theta)[u,v,w]| = |\nabla^3 \auxlocfunc[c](\theta) u \otimes v \otimes w|    \leq \frac{480 + 832 \temp \log |\A|}{(1-\gamma)^4} \norm{u}[2] \norm{v}[2] \norm{w}[2]\,.    
\end{align*}
\end{lemma}
\begin{proof}
    By \Cref{lem:derivatives_value} we have for any $u,v,w \in \logitspace$
    \[
        \norm{\rmd^3 \regvaluefunc[c][\policy_\theta][u,v,w]}[\infty] \leq \frac{480 + 832 \temp \log |\A| }{(1-\discount)^4} \norm{u}[2] \norm{v}[2] \norm{w}[2]\,.
    \]
    Next, we notice that
    \[
        \rmd^3 \auxlocfunc[c](\theta)[u,v,w] = \rho^\top \rmd^3 \valuefunc[c][\policy_\theta][u,v,w]\,,
    \]
    thus 
    \[
        | \rmd^3 \auxlocfunc[c](\theta)[u,v,w]| \leq \frac{480  + 832 \temp \log |\A|}{(1-\gamma)^4} \norm{u}[2] \norm{v}[2] \norm{w}[2]\,.    
    \]
\end{proof}

\begin{lemma}
\label{lem:bounded_gradient__regularised_frl}
Let $c\in [\nagent]$ and $\theta \in \logitspace$. It holds that
\begin{align*}
\norm{\nabla \auxobjective(\theta) -  \nabla \auxlocfunc[c](\theta)}[2]^{2} \leq \regsofthgty^2 \eqsp, \quad \text{ where } \regsofthgty^2 \eqdef \frac{56(1+\temp \log(\nactions))^2\hgkernel^2}{(1- \discount)^6}  + \frac{36 \hgreward^2}{(1-\discount)^4}\eqsp.
\end{align*}
\end{lemma}
\begin{proof}
Fix $c\in [\nagent]$ and $\theta \in \logitspace$. Using \cref{lem:regularised_gradient_fc}, we have
\begin{align*}
&\left|\frac{\partial \auxobjective(\theta)}{\partial \theta(s,a)} - \frac{ \partial \auxlocfunc[c](\theta)}{\partial \theta(s,a)} \right| \leq \frac{1}{\nagent}\frac{1}{1- \gamma} \sum_{k=1}^{\nagent} \policy_{\theta}(a|s)  \left|\occupancy[k][\initdist, \policy_{\theta}](s)\regadvvalue[k][ \policy_{\theta}](s,a) -  \occupancy[c][\initdist, \policy_{\theta}](s) \regadvvalue[c][\policy_{\theta}](s,a)\right| \\
& \leq \frac{1}{\nagent} \frac{\policy_{\theta}(a|s)}{1- \gamma} \sum_{k=1}^{\nagent} \left|\occupancy[k][\initdist, \policy_{\theta}](s) -  \occupancy[c][\initdist, \policy_{\theta}](s)\right| \underbrace{\left|\regadvvalue[k][\policy_{\theta}](s,a)  \right|}_{\term{X}} +  \underbrace{\left|\regadvvalue[k][ \policy_{\theta}](s,a) -  \regadvvalue[c][ \policy_{\theta}](s,a)\right|}_{\term{Y}} \occupancy[c][\initdist, \policy_{\theta}](s)
\eqsp.
\end{align*}
We bound each of $\term{X}$ and $\term{Y}$ separately.

\paragraph{Bounding $\term{X}$.} First note that we have 
\begin{align}
\label{eq:bound_on_regularised_value_function}
0 \leq \regvaluefunc[c][\policy_{\theta}](s) \leq \frac{1 + \temp \log(\nactions)}{1 - \discount} \eqsp,
\end{align}
where $\regvaluefunc[c][\policy_{\theta}](s)$ is defined in \eqref{def:regularised_value_function}. Combining the previous inequality and applying the triangle inequality yields
\begin{align*}
\term{X}  &=  \left|\rewardMDP[k](s,a) + \discount \sum_{s' \in \S} \kerMDP[k](s'|s,a) \regvaluefunc[k,\policy_{\theta}](s')- \temp \log(\policy_{\theta}(a|s)) - \regvaluefunc[k,\policy_{\theta}](s)  \right| \\
&\leq  \frac{2 + 2\temp \log(\nactions)}{1 - \discount}+ \temp|\log(\policy_{\theta}(a|s))| \eqsp.
\end{align*}

\paragraph{Bounding $\term{Y}$.} Using the triangle inequality, we get
\begin{align*}
\term{Y} & = \left| \rewardMDP[k](s,a) - \rewardMDP[c](s,a) + \discount \sum_{s' \in \S} \kerMDP[k](s'|s,a) \regvaluefunc[k][\policy_{\theta}](s') -  \discount \sum_{s' \in \S} \kerMDP[c](s'|s,a) \regvaluefunc[c][\policy_{\theta}](s') +   \regvaluefunc[c][\policy_{\theta}](s) - \regvaluefunc[k][\policy_{\theta}](s)\right|
\\
&\leq  \discount \sum_{s' \in \S} \kerMDP[k](s'|s,a) \left| \regvaluefunc[k][\policy_{\theta}](s') - \regvaluefunc[c][\policy_{\theta}](s')  \right| +  \discount \sum_{s' \in \S} \left| \kerMDP[k](s'|s,a) -  \kerMDP[c](s'|s,a)  \right|\regvaluefunc[c][\policy_{\theta}](s') + \hgreward \\
&+ \left| \regvaluefunc[k][\policy_{\theta}](s) - \regvaluefunc[c][\policy_{\theta}](s)  \right| \eqsp.
\end{align*}
where in the last inequality, we used that $\hgreward = \max_{(c,c')\in [\nagent]^2} \norm{\rewardMDP[c] -\rewardMDP[c']}$ (defined in \Cref{sec:bg}). Using $\hgkernel = \max_{(s,a,(c,c')) \in \S \times \A \times[\nagent]^2}\norm{\kerMDP[c](\cdot | s,a) - \kerMDP[c'](\cdot | s,a)}[1] $, combined with \eqref{eq:bound_on_regularised_value_function}, we obtain
\begin{align*}
\term{Y} \leq  \discount \sum_{s' \in \S} \kerMDP[k](s'|s,a) \left| \regvaluefunc[k][\policy_{\theta}](s') - \regvaluefunc[c][\policy_{\theta}](s')  \right| + \frac{\left(1+ \temp \log(\nactions)\right)\hgkernel}{1- \discount} + \left| \regvaluefunc[k][\policy_{\theta}](s) - \regvaluefunc[c][\policy_{\theta}](s)  \right| + \hgreward \eqsp.
\end{align*}
Using \eqref{def:regularised_value_function}, note that we have
\begin{align*}
\left| \regvaluefunc[k][\policy_{\theta}](s) - \regvaluefunc[c][\policy_{\theta}](s)  \right| \leq \left| \valuefunc[k][\policy_{\theta}](s) - \valuefunc[c][\policy_{\theta}](s)  \right| + \temp \left|  \regularisation{k}{s} (\theta) -  \regularisation{c}{s}(\theta)  \right| \eqsp.
\end{align*}
The bound on the first term of the previous bound is provided by \Cref{lem:bound_difference_value_function}. For the second term, we have
\begin{align*}
\temp \left|  \regularisation{k}{s}(\theta) -  \regularisation{c}{s}(\theta)  \right| 
&\leq   \frac{\temp}{1- \discount} \sum_{s_0 \in \S} \sum_{a \in \A} \left|   \occupancy[k][s, \theta](s_0)  -  \occupancy[c][s, \theta](s_0)   \right| \left| \policy_{\theta}(a|s_0) \log(  \policy_{\theta}(a|s_0)) 
 \right| \\
 &\leq   \frac{\temp \log(\nactions)}{1- \discount} \sum_{s_0 \in \S} \left|   \occupancy[k][s, \theta](s_0)  -  \occupancy[c][s, \theta](s_0)   \right| \eqsp,
\end{align*}
where in the last inequality we used $- \sum_{a\in \A} \policy_{\theta}(a|s) \log(  \policy_{\theta}(a|s)) \leq \log(\nactions)$. Finally plugging in the bound of \Cref{lem:improved_bound_diff_stationary_distribution} yields
\begin{align*}
\temp \left|  \regularisation{k}{s}(\theta) -  \regularisation{c}{s}(\theta) \right| \leq \frac{\temp \log(\nactions)\hgkernel}{(1- \discount)^2} \eqsp.
\end{align*}
Thus, we get the following bound on $\term{Y}$
\begin{align*}
\term{Y} \leq 3 \cdot \frac{(1+ \temp \log(\nactions)) \hgkernel}{(1- \discount)^2} +3 \frac{\hgreward}{1-\discount}\eqsp.
\end{align*}
Combining the bounds on $\term{X}$ and $\term{Y}$ yields
\begin{align*}
&\left|\frac{\partial \auxobjective(\theta)}{\partial \theta(s,a)} - \frac{ \partial \auxlocfunc[c](\theta)}{\partial \theta(s,a)} \right| \\
& \quad\leq \frac{1}{\nagent} \sum_{k=1}^{\nagent} \left[  \frac{2(1 + \temp \log(\nactions))}{1 - \discount}+ \temp|\log(\policy(a|s))| \right] \left|\occupancy[k][\initdist, \theta](s) -  \occupancy[c][\initdist, \theta](s)\right| \frac{\policy_{\theta}(a|s)}{1- \gamma} 
\\
& \quad + \left[ \frac{3(1+\temp \log(\nactions))\hgkernel}{(1- \discount)^2} + \frac{3\hgreward}{1-\discount}\right] \cdot \frac{\occupancy[c][\initdist, \theta](s)\policy_{\theta}(a|s)}{1- \gamma} 
\end{align*}
Thus, we get by Young's inequality 
\begin{align*}
&\norm{\nabla \auxobjective(\theta) -  \nabla \auxlocfunc[c](\theta)}[2]^{2}  \\
& \quad= \sum_{(s,a) \in \S \times \A} \left|\frac{\partial \auxobjective(\theta)}{\partial \theta(s,a)} - \frac{ \partial \auxlocfunc[c](\theta)}{\partial \theta(s,a)} \right|^2 \\
& \quad\leq \sum_{(s,a) \in \S \times \A} 2 \cdot \left(\frac{1}{\nagent} \sum_{k=1}^{\nagent} \left[ \frac{2(1 + \temp \log(\nactions))}{1 - \discount}+ \temp|\log(\policy(a|s))| \right] \left|\occupancy[k][\initdist, \theta](s) -  \occupancy[c][\initdist, \theta](s)\right| \frac{\policy_{\theta}(a|s)}{1- \gamma} \right)^2\\
& \quad + \sum_{(s,a) \in \S \times \A} 2 \cdot \left( \left[ \frac{3(1+\temp \log(\nactions))\hgkernel}{(1- \discount)^2}+   \frac{3\hgreward}{1-\discount}\right] \cdot \frac{\occupancy[c][\initdist, \theta](s)\policy_{\theta}(a|s)}{1- \gamma} \right)^2 \eqsp.
\end{align*}
Now applying Jensen's inequality yields
\begin{align*}
&\norm{\nabla \auxobjective(\theta) -  \nabla \auxlocfunc[c](\theta)}[2]^{2}  \\
& \quad\leq \frac{2}{\nagent} \sum_{(s,a) \in \S \times \A} \sum_{c=1}^{\nagent}  \left[ \frac{2(1 + \temp \log(\nactions))}{1 - \discount}+ \temp|\log(\policy(a|s))| \right]^2 \left|\occupancy[k][\initdist, \theta](s) -  \occupancy[c][\initdist, \theta](s)\right|^2 \frac{\policy_{\theta}(a|s)^2}{(1- \gamma)^2} \\
& \quad + \frac{2}{\nagent} \sum_{(s,a) \in \S \times \A} \sum_{k=1}^{\nagent} \left[\frac{18(1+\temp \log(\nactions))^2\hgkernel^2}{(1- \discount)^6} + \frac{18 \hgreward^2}{(1-\discount)^4} \right] \cdot \occupancy[c][\initdist, \theta](s)^2\policy_{\theta}(a|s)^2 \\
& \quad\leq \frac{1}{\nagent} \sum_{(s,a) \in \S \times \A} \sum_{c=1}^{\nagent}   \frac{16(1 + \temp \log(\nactions))^2}{(1 - \discount)^2} \left|\occupancy[k][\initdist, \theta](s) -  \occupancy[c][\initdist, \theta](s)\right|^2 \frac{\policy_{\theta}(a|s)^2}{(1- \gamma)^2} \\
& \quad + \frac{1}{\nagent} \sum_{(s,a) \in \S \times \A} \sum_{c=1}^{\nagent}    4 \temp^2|\log(\policy(a|s))|^2  \left|\occupancy[k][\initdist, \theta](s) -  \occupancy[c][\initdist, \theta](s)\right|^2 \frac{\policy_{\theta}(a|s)^2}{(1- \gamma)^2} \\
& \quad + \frac{1}{\nagent} \sum_{(s,a) \in \S \times \A} \sum_{k=1}^{\nagent} \left[\frac{36(1+\temp \log(\nactions))^2\hgkernel^2}{(1- \discount)^6} + \frac{36 \hgreward^2}{(1-\discount)^4}\right]\cdot \occupancy[c][\initdist, \theta](s)^2\policy_{\theta}(a|s)^2 \eqsp,
\end{align*}
For the first term, using that $\policy_{\theta}(a|s) \leq 1$, $|\occupancy[k][\initdist, \theta](s) -  \occupancy[c][\initdist, \theta](s)| \leq \hgkernel/(1- \discount)$, for the second term using that $|\log(\policy_{\theta}(a|s))|\policy_{\theta}(a|s) \leq 1$, $|\occupancy[k][\initdist, \theta](s) -  \occupancy[c][\initdist, \theta](s)| \leq \hgkernel/(1- \discount)$, and for the third term applying that $\policy_{\theta}(a|s) \occupancy[c][\initdist, \theta](s) \leq 1$ gives
\begin{align*}
\norm{\nabla \auxobjective(\theta) -  \nabla \auxlocfunc[c](\theta)}[2]^{2} & \quad\leq \frac{1}{\nagent} \sum_{(s,a) \in \S \times \A} \sum_{c=1}^{\nagent}   \frac{16(1 + \temp \log(\nactions))^2 \hgkernel}{(1 - \discount)^3} \left|\occupancy[k][\initdist, \theta](s) -  \occupancy[c][\initdist, \theta](s)\right| \frac{\policy_{\theta}(a|s)}{(1- \gamma)^2} \\
& \quad + \frac{1}{\nagent} \sum_{(s,a) \in \S \times \A} \sum_{c=1}^{\nagent}    4 \temp^2|\log(\policy(a|s))|  \left|\occupancy[k][\initdist, \theta](s) -  \occupancy[c][\initdist, \theta](s)\right| \frac{\policy_{\theta}(a|s)\hgkernel}{(1- \gamma)^3} \\
& \quad + \frac{1}{\nagent} \sum_{(s,a) \in \S \times \A} \sum_{k=1}^{\nagent}\left[\frac{36(1+\temp \log(\nactions))^2\hgkernel^2}{(1- \discount)^6} + \frac{36 \hgreward^2}{(1-\discount)^4}\right]\cdot \occupancy[c][\initdist, \theta](s)\policy_{\theta}(a|s) \eqsp,
\end{align*}
Finally, for the first term using that $\sum_{a\in \A} |\log(\policy_{\theta}(a|s))|\policy_{\theta}(a|s) \leq \log(\nactions)$, and using \Cref{lem:improved_bound_diff_stationary_distribution} for both the first and second term yields
\begin{align*}
 \norm{\nabla \auxobjective(\theta) -  \nabla \auxlocfunc[c](\theta)}[2]^{2} 
\leq \frac{56(1+\temp \log(\nactions))^2\hgkernel^2}{(1- \discount)^6}  + \frac{36 \hgreward^2}{(1-\discount)^4} \eqsp,
\end{align*}
which concludes the proof.
\end{proof} 
The following lemma bounds the bias and the variance of this stochastic gradient.
\begin{lemma}[Lemma 6 from \cite{ding2025beyond}]
\label{lem:bias_and_variance_regularised_stochastic_gradient}
Consider the stochastic gradient defined in \eqref{eq:expression_of_stochastic_regularised_gradient_softmax_fedpg_appendix}. We have
\begin{align*}
\| \regsoftegrb{c}(\theta) - \nabla \auxlocfunc[c](\theta)\|_2 &\leq \regsoftbias \eqdef \frac{2 (1+ \temp \log(\nactions))\discount^{\lentrunc}}{1- \discount}\left(\lentrunc + \frac{1}{1- \discount}\right) \eqsp, \\
\Var{\regsoftgrb{c}{Z_c}} &\leq \regsoftmoments{2}{2} \eqdef \frac{12 + 24 \temp^2(\log(\nactions))^2}{\sizebatch(1- \discount)^4} \eqsp.
\end{align*}
\end{lemma}

Finally, we show that the fourth-order moment of our biased estimator is bounded.
\begin{lemma}
\label{lem:bound_fourth_moment_regsoft}
For any $c \in [\nagent]$, for any $\theta \in \logitspace$, the fourth central moment of $\regsoftgrb{c}{\randState{c}}$ is bounded, that is
\begin{align}
    \PE_{\randState{c} \sim \softsampledist{c}{\param}} \left[\norm{ \regsoftgrb{c}{\randState{c}}(\param) - \regsoftegrb{c}(\param)}[2]^4 \right]
    \le
    \regsoftmoments{4}{4} \eqdef \frac{1120 + 4480 \temp^4\log(\nactions)^4}{\sizebatch^2(1- \discount)^8} 
    \eqsp.
\end{align}
\end{lemma}
\begin{proof}
Fix $c\in [\nagent]$, $\theta \in \logitspace$ and an observation $\randState{c} = (\randState{c,1}, \dots \randState{c,\sizebatch}) \sim \softsampledist{c}{\param}^{\otimes \sizebatch}$. For more readability of the proof, we define for any $z = (s^t, a^t)_{t=0}^{\lentrunc-1} \in (\S\times\A)^{\lentrunc}$:
\begin{align*}
\dzerotermmoment(z) \eqdef \sum_{t=0}^{\lentrunc-1} \discount^{t}  \left( \sum_{\ell =0}^t \nabla \log \policy_{\theta}(a^t \mid s^t) \right)   \left[ \rewardMDP(s^t, a^t) - \temp \log(\policy_{\theta}(a^t\mid s^t)\right] \eqsp.
\end{align*}
Importantly, note that
\begin{align*}
 \regsoftgrb{c}{\randState{c}}(\theta) = \frac{1}{\sizebatch}\sum_{b=1}^{\sizebatch}  \bzerotermmoment \eqsp, \quad \text{ and } \regsoftegrb{c}(\theta) =  \azerotermmoment \eqsp,
\end{align*}
where we define $\azerotermmoment = \PE_{\randState{c,b} \sim \softsampledist{c}{\param}}[ \bzerotermmoment]$.
Using this decomposition, we can bound the fourth order of $\regsoftgrb{c}{\randState{c}}(\theta)$ by the fourth central moment of $\zerotermmoment{b} $. Indeed, expanding the norm to the fourth power yields
\begin{align*}
&\PE_{\randState{c} \sim \softsampledist{c}{\param}^{\otimes \sizebatch}} \left[\norm{ \regsoftgrb{c}{\randState{c}}(\param) - \regsoftegrb{c}(\param)}[2]^4 \right]  = \PE_{\randState{c}} \left[\left\| \frac{1}{\sizebatch} \sum_{b=1}^{\sizebatch} \left[\zerotermmoment{b} - \ezerotermmoment  \right] \right\|_{2}^4 \right] \\
& = \frac{1}{\sizebatch^4}  \sum_{b_1=1}^{\sizebatch} 
\sum_{b_2=1}^{\sizebatch} 
\sum_{b_3=1}^{\sizebatch} 
\sum_{b_4=1}^{\sizebatch}  \underbrace{\PE_{\randState{c}} \left[ 
\pscal{\zerotermmoment{b_1} -\ezerotermmoment}{\zerotermmoment{b_2} - \ezerotermmoment}\pscal{\zerotermmoment{b_3} - \ezerotermmoment}{\zerotermmoment{b_4}- \ezerotermmoment}  \right]}_{U(b_1,b_2, b_3, b_4)} \eqsp.
\end{align*}
Note that by independence of the trajectories, $U(b_1,b_2, b_3, b_4)$ is non-zero if and only if all of the indices are equal or there are two pairs of equal indices. In this case, as the trajectories are identically distributed, $U(b_1,b_2, b_3, b_4)$ is respectively equal to $\PE \left[\norm{ \zerotermmoment{1} -\ezerotermmoment}[2]^4 \right]$ and $ \PE\left[ \|\zerotermmoment{1} -\ezerotermmoment \|^2\right]^2$.
There are exactly $\sizebatch$ combinations where all indices are equal, and 
\[
\frac{\sizebatch(\sizebatch - 1)}{2} \cdot \frac{4 \cdot 3}{2}
\]
combinations corresponding to the two distinct pairs of equal indices case. Combining these, we arrive at the following identity:
\begin{align}
\label{eq:main_bound_fourth_central_moment}
\!\!\!\!\! \PE_{\randState{c}} \left[\norm{ \regsoftgrb{c}{\randState{c}}(\param) - \regsoftegrb{c}(\param)}[2]^4 \right]  &=  \frac{1}{\sizebatch^3} \left[ \underbrace{\PE \left[\norm{ \zerotermmoment{1} -\ezerotermmoment}[2]^4 \right]}_{\term{M}} + 3 (\sizebatch-1) \PE\left[ \|\zerotermmoment{1} -\ezerotermmoment \|^2\right]^2\right] \eqsp.
\end{align}
We now decompose $\zerotermmoment{1}$ into two components, one that comes from the rewards of the MDP and a second associated with the regularization. Precisely, we define
\begin{align*}
\firsttermmoment &\eqdef \sum_{t=0}^{\lentrunc-1} \discount^{t}  \left( \sum_{\ell =0}^t \nabla \log \policy_{\theta}(\varaction{c,1}{\ell} \mid \varstate{c,1}{\ell}) \right) \rewardMDP(\varstate{c,1}{t}, \varaction{c,1}{t}) \eqsp, \\
\secondtermmoment &\eqdef - \temp \sum_{t=0}^{\lentrunc-1} \discount^{t}  \left( \sum_{\ell =0}^t \nabla \log \policy_{\theta}(\varaction{c,1}{\ell} \mid \varstate{c,1}{\ell}) \right)  \log(\policy_{\theta}(\varaction{c,1}{t}, \varstate{c,1}{t})) \eqsp.
\end{align*}
Additionally, define $\efirsttermmoment$ and $\esecondtermmoment$ respectively as the expectations of $\firsttermmoment$ and $\secondtermmoment$. Importantly, note that 
\begin{align*}
 \zerotermmoment{1}  =  \firsttermmoment + \secondtermmoment  \eqsp, \quad \text{ and } \ezerotermmoment =  \efirsttermmoment + \esecondtermmoment  \eqsp.
\end{align*}
Thus, using the triangle inequality, we have
\begin{align*}
\term{M} & \leq \PE_{\randState{c} \sim \softsampledist{c}{\param}} \left[\left(\norm{ \firsttermmoment -\efirsttermmoment}[2] + \norm{  \secondtermmoment  - \esecondtermmoment }[2] \right)^4 \right] \\
&\leq  \PE_{\randState{c} \sim \softsampledist{c}{\param}} \left[\left(2 \max(\norm{ \firsttermmoment -\efirsttermmoment}[2],\norm{  \secondtermmoment  - \esecondtermmoment }[2])\right)^4 \right] \\
&\leq  16 \underbrace{\PE_{\randState{c}} \left[ \norm{ \firsttermmoment -\efirsttermmoment}[2]^4 \right]}_{\term{M_1}} + 16 \underbrace{\PE_{\randState{c}} \left[ \norm{  \secondtermmoment  - \esecondtermmoment }[2]^4 \right]}_{\term{M_2}}
\end{align*}
Subsequently, we bound each of these two terms separately.

\paragraph{Bounding $\term{M_1}$.} Applying the triangle inequality, combined with Jensen's inequality, and using the fact that for any $(s,a) \in \S \times \A$, we have $\norm{\nabla \log (\policy_{\theta}(a \mid s))}[2] \leq 2$ (see, e.g., proof of Lemma 7 in \cite{ding2025beyond}), gives
\begin{align*}
\term{M_1} =  \PE_{\randState{c}} \left[ \norm{ \firsttermmoment -\efirsttermmoment}[2]^4 \right] 
\leq   \PE_{\randState{c}} \left[ \norm{ \firsttermmoment}[2]^4 \right]  
\leq  \left( \sum_{t=0}^{\lentrunc-1} 2 t\discount^{t} \right)^4  \leq \frac{16}{(1- \discount)^8} \eqsp,
\end{align*}
where in the last inequality, we used that for any $x \in [0,1[$, $\sum_{k=0}^{\infty}k x^k \leq 1/(1-x)^2$.

\paragraph{Bounding $\term{M_2}$.}
Applying the triangle inequality combined with Jensen's inequality yields
\begin{align*}
\term{M_2} 
&\leq  \PE_{\randState{c}} \left[ \norm{ \secondtermmoment}[2]^4 \right] 
\\
&=  \temp^4 \PE_{\randState{c}} \left[ \norm{ \sum_{t=0}^{\lentrunc-1} \discount^{t}  \left( \sum_{\ell =0}^t \nabla \log \policy_{\theta}(\varaction{c,1}{\ell} \mid \varstate{c,1}{\ell}) \right)  \log(\policy_{\theta}(\varaction{c,1}{t}, \varstate{c,1}{t}))}[2]^4 \right]
\\ 
& \leq  \temp^4 \PE_{\randState{c}} \left[ \left(\sum_{t=0}^{\lentrunc-1} \discount^{t} \sum_{\ell=0}^t \norm{ \nabla \log \policy_{\theta}(\varaction{c,1}{\ell} \mid \varstate{c,1}{\ell})}[2]|\log(\policy_{\theta}(\varaction{c,1}{t}, \varstate{c,1}{t}))| \right)^4\right] \eqsp,
\end{align*}
where in the last inequality, we used the triangle inequality. Now, using that $\norm{\nabla \log (\policy_{\theta}(a \mid s))}[2] \leq 2$, we obtain
\begin{align*}
\term{M_2} \leq  \temp^4 \PE_{\randState{c}} \left[ \left(\sum_{t=0}^{\lentrunc-1} 2t \discount^{t} |\log(\policy_{\theta}(\varaction{c,1}{t}, \varstate{c,1}{t}))| \right)^4\right] \eqsp.
\end{align*}
Next, applying Cauchy-Schwarz inequality gives
\begin{align*}
\term{M_2} &
\leq \temp^4 \PE_{\randState{c}} \left[ \left(\sum_{t=0}^{\lentrunc-1} 2t \discount^{t/2} \discount^{t/2} |\log(\policy_{\theta}(\varaction{c,1}{t}, \varstate{c,1}{t}))| \right)^4\right] \\
&
\leq \temp^4 \left(\sum_{t=0}^{\lentrunc-1} 4t^2 \discount^{t} \right)^2 \PE_{\randState{c}} \left[  \left(\sum_{t=0}^{\lentrunc-1}  \discount^{t} |\log(\policy_{\theta}(\varaction{c,1}{t}, \varstate{c,1}{t}))|^2 \right)^2\right] \\
&
\leq \temp^4 \left(\sum_{t=0}^{\lentrunc-1} 4t^2 \discount^{t} \right)^2 \PE_{\randState{c}} \left[  \frac{(1- \discount^{\lentrunc})^2}{(1- \discount)^2}\left( \frac{1- \discount}{1- \discount^{\lentrunc}}\sum_{t=0}^{\lentrunc-1}  \discount^{t} |\log(\policy_{\theta}(\varaction{c,1}{t}, \varstate{c,1}{t}))|^2 \right)^2\right] 
\eqsp.
\end{align*}
For the first sum, using that for any $x \in [0,1[$, $\sum_{k=0}^{\infty}k^2 x^k \leq 2/(1-x)^3$, and for the second sum using Jensen's inequality gives
\begin{align}
\nonumber
\term{M_2} &\leq \temp^4 \frac{64}{(1- \discount)^6} \PE_{\randState{c}} \left[  \frac{1- \discount^{\lentrunc}}{1- \discount}\sum_{t=0}^{\lentrunc-1}  \discount^{t} |\log(\policy_{\theta}(\varaction{c,1}{t}, \varstate{c,1}{t}))|^4 \right] \\
& \label{eq:plot_m_2}
\leq  \temp^4 \frac{64}{(1- \discount)^7} \sum_{t=0}^{\lentrunc-1}  \discount^{t} \PE_{\randState{c}} \left[  |\log(\policy_{\theta}(\varaction{c,1}{t} \mid \varstate{c,1}{t}))|^4 \right] \eqsp.
\end{align}
Denote by $\pA$ the set of probability distributions on $\A$. Note that for any policy. Note, that, we have
\begin{align}
\label{eq:bound_constrained_problem}
\max_{\pi \in \pA} -\sum_{a\in \A} \policy (a\mid s) \log(\policy(a \mid s))^4  = (\log(\nactions))^4 \eqsp.
\end{align}
Plugging in the previous bound in \eqref{eq:plot_m_2} yields
\begin{align*}
\term{M_2} \leq \frac{64\temp^4}{(1- \discount)^8} \log(\nactions)^4 \eqsp.
\end{align*}
Combining the bounds on $\term{M_1}$ and $\term{M_2}$ gives the following bound on $\term{M}$.
\begin{align*}
\term{M} \leq 16 \term{M_1} +16 \term{M_2}  \leq \frac{256}{(1- \discount)^8} + \frac{1024\temp^4}{(1- \discount)^8} \log(\nactions)^4 \eqsp.
\end{align*}
Plugging in the previous bound in \eqref{eq:main_bound_fourth_central_moment} concludes the proof.
\end{proof}
 We first show that, in general, the objective $\auxobjective$ does not have Łojasiewicz structure.
\begin{figure}
    \centering
    \includegraphics[width=0.62\linewidth]{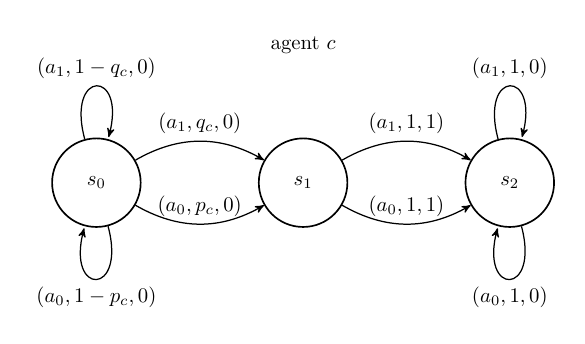}
    \caption{FRL task with an objective that admits strictly local minimas. The triplet means (action, probability, reward) , $\discount = 0.999$, and $\temp = 1$. If the action is not specified, it means that all the actions give the same reward and lead to the same state.}
    \label{frl_is_not_concave_neither_pl}
\end{figure}
\begin{lemma}\label{lem:counterexample_pl_regularized}
There exists an FRL instance where the objective $\auxobjective$ admits a strict local minima.
\end{lemma}
\begin{proof}
Consider the FRL task defined in \Cref{frl_is_not_concave_neither_pl}. Define $x \eqdef \policy_{\theta}(a_1|s_0)$.
From the flow conservation constraints for occupancy measures for any agent $c \in [\nagent]$, it holds that
\begin{align*}
\occupancy[c][\initdist, \theta](s_2) &= \discount \occupancy[c][\initdist, \theta](s_1) \eqsp, \quad
\occupancy[c][\initdist, \theta](s_1) = \discount \left( q_c x + (1-x) p_c\right) \occupancy[c][\initdist, \theta](s_1)\\
\occupancy[k][\initdist, \theta](s_0) &= (1- \discount) + \discount \left( (1-q_c)x + (1-p_c)(1-x) \right) \occupancy[c][\initdist, \theta](s_0)\eqsp.
\end{align*}
Rearranging the precedent terms yields
\begin{align*}
\occupancy[c][\initdist, \theta](s_0) \eqdef \frac{1- \discount}{1-\discount \left( (1-q_c)x + (1-p_c)(1-x) \right)}  \eqsp,
\end{align*}
which implies
\begin{align*}
\occupancy[c][\initdist, \theta](s_1) \eqdef \frac{\discount(1- \discount)\left( q_c x + (1-x) p_c \right)}{1-\discount \left( (1-q_c)x + (1-p_c)(1-x) \right)} = \frac{\discount(1- \discount)\left( p_c + (q_c-p_c)x\right)}{1-\discount \left( 1-p_c + x(p_c-q_c) \right)}  \eqsp.
\end{align*}
The value of the objective function is thus
\begin{align*}
\auxlocfunc[c](\theta) = \frac{\temp}{1-\discount}\bigg[    \occupancy[c][\initdist, \theta](s_0) H(x)  + \frac{\occupancy[c][\initdist, \theta](s_1)}{\temp}  +  \occupancy[c][\initdist, \theta](s_1) H (\policy_\theta(a_1 |s_1)) + \occupancy[c][\initdist, \theta](s_2) H(\policy_\theta(a_1 | s_2))
\bigg] \eqsp,
\end{align*}
where for any $y \in (0,1), H(y) \eqdef - y \log y - (1-y) \log (1-y)$. Now, let us assume that the policy for states $s_1$ and $s_2$ is uniform since it is an optimal solution given any values of $p_c$ and $q_c$, and in this case we have $H(\policy_\theta(a_1 | s_1)) = H(\policy_\theta(a_1 | s_2)) = \log2$. Then, let us define a value $f(x;p_c,q_c) = p_c + (q_c - p_c) x$, where $x = \sigma(\theta)$ for $\sigma(\theta) = \frac{1}{1+\exp(-\theta)}$ is a sigmoid parametrization. 

Thus, after plugging in a value of our occupancy measures in our MDP, we have
\[
    \auxlocfunc[c](\theta) = \frac{\tau H(\sigma(\theta)) + \gamma \cdot f(\sigma(\theta); p_c,q_c) \cdot \left(1 + \tau \log 2 + \gamma \tau \log 2\right) }{1-\gamma(1-f(\sigma(\theta);p_c,q_c))}.
\]
The plot of $\auxobjective$ (for $ \nagent=2$, $p_1 = 0$, $q_1=1$, $p_2=0.99$, $q_2 = 0.01$, $\discount = 0.999$, and $\temp = 1$) in \Cref{fig:function_non_pl} shows that this problem does not have a Łojasiewicz structure.
\end{proof}

However, each agent locally satisfies a Łojasiewicz-type property

\begin{lemma}[Lemma 15 of \cite{mei2020global}]
\label{lem:pl_structure_regularised}
Assume \assumptionmdp. For any agent $c\in [\nagent]$, denote by  $\policy_{c,\temp}^{\star}$ the unique optimal regularized policy (see e.g \citet{nachum2017bridging} for the proof of existence and unicity) of this agent. It holds
\begin{align*}
\norm{\nabla \auxlocfunc[c](\theta)}[2]^2 \geq 2 \regsoftmu[c](\theta) \left[ \optauxlocfunc[c] - \auxlocfunc[c](\theta)\right]
\eqsp, %
\end{align*}
where $\regsoftmu[c](\theta)$ is defined as
\begin{align*}
    \regsoftmu[c](\theta)
    =  \frac{\temp \min_{s} \occupancy[c][\initdist, \policy_{\theta}](s)\min_{s,a} \policy_{\theta}(a|s)^2 }{\nstates(1- \discount)}  \left\| \frac{ \occupancy[c][\initdist,\policy_{c,\temp}^{\star}]}{ \occupancy[c][\initdist, \theta]} \right\|_{\infty}^{-1}
    \eqsp.
\end{align*}
\end{lemma}
\begin{figure}
     \centering
 \includegraphics[width=0.85\textwidth]{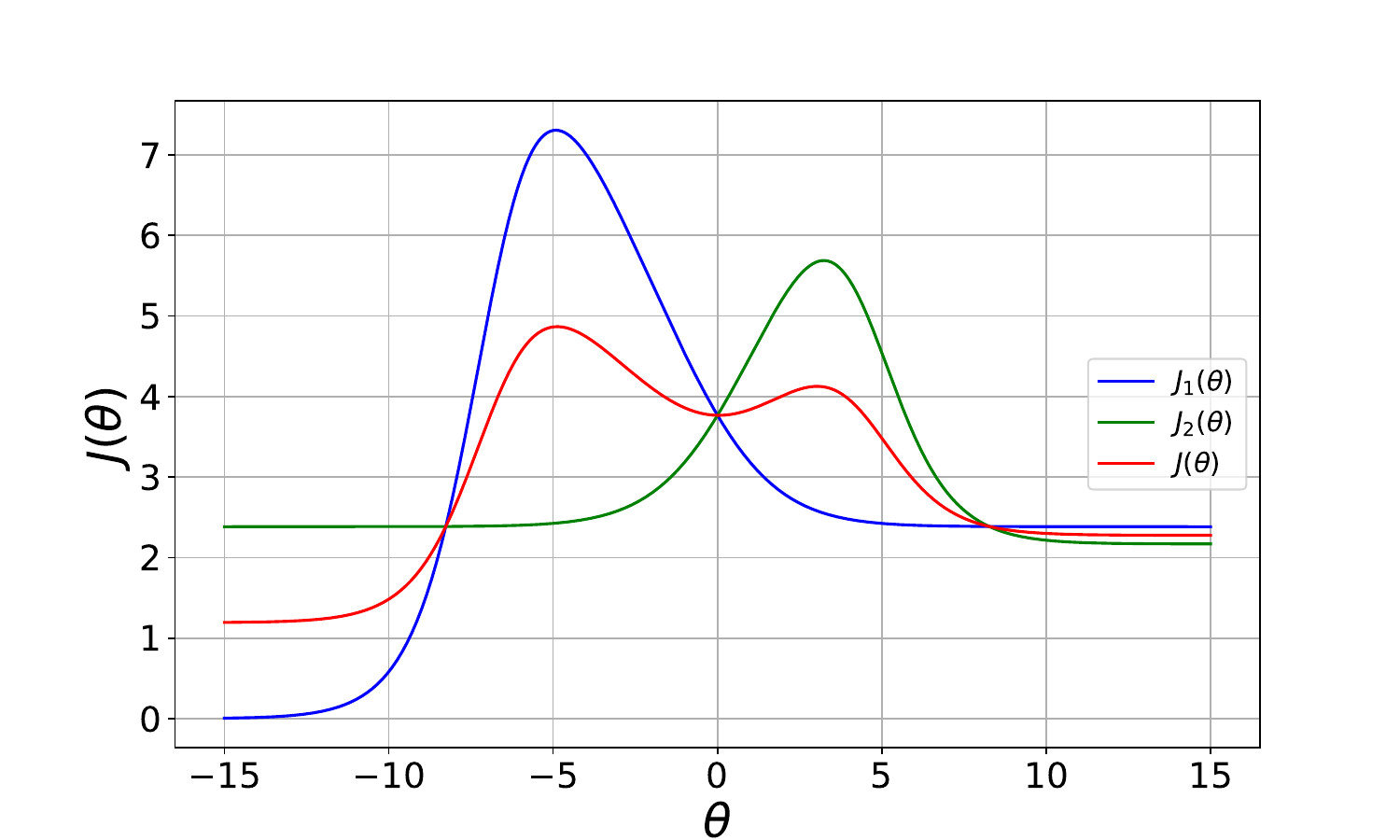}
 \caption{An example that shows that FRL objective $\auxobjective$ does not necessarily have a Łojasiewicz structure.}
\vspace{-1.em}
\label{fig:function_non_pl}
\end{figure}

\subsection{Construction of the projection operator}
\label{subsec:monotone_improvement_operator}

The goal of this section is therefore to show the existence of an operator $\mathcal{U}$  with two crucial properties:  
(i) for any policy and agent, applying this operator produces a new policy with a higher regularized value, and;  
(ii) Every policy generated by this operator assigns at least a fixed minimum probability to every action. The main idea is to build the improvement operator such that it slightly augments the smallest probability weights, such that for any state action pair $(s,a) \in \S\times \A$ the probability $\policy(a|s)$ stays above a certain threshold.  We will show below that this procedure improves the global objective while keeping the probabilities uniformly bounded away from $0$ when the threshold is properly chosen. 
For any policy $\policy$, state $s \in \S$, $\tau < 1/(2\nactions^2)$, we respectively define $\A_{\tau}^{\policy}(s)$, and $a_{\max}^{\policy}(s)$ as 
\begin{align*}
\A_{\tau}^{\policy}(s) \eqdef \left\{ a \in \A, \policy(a|s) \leq \tau/2 \right\} \eqsp, \quad a_{\max}^{\policy}(s) = \argmax_{a\in \A} \{ \policy(a|s)\} \eqsp,
\end{align*}
where the $\argmax$ is chosen at random in the case of ties.
Note that the definition of $\sthreshold$ ensures that $a_{\max}^{\policy}(s)$ does not belong to the set $\A_{\tau}^{\policy}(s)$ as
\begin{align*}
\max_{a \in \A} \policy(a|s) \geq 1/\nactions \eqsp.
\end{align*}
Finally, we define the improvement operator as follows:
\begin{equation}
\label{def:improvement_operator_bounded_case}
\begin{aligned}
\mathcal{U}_\tau : \pA^{\S} &\;\longrightarrow\; \pA^{\S}, \\
\policy &\;\longmapsto\; \mathcal{U}_\tau(\policy),
\end{aligned}
\end{equation}
\noindent
where for every $(s,a) \in \S\times\A$,
\begin{equation*}
\mathcal{U}_\tau(\policy)(a|s) \;=\;
\begin{cases}
\tau, & \text{if } \policy(a|s) \leq \tau/2, \\[6pt]
\policy(a|s) 
   - \displaystyle\sum_{b \in \A_{\tau}^{\policy}(s)}
     \bigl(\tau - \policy(b|s)\bigr),
   & \text{if } a = a_{\max}^{\policy}(s), \\[10pt]
\policy(a|s), & \text{otherwise}.
\end{cases}
\end{equation*}
The operator $\mathcal{U}_\tau$ builds $\mathcal{U}_\tau(\policy)(a|s)$ by (statewise) raising each $a \in \A_{\tau}^{\policy}(s)$ to $\tau$, substracting the total added mass from the single action  $ a_{\max}^{\policy}(s)$, and leaving other actions unchanged. If $\A_{\tau}^{\policy}(s) = \emptyset$, for all $s \in \S$, then $\mathcal{U}_\tau(\policy)= \policy$. Note that mass conservation is immediate from the definition and the fact that $\tau < 1/(2\nactions^2)$. Non-negativity of $\mathcal{U}_\tau(\policy)(a_{\max}^{\policy}(s)|s)$ follows because the removed mass is 
\[
\sum_{a \in \A_{\tau}^{\policy}(s)} 
\{\tau - \policy(a|s)\} \leq  \tau \times \nactions \leq \frac{1}{2\nactions}
\]
Since $\policy(a_{\max}^{\policy}(s)|s) \geq 1/\nactions$, we get that $\mathcal{U}(\policy)(a_{\max}^{\policy}(s)|s) \geq1/(2\nactions)$. This in particular shows that $\mathcal{U}_\tau(\policy)$ is a policy. Next, define
\begin{align}
\label{def:threshold_improvement}
\sthreshold \eqdef \min \left(  \frac{1}{3} \exp\left( -\frac{16 + 8 \discount\temp \log(\nactions)}{\temp(1- \discount)^2 \initdistmin}\right), \frac{1}{3^8 \nactions^4}\right) \eqsp.
\end{align}
The following lemma establishes the crucial improvement property when $\tau = \sthreshold $. 
\begin{lemma}
\label{lem:improvement_lemma_appendix}
Assume that the initial distribution $\initdist$ satisfies \assumptionmdp. For any policy $\policy$, for any agent $c \in [\nagent]$, it holds that
\begin{align*}
\regvaluefunc[c][\mathcal{U}_{\sthreshold}(\policy)](\initdist) \geq \regvaluefunc[c][\policy](\initdist)  \eqsp.
\end{align*}
Additionally, for any policy $\policy$, we have that
\begin{align*}
\mathcal{U}_{\sthreshold}(\policy)(a|s) \geq \sthreshold/2 \eqsp.
\end{align*}
\end{lemma}
\begin{proof}
Set an arbitrary policy $\policy$. For avoiding heavy notations, we will, through this proof, denote by $\A_{\tau}^{\policy} = \A_{\sthreshold}^{\policy}$. We consider the case where there is $s \in \S$ such that $\A_{\tau}^{\policy}(s) \neq \emptyset$ (alternatively $\mathcal{U}_{\sthreshold}(\policy) = \policy$, which makes the previous inequality immediately valid). Define $\truncpolicy = \mathcal{U}_{\sthreshold}(\policy)$. The following applies
\begin{align*}
\regvaluefunc[c][\truncpolicy](\initdist) 
- \regvaluefunc[c][\policy](\initdist) &= \sum_{s \in \S} \occupancy[c][\initdist,\truncpolicy](s)  \sum_{a \in \A} \left[\truncpolicy(a|s)\rewardMDP(s, a) - \temp \truncpolicy(a|s)\log\left(\truncpolicy(a|s)\right) \right]
\\
&\hspace{72pt}- \sum_{s \in \S} \occupancy[c][\initdist,\policy](s) \sum_{a \in \A} \left[\policy(a|s)\rewardMDP(s, a) - \temp \policy(a|s)\log \left(\policy(a|s)\right) \right]
\\
&= \underbrace{\sum_{s \in \S} \left(\occupancy[c][\initdist,\truncpolicy](s) - \occupancy[c][\initdist,\policy](s) \right) \sum_{a \in \A} \left[\truncpolicy(a|s)\rewardMDP(s, a) - \temp \truncpolicy(a|s)\log\left(\truncpolicy(a|s)\right) \right]}_{\term{I}}
\\
&+ \underbrace{\sum_{s \in \S} \occupancy[c][\initdist,\policy](s)  \sum_{a \in \A} (\truncpolicy(a|s) - \policy(a|s) ) \rewardMDP(s, a)}_{\term{II}} 
\\
&+ \underbrace{\temp \sum_{s \in \S} \occupancy[c][\initdist,\policy](s)  \sum_{a \in \A} \left[ \policy(a|s)\log\left(\policy(a|s)\right)  - \truncpolicy(a|s)\log\left(\truncpolicy(a|s)\right)  \right]}_{\term{III}} \eqsp.
\end{align*}
We now lower-bound each of the three terms separately.
\paragraph{Bounding $\term{I}$.} Using \Cref{lem:bound_difference_stationary_occupancy_measure}, we have
\begin{align*}
\term{I} &\geq - \norm{\occupancy[c][\initdist,\truncpolicy] - \occupancy[c][\initdist,\policy]}[1] \max_{s \in \S} \left| \sum_{a \in \A} \left[\truncpolicy(a|s)\rewardMDP(s, a) - \temp \truncpolicy(a|s)\log\left(\truncpolicy(a|s)\right) \right]\right|
\\
&\geq - \frac{\discount}{1-\discount} \sup_{s \in \S} \norm{\truncpolicy(\cdot|s) - \policy(\cdot|s)}[1] \left( 1+ \temp \log(\nactions)\right) \\
&\geq - \frac{2\discount}{1-\discount} \sthreshold \max_{s \in \S} \left\{ \sum_{a \in \A_{\tau}^{\policy}(s)} 1\right\} \left( 1+ \temp \log(\nactions)\right) \eqsp,
\end{align*}
where in the last inequality we used that (because we increase the probability of the actions in  $A_{\tau}^{\policy}(s)$ by $\sthreshold$ and remove the total added mass from the probability of $\policy(a_{\max}^{\policy}(s)$)
\begin{align*}
\sup_{s \in \S} \norm{\truncpolicy(\cdot|s) - \policy(\cdot|s)}[1] \leq 2 \max_{s \in \S} \left\{ \sum_{a \in A_{\tau}^{\policy}(s)} \right\} \sthreshold \eqsp.
\end{align*}
\paragraph{Bounding $\term{II}$.} Using the triangle inequality yields
\begin{align*}
\term{II} \geq -\sup_{s\in \S} \norm{\truncpolicy(\cdot|s) - \policy(\cdot|s)}[1] \geq -2 \max_{s \in \S} \left\{ \sum_{a \in A_{\tau}^{\policy}(s)} 1 \right\} \sthreshold  \eqsp.
\end{align*}
\paragraph{Bounding $\term{III}$.} 
All the state-action pairs on which the original $\policy$ allocates the same probability then the policy $\truncpolicy$ are equal to $0$ in $\term{III}$ allowing us to simplify this term
\begin{align*}
\term{III} &= \temp \sum_{s \in \S} \occupancy[c][\initdist,\policy](s)  \sum_{a \in \A} \left[ \policy(a|s)\log\left(\policy(a|s)\right)  - \truncpolicy(a|s) \log\left(\truncpolicy(a|s)\right)  \right]
\\
&= \temp \sum_{s \in \S} \occupancy[c][\initdist,\policy](s)  \sum_{a \in \A_{\tau}^{\policy}(s)} \left[ \policy(a|s)\log\left(\policy(a|s)\right)  - \truncpolicy(a|s)\log(\truncpolicy(a|s))  \right]
\\
&+ \temp \sum_{s \in \S} \Ind(\A_{\tau}^{\policy}(s) \neq \emptyset) \occupancy[c][\initdist,\policy](s)\left[ \policy(a_{\max}^{\policy}(s)|s) \log\left(\policy(a_{\max}^{\policy}(s)|s)\right)  -\truncpolicy(a_{\max}^{\policy}(s)|s) \log\left(\truncpolicy(a_{\max}^{\policy}(s)|s)\right)  \right] \eqsp.
\end{align*}
Since $x\mapsto x \log(x)$ is convex, for all $u,v \in [0;1]$, $u\log(u)- v\log(v) \geq \left[\log(v)+1\right](u-v)$, we have
\begin{align*}
\term{III} &\geq \temp \sum_{s \in \S} \occupancy[c][\initdist,\policy](s)  \sum_{a \in \A_{\tau}^{\policy}(s)} (\policy(a|s) - \truncpolicy(a|s))\left[\log(\sthreshold) +1 \right] \qquad \text{(since $\truncpolicy(a|s) = \sthreshold$)}
\\ &+  \temp \sum_{s \in \S} \Ind(\A_{\tau}^{\policy}(s) \neq \emptyset) \occupancy[c][\initdist,\policy](s)  \left[ \policy(a_{\max}^{\policy}(s)|s) -  \truncpolicy(a_{\max}^{\policy}(s)|s)  \right] \left[\log\left( \truncpolicy(a_{\max}^{\policy}(s)|s)\right) + 1\right] \eqsp,
\end{align*}
Next, using that 
\begin{align*}
\truncpolicy(a_{\max}^{\policy}(s)|s)
 \geq \policy(a_{\max}^{\policy}(s)|s) - \nactions\sthreshold  \geq  \frac{1}{\nactions} - \frac{1}{2\nactions}
= \frac{1}{2\nactions} \eqsp,
\end{align*}
combined with the monotonicity of $x \colon \log(x) +1$ and the fact that $  \policy(a_{\max}^{\policy}(s)|s) -  \truncpolicy(a_{\max}^{\policy}(s)|s)   \geq 0$ yields
\begin{align*}
\term{III} &\geq \temp \sum_{s \in \S} \occupancy[c][\initdist,\truncpolicy](s)  \sum_{a \in \A_{\tau}^{\policy}(s)} (\policy(a|s) - \truncpolicy(a|s))\left[\log(\sthreshold) +1 \right]
\\ &+  \temp \sum_{s \in \S} \Ind(\A_{\tau}^{\policy}(s) \neq \emptyset) \occupancy[c][\initdist,\truncpolicy](s)  \left[ \policy(a_{\max}^{\policy}(s)|s) -  \truncpolicy(a_{\max}^{\policy}(s)|s)  \right] \left[ \log(\frac{1}{2\nactions}) + 1 \right] \eqsp,
\end{align*}
Additionally, since
\begin{align*}
0 \leq \policy(a_{\max}^{\policy}(s)|s) -  \truncpolicy(a_{\max}^{\policy}(s)|s) \leq  \sum_{a \in \A_{\tau}^{\policy}(s)} (\policy(a|s) - \truncpolicy(a|s)) \leq  \sthreshold \sum_{a \in \A_{\tau}^{\policy}(s)} 1 \eqsp,
\end{align*}
implies
\begin{align*}
\term{III} &\geq - \frac{\temp}{2} \sum_{s \in \S} \occupancy[c][\initdist,\truncpolicy](s)  \Ind(\A_{\tau}^{\policy}(s) \neq \emptyset)\left(\sum_{a \in \A_{\tau}^{\policy}(s)} 1 \right)\sthreshold  \left[ \log(\sthreshold) +1 \right]
\\ &-  \temp \sum_{s \in \S} \occupancy[c][\initdist,\truncpolicy](s)  \Ind(\A_{\tau}^{\policy}(s) \neq \emptyset)  \left(\sum_{a \in \A_{\tau}^{\policy}(s)}  1\right) \sthreshold [\log(2\nactions)+1] \eqsp,
\\ &\geq - \frac{\temp}{4} \sum_{s \in \S} \occupancy[c][\initdist,\truncpolicy](s)  \Ind(\A_{\tau}^{\policy}(s) \neq \emptyset)\left(\sum_{a \in \A_{\tau}^{\policy}(s)} 1 \right) \sthreshold  \left[\log(\sthreshold) +1 \right] \eqsp,
\end{align*}
where in the last inequality, we used that $\sthreshold \leq \frac{1}{3^8 \nactions^4} \leq \exp(-4\log(2\nactions) -5)$.  Hence, by using \assumptionmdp, we can lower bound this term as follows
\begin{align*}
\term{III} \geq - \frac{\temp}{4} (1- \discount) \initdistmin \max_{s \in \S}  \left\{ \sum_{a \in A_{\tau}^{\policy}(s)} 1\right\}\sthreshold  \left[\log(\sthreshold) +1 \right] \eqsp.
\end{align*}
Collecting these lower bounds and using that
\begin{align*}
\left[\log(\sthreshold) +1 \right] \leq -\frac{16 + 8 \discount\temp \log(\nactions)}{\temp(1- \discount)^2 \initdistmin}
\end{align*}
concludes the proof.
\end{proof}
Finally, we define the operator that maps each policy to one corresponding parameter
\[
\policytoparam : \Pi \;\to\; \logitspace
\]
by
\begin{align}
\label{def:translation_policy_param}
\policytoparam (\policy)(s,a) 
\;\eqdef\; \log(\policy(a|s)),
\quad \text{for all} (s,a)\in\mathcal{S}\times\mathcal{A} \,.
\end{align}
Finally, we define the improvement operator on the logitspace as 
\begin{align*}
\projset_{\tau} \eqdef \policytoparam \circ \projset_\tau \eqsp. 
\end{align*} 
The following lemma shows that $\policytoparam_{\tau}$ successfully recovers a parameter that gives the policy and that $\projset_{\tau}$ improves the value of the objective when $\temp = \sthreshold$.
\begin{lemma}
\label{lem:improvement_on_theta}
Assume that the initial distribution $\initdist$ satisfies \assumptionmdp. For any policy $\policy$, it holds that
\begin{align*}
\policy_{\policytoparam(\policy)} = \policy \eqsp,
\end{align*}
Additionally, for any $\param \in \logitspace$ and $(s, a)\in \S \times \A$, we have that
\begin{align*}
\regvaluefunc[\projset_{\sthreshold}(\param)](\initdist) \geq  \regvaluefunc[\param](\initdist) \eqsp, \quad \policy_{\projset_{\sthreshold}(\param)} \geq  \sthreshold \eqsp.
\end{align*}
\end{lemma}
\begin{proof}
The proof follows immediately from the definition of the softmax policy, from \eqref{def:translation_policy_param}, and  \Cref{lem:improvement_lemma_appendix}.
\end{proof}

\subsection{Convergence rates,  sample and communication complexities}

Firstly, define 
\begin{align}
\label{def:min_coeff_pl_reg_case_appendix}
\minregsoftmu(\param) \eqdef \min_{c \in [\nagent]} \regsoftmu[c](\theta) \eqsp, \quad \minminregsoftmu \eqdef \temp(1-\discount) \initdistmin^2 \sthreshold^2/\nstates\eqsp.
\end{align}
where $\regsoftmu[c](\theta)$ is defined \Cref{lem:pl_structure_regularised}. Applying \Cref{theorem:sto_fedavg_alpha_2} and \Cref{thm:complexity-linear-pl} yields the following convergence results.
\begin{theorem}[Convergence rates of $\RegSoftfedPG$] 
\label{theorem:sto_fedavg_alpha_1}
Assume \assumptionmdp\, and that the projection operator is $\projset_{\sthreshold}$. For any $\step >0$ such that $\step \lsteps \regsoftsmooth \leq 1/74$, the iterates of $\RegSoftfedPG$ satisfy
\begin{align*}
\optauxobjective\!-\! \PE[\auxobjective(\theta^{\trounds})] &\leq \left(1 -\frac{\step \lsteps \minminregsoftmu }{2}\right)^{\trounds} \! (\optauxobjective \!-\! \auxobjective(\theta^{0}) + \frac{3 \step \regsoftsmooth \regsoftmoments{2}{2}}{ \nagent \minminregsoftmu } + \frac{\regsofthgty^2}{\minminregsoftmu } \\
& + 4 \frac{\regsoftbias^2}{\minminregsoftmu }
+ \frac{8 \cdot 12^3 \step^4 \regsoftthirddbound^2 \lsteps(\lsteps-1) \regsoftmoments{4}{4}}{\minminregsoftmu }
\eqsp.
\end{align*}
\end{theorem}
\begin{proof}
First, note that the combination of all the lemmas of \Cref{subsec:applying_ascent_lemma_regsoft} shows that Assumptions~\Cref{assumFL:uniform_grad_bound} to \Cref{assumFL:bias} , \Cref{assum:local_linear_pl_appendix}, and \Cref{assum:improvement_linear_pl_appendix} holds. Next, note that if $\step \lsteps \regsoftsmooth \leq 1/74$ then it holds that $32 \step^2 \lsteps^2 \regsoftthirddbound^2 \regsoftfourmoment^2 \le \regsoftsmooth^2$ (as $32\regsoftthirddbound^2 \regsoftfourmoment^2 \leq 74^2 \regsoftsmooth^4$ by \Cref{lem:smoothness_regularised_value_func}, \Cref{lem:regsoft_bound_grad}, and \Cref{lem:third_derivative_bound_reg}). Thus, applying \Cref{theorem:sto_fedavg_alpha_2} concludes the proof.
\end{proof}
Recall that
\begin{gather*}
    \regsoftboundgrad
    = 
    \frac{1+ \temp \log(\nactions)}{(1- \discount)^{2}}
    ~,~~
    \regsoftsmooth = \frac{8+ \temp(4+ 8\log(\nactions)}{(1-\gamma)^{3}}
    ~,~~
    \regsoftthirddbound =
    \frac{480 + 832 \temp \log |\A|}{(1-\discount)^{4}}
    ~,~~
    \\
    \regsofthgty^2 = 
    \frac{56(1+\temp \log(\nactions))^2\hgkernel^2}{(1- \discount)^6}  + \frac{36 \hgreward^2}{(1-\discount)^4}
    ~,~~
    \regsoftbias = 
    \frac{2 (1+ \temp \log(\nactions))\discount^{\lentrunc} \lentrunc}{1-\discount} + \frac{2 (1+ \temp \log(\nactions))\discount^{\lentrunc}}{(1-\discount)^{2}}
    \eqsp,
    \\
    \regsoftmoments{2}{2}
    = \frac{12 + 24 \temp^2(\log(\nactions))^2 }{(1- \discount)^{4} \sizebatch}
    ~,~~ 
    \regsoftmoments{4}{4} 
    = 
    \frac{(1120 + 4480 \temp^4\log(\nactions)^4) }{(1- \discount)^{8} \sizebatch^2}
    \eqsp,
\end{gather*}
which are defined respectively in \Cref{lem:smoothness_estimator_regsoft,lem:regsoft_bound_grad,lem:third_derivative_bound_reg,lem:bounded_gradient__regularised_frl,lem:bias_and_variance_regularised_stochastic_gradient,lem:bound_fourth_moment_regsoft}.
We obtain the following explicit result.
\begin{corollary}[Explicit Convergence Rate of $\RegSoftfedPG$] 
\label{thm:complexity-regsoftmaxfedpg_appendix}
Under the assumptions of \Cref{theorem:sto_fedavg_alpha_1}, let $\step >0$ such that $\step \lsteps \leq 888^{-1}(1-\discount)^{3}(1+\temp \log(\nactions))^{-1}$, and $\lentrunc \geq 1/(1 - \discount)$. Then, the iterates of $\RegSoftfedPG$ satisfy
\begin{align*}
\optauxobjective\!-\! \PE[\auxobjective(\theta^{\trounds})] &\leq \left(1 -\frac{\step \lsteps \minminregsoftmu }{2}\right)^{\trounds} \! (\optauxobjective \!-\! \auxobjective(\theta^{0}) +  \frac{864 \step(1+ \temp \log(\nactions)^3}{ \sizebatch \nagent \minminregsoftmu  (1- \discount)^7} + \frac{56(1 + \temp \log(\nactions))^2 \hgkernel^2}{\minminregsoftmu (1- \discount)^6} 
\\
&+ \frac{36(1 + \temp \log(\nactions))^2 \hgreward^2}{\minminregsoftmu (1- \discount)^6} 
 +  \frac{16 (1+ \temp \log(\nactions))^2 \discount^{2\lentrunc}\lentrunc^2}{\minminregsoftmu (1- \discount)^2}
+ \frac{ 51^8 \step^4 \lsteps(\lsteps-1) (1+ \temp \log(\nactions))^6}{\minminregsoftmu  \sizebatch^2 (1- \discount)^{16}}
\eqsp,
\end{align*}
where we recall that $\minminregsoftmu$ is defined in \eqref{def:min_coeff_pl_reg_case_appendix}.
\end{corollary}

\begin{corollary}[Sample and Communication Complexity of $\RegSoftfedPG$] 
\label{thm:complexity-regsoftmaxfedpg}
Under the assumptions of \Cref{theorem:sto_fedavg_alpha_1}, let
\begin{align*}
\epsilon \geq \frac{224 (1+ \temp \log(\nactions))^2\hgkernel^2}{\minminregsoftmu(1- \discount)^{6}} + \frac{144 \hgreward^2}{(1-\discount)^4 \minminregsoftmu} + \frac{256(1+\lambda \log(\nactions))^2\discount^{2\lentrunc}\lentrunc^2}{(1-\discount)^2} \eqsp.
\end{align*}
and
\begin{align*}
\step \le \min\Big( \frac{(1-\discount)^3}{72(1+\temp\log(\nactions)))}, \frac{\minlmufl  \epsilon \nagent \sizebatch(1-\discount)^7}{3456(1+\temp \log(\nactions))^3}, \frac{\minlmufl^{1/2} \sizebatch (1-\discount)^{7/2} \epsilon^{1/2}}{9^{6} (1+ \temp \log(\nactions))^{3/2}} \Big)
\eqsp.
\end{align*}
Then $\RegSoftfedPG$ achieves $\optauxobjective - \PE[\auxobjective{\theta^{\trounds}}] \leq \epsilon$, with a number of communication
\begin{align*}
\trounds \ge 
\frac{12}{ \minlmufl }
\log\Big(
\frac{4(\optauxobjective - \PE[\auxobjective{\theta^{0}}])}{\epsilon}
\Big) \frac{104(1+ \temp \log(\nactions))}{(1-\discount)^3}
\eqsp,
\end{align*}
for a total number of sampled trajectories per agent of
\begin{align*}
\trounds \lsteps \sizebatch
\ge 
\frac{2}{\minlmufl } 
\log\Big(
\frac{4(\optauxobjective - \PE[\auxobjective{\theta^{0}}])}{\epsilon}
\Big)\max\Big( \frac{72(1+\temp\log(\nactions))) \sizebatch}{(1-\discount)^3}, \frac{3456(1+\temp \log(\nactions))^3}{\minlmufl  \epsilon \nagent (1-\discount)^7}, \frac{9^{6} (1+ \temp \log(\nactions))^{3/2}}{\minlmufl^{1/2} (1-\discount)^{7/2} \epsilon^{1/2}}\Big)
\eqsp.
\end{align*}
\end{corollary}

%% file: AISTATS/appendix/counter_examples.tex
\label{sec:proof_theorem_upset}

The goal of this section is to prove \Cref{thm:biggest_upset_of_frl}. For clarity and readability, we prove each statement of the theorem in a separate lemma. First, we define the value function of an agent $c \in [\nagent]$, of a policy $\policy \in \Pi$, and for an initial distribution $\initdist$ as
\begin{align}
\label{def:valeu_function_non_stationnary} 
 \valuefunc[c][\policy](\initdist) \eqdef \PE_{\policy} \left[ \sum_{t=0}^{\infty} \discount^t \rewardMDP(\varstate{c}{t},\varaction{c}{t}) \right]\eqsp,
\end{align}
where $\PE_{\policy}[\cdot]$ is the expectation over random trajectories generated by following a policy $\policy = (\policy^t)_{t\in \nset}$: the initial state is sampled from a distribution $\varstate{c}{0} \sim \rho(\cdot)$ and $\forall t \geq 0:$  $\varaction{c}{t} \sim \policy^t(\cdot | \varglobhistory{t}, c), \varstate{c}{t+1} \sim \kerMDP[c](\cdot | \varstate{c}{t},\varaction{c}{t})$, for $\varglobhistory{t} = (\varlochistory{c}{t})_{c \in [\nagent]}$ where $\varlochistory{c}{t} = (\varstate{c}{0}, \varaction{c}{0}, \ldots, \varstate{c}{t})$ for all $c \in [\nagent]$.
\begin{lemma}
\label{lem:localstatvslocal}
There exists an FRL instance such that any stochastic stationary policy is suboptimal with respect to some history-dependent policy.
\end{lemma}
\begin{proof}
We consider the FRL task described in \Cref{fig:local_stat_vs_local}  with  $\initdist = (1, 0, 0, 0)$. 
\begin{figure}
    \centering
    \includegraphics[width=0.85\linewidth]{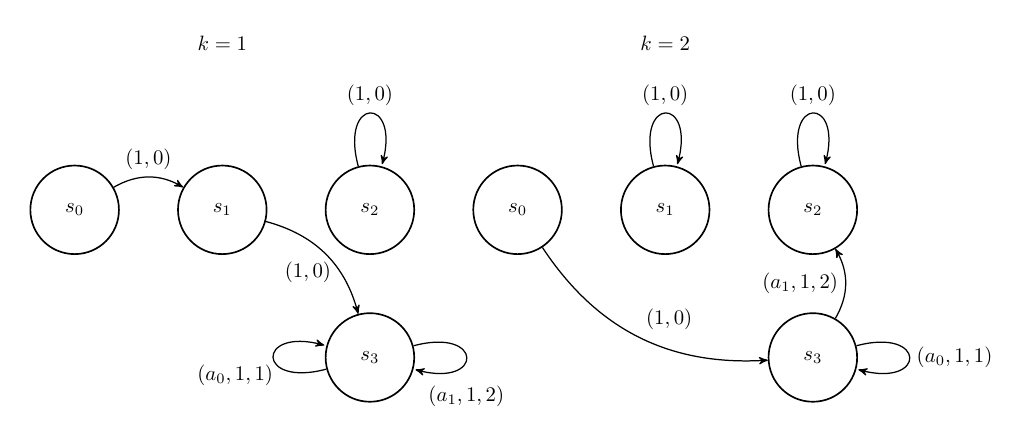}
    \caption{FRL task with no optimal stationary policy.  The triplet means (action, probability, reward) and $\discount = 0.9$. If the action is not specified, it means that all the actions give the same reward and lead to the same state}
    \label{fig:local_stat_vs_local}
\end{figure}
We show here that it holds
\begin{align*}
\max_{\policy \in \PiSta} \frac{1}{2} \left( \valuefunc[1][\policy](s_0) + \valuefunc[2][\policy](s_0) \right) < \max_{\policy \in \Pi_{\ell}} \frac{1}{2} \left( \valuefunc[1][\policy](s_0) + \valuefunc[2][\policy](s_0) \right) \eqsp.
\end{align*}
Define the following history-dependent policy $\policy_{\ell} = (\policy_{\ell}^t)_{t\in \nset}$ that satisfies
\begin{align*}
\policy_{\ell}^1(a_0 |s_3) = 1 \cdot \Ind_{t=1} + 0 \cdot \Ind_{t \geq 2} \eqsp,
\end{align*}
which intuitively describes the policy that takes action $a_0$ at the instant where the second agent reaches the state $s_3$ and then takes action $a_1$ when the first agent reaches the state $s_3$. The (double of the) FRL objective of this policy is equal
\begin{align*}
 \valuefunc[1][\policy_{\ell}](s_0) + \valuefunc[2][\policy_{\ell}](s_0) = \frac{2\discount^2}{1- \discount} + \discount + 2\discount^2 \eqsp.
\end{align*}
\begin{figure}
    \centering
    \includegraphics[width=0.8\linewidth]{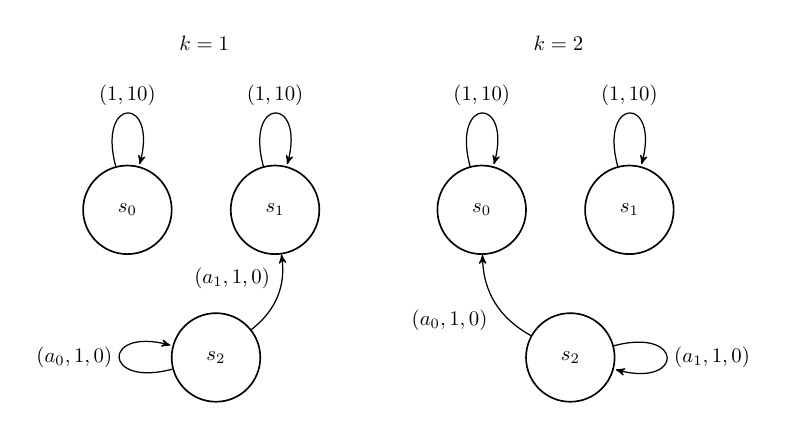}
    \caption{FRL task with no optimal deterministic policy. The triplet means (action, probability, reward) and $\discount = 0.9$. If the action is not specified, it means that all the actions give the same reward and lead to the same state}
    \label{fig:local_det_vs_local_stat}
\end{figure}
Let $\policy_{\mathrm{sta}}^{\star}$ be a stationary policy that maximizes $ \valuefunc[1][\policy](s_0) + \valuefunc[2][\policy](s_0) $ on the set of the stationary policies $\PiSta$. We define $p = \policy_{\mathrm{sta}}^{\star}(a_0|s_3)$. The (double of the) federated objective for this policy is
\begin{align*}
 \valuefunc[1][\policy_{\mathrm{sta}}^{\star}](s_0) + \valuefunc[2][ \policy_{\mathrm{sta}}^{\star}](s_0) = \sum_{k=2}^{\infty} \discount^k (1 \cdot p + 2 \cdot(1-p)) + \valuefunc[2][ \policy_{\mathrm{sta}}^{\star}](s_0) \eqsp. 
\end{align*}
The first instant at which the second agent takes actions $a_1$ follows a geometric distribution of parameter $1-p$. Thus, we have
\begin{align*}
\valuefunc[2][\policy_{\mathrm{sta}}^{\star}](s_0) &= \gamma \sum_{k=0}^\infty \left( (1 -p)^k p \cdot \left(\sum_{i=0}^{k-1} \discount^i \cdot 1 + 2 \gamma^k \right)\right) \\
&= \gamma \sum_{k=0}^\infty \left( (1 -p)^k p \cdot \left(\frac{1 - \gamma^k}{1- \gamma} + 2 \gamma^k \right)\right) \\
&= \frac{\discount}{1 -\discount}\sum_{k=0}^\infty \left( (1 -p)^k p\cdot (1 - \discount^k + 2 \gamma^k -2 \discount^{k+1} )\right) \\
&= \frac{\discount}{1 -\discount}\sum_{k=0}^\infty \left( (1 -p)^k p \cdot (1 + \gamma^k -2 \discount^{k+1} )\right) \\
&= \frac{\discount p}{1 -\discount}\sum_{k=0}^\infty \left( (1 -p)^k + ((1-p)\discount)^k -2 \discount ((1-p)\discount)^k )\right) \\
&= \frac{\discount p}{1- \discount}\left( \frac{1}{p} + \frac{1}{1 - \discount + p\discount} - \frac{2 \discount}{1 - \discount + p\discount} \right) \eqsp.
\end{align*}
By gathering the two precedent expressions, we get
\begin{align*}
 \valuefunc[1][\policy_{\mathrm{sta}}^{\star}](s_0) + \valuefunc[2][\policy_{\mathrm{sta}}^{\star}](s_0) &= \frac{(2-p)\discount^2}{1- \discount} + \frac{\discount p}{1- \discount}\left( \frac{1}{p} + \frac{1}{1 - \discount + p\discount} - \frac{2 \discount}{1 - \discount + p\discount} \right) \\
 &= \frac{1}{1- \discount} \left[(2-p)\discount^2 + \discount + (1 - 2\discount)\frac{\discount p}{1 - \discount + p \discount} \right]  \\
 &= \frac{1}{1- \discount} \left[(2-p)\discount^2 + \discount + (1 - 2\discount) - (1 - 2\discount)\frac{1- \discount}{1 - \discount + p \discount} \right] \\
 &\leq \frac{1}{1- \discount} \left[2\discount^2  + (1 - \discount) - (1 - 2\discount)\frac{1- \discount}{1 - \discount + p \discount} \right] \leq \frac{2\discount^2}{1-\discount} +2\discount \eqsp,
\end{align*}
where the last inequality holds as $\discount > 1/2$. As for any $\discount > 1/2$, we have $2 \discount < \discount + 2 \discount^2$ then this proves the suboptimality of the stationary policy $\policy_{\mathrm{sta}}^{\star}$ with respect to the history-dependent policy $\policy_{\ell}$.
\end{proof}

\begin{lemma}
\label{lem:localstatvslocaldet}
There exists an FRL instance such that any deterministic policy is suboptimal with respect to some stationary stochastic policy.
\end{lemma}
\begin{proof}
We consider the FRL task of \Cref{fig:local_det_vs_local_stat}, and we consider the setting where the two agents start from the state $s_2$, i.e, $\initdist = (0,0,1)$. We show here that it holds
\begin{align*}
\max_{\policy \in \PiDet} \frac{1}{2} \left( \valuefunc[1][\policy](s_2) + \valuefunc[2][\policy](s_2) \right) < \max_{\policy \in \PiSta} \frac{1}{2} \left( \valuefunc[1][\policy](s_2) + \valuefunc[2][\policy](s_2) \right) \eqsp.
\end{align*}
We define the stationary policy $\policy_{\mathrm{sta}}$ that satisfies $\policy_{\mathrm{sta}}(a_0|s_2) = 1/2$ and $\policy_{\mathrm{sta}}(a_1|s_2) = 1/2$. First, note that the probability of each agent being in state $s_2$ at time $t$, while following $\policy_{\mathrm{sta}}$, is $1/2^t$. Thus, the FRL objective of this policy is equal to
\begin{align*}
\frac{1}{2} \left(\valuefunc[1][\policy_{\mathrm{sta}}](s_2) + \valuefunc[2][\policy_{\mathrm{sta}}](s_2) \right) = \frac{10\discount}{1- \discount} - \frac{10\discount}{1- \discount/2} = \frac{6\discount}{1- \discount} + \frac{2\discount^2 -6\discount(1- \discount)}{(1- \discount)(1- \discount/2)} \geq \frac{6 \discount}{1- \discount}\eqsp,
\end{align*}
where the last inequality follows from the fact that $\discount =0.9$. Let $\policy_{\mathrm{det}}^{\star}$ be an optimal deterministic policy. We distinguish two cases

\textbf{Case  $\policy_{\mathrm{det}}^{\star}(s_2) = a_0$: } In this case, the second agent will reach state $s_0$ at first iteration, but the first agent will be stuck at $s_2$ where he will get no reward. Thus, the FRL objective for this policy is equal to
\begin{align*}
\frac{1}{2} \left(\valuefunc[1][\policy_{\mathrm{det}}^{\star}](s_2) + \valuefunc[2][\policy_{\mathrm{det}}^{\star}](s_2) \right) = \frac{1}{2} \sum_{t=1}^{\infty} 10\discount^t = \frac{5 \discount}{1- \discount} \eqsp,
\end{align*}
proving that $\policy_{\mathrm{sta}}$ achieves a higher value than $\policy_{\mathrm{det}}^{\star}$.

\textbf{Case $\policy_{\mathrm{det}}^{\star}(s_1) = a_1$: } This case is similar to the previous one.
\end{proof}
Combining the two previous lemmas concludes the proof of   \Cref{thm:biggest_upset_of_frl}. 

\subsection{Heterogeneous rewards}
\label{sec:multitask-RL}
To further clarify the novelty of our setting, we contrast it with a commonly studied setup in the literature, often referred to as the federated multi-task RL setting, where agents share identical dynamics but differ in their reward functions. This setting has been explored in prior work \cite{zhu2024towards, chen2021communication, yang2024federated}. This setup does not introduce additional structural challenges and, thus, more closely aligns with the standard single-agent setting. In particular, when agents differ only in rewards, the optimal FRL objective over the space of history-dependent policies is achieved by a deterministic policy. The following lemma formalizes this observation:
\begin{lemma}
Let $\{\cM_c\}_{c=1}^{\nagent}$ be an FRL instance consisting of $\nagent$ MDPs that share the same transition kernel $\kerMDP$ and initial distribution $\rho$, but have distinct reward functions $\rewardMDP[c]$. Denote by $\frlobjective$ the corresponding FRL objective. Then,
\begin{align*}
\max_{\pi \in \PiDet} \frlobjective(\pi) = \max_{\pi \in \Pi_\ell} \frlobjective(\pi) \eqsp,
\end{align*}
and furthermore, the FRL objective is equivalent to the RL objective of a single MDP with reward function equal to the average of the individual rewards.
\end{lemma}

\begin{proof}
Consider an FRL instance where each agent's MDP is defined as $\cM_c \eqdef (\S, \A, \discount, \kerMDP, \rewardMDP[c], \rho)$. Let $\policy = (\policy^t)_{t\in \nset} \in \Pi$ be an arbitrary history-dependent policy. Since all agents share the same transition kernel, their trajectories under $\policy$ follow identical distributions. Precisely, for any $t\geq0$ and $c\in [\nagent]$, it holds that $(\varstate{c}{t}, \varaction{c}{t}) \sim (\varstate{1}{t}, \varaction{1}{t})$. Thus, the FRL objective simplifies as:
\begin{align*}
\frac{1}{\nagent} \sum_{c=1}^{\nagent} \frllocfunc[c](\policy)= \sum_{t=0}^\infty \discount^t \PE_{\policy} \left[ \frac{1}{\nagent} \sum_{c=1}^{\nagent} \rewardMDP[c](\varstate{c}{t}, \varaction{c}{t}) \right] \
= \sum_{t=0}^\infty \discount^t \PE_{\policy} \left[ \bar{\rewardgen}(\varstate{1}{t}, \varaction{1}{t}) \right] \eqsp,
\end{align*}
where $\bar{\rewardgen} \eqdef \frac{1}{\nagent} \sum_{c=1}^{\nagent} \rewardMDP[c]$ denotes the average reward function. This expression corresponds to the standard RL objective of the MDP $(\S, \A, \discount, \kerMDP, \bar{\rewardgen}, \rho)$. By \cite[Theorem~1.7]{agarwal2019reinforcement}, the optimal value of this objective is attained by a deterministic policy, which concludes the proof.
\end{proof}

%% file: AISTATS/appendix/technical_lemmas.tex
\subsection{Basic Lemmas}
For completeness, we state without proof basic results that are routinely used in our proofs.
\begin{lemma}[Theorem 2.1.5, \cite{Nesterov}]
\label{lem:descent_lemma}
If $f\colon \rset^d \rightarrow \rset$ is a $L$-smooth function, then we have for any $x,y \in \rset^d$
\begin{align*}
f(y) \geq f(x) + \pscal{\nabla f(x)}{y -x} - \frac{L}{2} \norm{x-y}[2]^2 \eqsp.
\end{align*}
\end{lemma}

\begin{lemma}
[Reinforce]
\label{lem:reinforce}
Let $(\mathsf{Z},\mathcal{Z})$ be a measurable space, let $\Theta\subset\mathbb{R}^d$ be open, and let $\mu$ be a $\sigma$-finite measure on $(\mathsf{Z},\mathcal{Z})$.  Suppose
\begin{enumerate}
  \item $Y\colon \mathsf{Z}\times\Theta\to\mathbb{R}$ is $\mathcal{Z}\otimes\mathcal{B}(\Theta)$-measurable.
  \item For each $z\in\mathsf{Z}$ and each $i=1,\dots,d$, the partial derivative 
    $$\frac{\partial Y(z,\theta)}{\partial\theta_i}$$ 
    exists for all $\theta\in\Theta$ and the map 
    $$\mathsf{Z}\times\Theta\;\ni\;(z,\theta)\;\mapsto\;\frac{\partial Y(z,\theta)}{\partial\theta_i}$$
    is measurable.
  \item For each $\theta\in\Theta$, $\gamma_\theta\colon\mathsf{Z}\to[0,\infty)$ is a probability density w.r.t.\ $\mu$, and for each $i=1,\dots,d$ the map 
    $$z\;\mapsto\;\frac{\partial\gamma_\theta(z)}{\partial\theta_i}$$
    exists for all $\theta\in\Theta$ and is measurable on $\mathsf{Z}$.
  \item (Dominating function.)  For each $i=1,\dots,d$ and each $\theta_0\in\Theta$, there exist a neighborhood $U\subset\Theta$ of $\theta_0$ and an integrable function $h_i\in L^1(\mu)$ such that for $\mu$-a.e.\ $z\in\mathsf{Z}$ and all $\theta\in U$,
    \[
      \biggl|\frac{\partial}{\partial\theta_i}\Bigl[Y(z,\theta)\,\gamma_\theta(z)\Bigr]\biggr|
      \;=\;
      \Bigl|\frac{\partial Y(z,\theta)}{\partial\theta_i}\,\gamma_\theta(z)
      +Y(z,\theta)\,\frac{\partial\gamma_\theta(z)}{\partial\theta_i}\Bigr|
      \;\le\;h_i(z).
    \]
\end{enumerate}
Define
\[
  J(\theta)\;=\;\int_{\mathsf{Z}}Y(z,\theta)\,\gamma_\theta(z)\,\mu(\rmd z).
\]
Then $J\colon\Theta\to\mathbb{R}$ is continuously differentiable, and for each $i=1,\dots,d$,
\[
  \frac{\partial J(\theta)}{\partial\theta_i}
  =\int_{\mathsf{Z}}
    \frac{\partial}{\partial\theta_i}\Bigl[Y(z,\theta)\,\gamma_\theta(z)\Bigr]
  \,\mu(dz).
\]
Equivalently,
\[
  \frac{\partial J(\theta)}{\partial\theta_i}
  =\int_{\mathsf{Z}}\biggl[
    \frac{\partial Y(z,\theta)}{\partial\theta_i}
    +Y(z,\theta)\,\frac{\partial \ln \gamma_\theta(z)}{\partial\theta_i}
  \biggr] \gamma_\theta(z) \mu(\rmd z).
\]
\end{lemma}

\subsection{Performance difference lemma}

\begin{lemma}[First performance-difference lemma, \cite{kakade2002}]
\label{lem:performance_difference_langford}
Consider an MDP $\cM = (\S, \A, \discount, \kerMDP,\rewardMDP)$ and let  $\valuefunc[\,][\policy]$ and be the value function in this MDP.  For any policies $\policy_1$ and $\policy_2$, it holds
\begin{align*}
\valuefunc[][\policy_1](\initdist) - \valuefunc[][\policy_2](\initdist)  = \frac{1}{1- \discount} \sum_{s\in \S} \occupancy[][\initdist,\policy_1](s) \sum_{a\in \A} \policy_1(a|s) \cdot A^{\policy_2}(s,a) \eqsp,
\end{align*}
where $A^{\policy_2}$ is the advantage function.
\end{lemma}

\begin{lemma}[Second Performance difference lemma, \cite{russo2019worst} ]
\label{lem:russo_performance_difference_lemma}
Let us consider two MDPs $\cM_{1} = (\S, \A, \discount, \kerMDP[1],\rewardMDP[1])$ and $\cM_{2} = (\S, \A, \discount, \kerMDP[2],\rewardMDP[2])$. Let  $\valuefunc[1][\policy]$ and $\valuefunc[2][\policy]$ be respectively the two value functions in these two MDPs. It holds that
\begin{align*}
\valuefunc[1][\policy](s) - \valuefunc[2][\policy](s)  = \CPE[]{\sum_{t=0}^{\infty}\discount^t\left[ \rewardMDP[1](\varstate{}{t},\varaction{}{t}) - \rewardMDP[2](\varstate{}{t},\varaction{}{t}) (\kerMDP[1] - \kerMDP[2]) \valuefunc[2][\policy](\varstate{}{t},\varaction{}{t}) \right]}{s_{0}=s} \eqsp,
\end{align*}
where the expectation is taken over the trajectories $(\varstate{}{0}, \varaction{}{0}, \varstate{}{1}, \varaction{}{1}\dots )$ generated by a stationary policy $\policy$ in the MDP $\cM_{2}$.
\end{lemma}

\begin{lemma}
\label{lem:bound_difference_value_function}
Let us consider two MDPs $\cM_{1} = (\S, \A, \discount, \kerMDP[1],\rewardMDP[1])$ and $\cM_{2} = (\S, \A, \discount, \kerMDP[2],\rewardMDP[2])$ such that $\sup_{s,a \in \S\times \A}\norm{\kerMDP[1](\cdot|s,a) - \kerMDP[2](\cdot|s,a)}[1] \leq \hgkernel$ and $\norm{\rewardMDP[1] - \rewardMDP[2]}[\infty] \leq \hgreward$. For a given stationary policy $\policy$, let  $\valuefunc[1][\policy]$ and $\valuefunc[2][\policy]$ be respectively the two value functions of this policy in these two MDPs. If $\norm{\valuefunc[1][\policy]}[\infty] \leq c$ and $\norm{\valuefunc[2][\policy]}[\infty] \leq c$ then it holds that for all $s\in \S$
\begin{align*}
|\valuefunc[1][\policy](s) - \valuefunc[2][\policy](s)  | \leq \frac{\hgkernel c}{1- \discount} + \frac{\hgreward}{1-\discount} \eqsp.
\end{align*}
\begin{proof}
Follows directly from a combination of \Cref{lem:russo_performance_difference_lemma}, Holder’s inequality and the fact that $\norm{\valuefunc[2][\policy]}[\infty] \leq c$ and $\norm{\valuefunc[2][\policy]}[\infty] \leq c$.
\end{proof}
\end{lemma}
\begin{lemma}
\label{lem:improved_bound_diff_stationary_distribution}
For all $c,c' \in [\nagent]$, it holds that   
\[
\norm{\occupancy[c'][\initdist, \theta] - \occupancy[c][\initdist, \theta]}[1] \leq \frac{\discount \hgkernel}{1-\discount}\eqsp.
\]
\end{lemma}
\begin{proof}
    Let us start from the definition of  flow conservation constraints for occupancy measures \citep{puterman94} for any agent $c \in [\nagent]$
    \[
        \occupancy[c][\initdist, \theta](s) = (1-\discount) \initdist(s) + \discount \sum_{(s',a')} \kerMDP[c](s | s',a') \policy_\theta(a'|s') \occupancy[c][\initdist, \theta](s')\eqsp.
    \]
    Then, we have
    \begin{align*}
        \sum_{s} |\occupancy[c'][\initdist, \theta](s) - \occupancy[c][\initdist, \theta](s)| &\leq \discount \sum_{s',a'} \sum_{s} \left|\kerMDP[c'](s | s',a') \policy_\theta(a'|s') \occupancy[c'][\initdist, \theta](s') - \kerMDP[c](s | s',a') \policy_\theta(a'|s') \occupancy[c][\initdist, \theta](s') \right| \\
        &\leq \discount \sum_{s',a'} \underbrace{\sum_{s} \left|\kerMDP[c'](s | s',a') - \kerMDP[c](s | s',a') \right| }_{\leq \hgkernel}\policy_\theta(a'|s') \occupancy[c'][\initdist, \theta](s') \\
        &+ \discount \sum_{s',a'} \underbrace{\sum_{s}  \kerMDP[c](s | s',a')}_{=1} \policy_\theta(a'|s') \left| \occupancy[c'][\initdist, \theta](s') - \occupancy[c][\initdist, \theta](s') \right| \\
        & \leq \discount \hgkernel + \discount \sum_{s} |\occupancy[c'][\initdist, \theta](s) - \occupancy[c][\initdist, \theta](s)|\eqsp,
    \end{align*} 
which concludes the proof.
\end{proof}

\begin{lemma}
\label{lem:bound_difference_stationary_occupancy_measure}
Consider any two policies $\policy_i$, $i=1,2$, and any agent $c \in [\nagent]$. It holds that   
\begin{align*}
\norm{\occupancy[c][\initdist,\policy_1] - \occupancy[c][\initdist,\policy_2]}[1] \leq \frac{\discount}{1-\discount} \sup_{s\in \S} \norm{\policy_{1}(\cdot|s) - \policy_{2}(\cdot|s)}[1] \eqsp.
\end{align*}
\end{lemma}
\begin{proof}
Let us start from the definition of  flow conservation constraints for the discounted state occupancy \citep{puterman94}, for $i \in \{1,2\}$, we have
\begin{align*}
\occupancy[c][\initdist,\policy_{i}](s) = (1-\discount) \initdist(s) + \discount \sum_{s'} \kerMDP[c,\policy_i](s| s') \occupancy[c][\initdist,\policy_{i}](s')\eqsp.
\end{align*}
Then, we have
\begin{align*}
\sum_{s \in \S} |\occupancy[c][\initdist,\policy_{2}](s) - \occupancy[c][\initdist,\policy_{1}](s)| &\leq \discount \sum_{(s',s')} \sum_{s} \left|\kerMDP[c](s | s',s') \policy_{2}(s'|s') \occupancy[c][\initdist,\policy_2](s') - \kerMDP[c](s | s',s') \policy_{1}(s'|s') \occupancy[c][\initdist,\policy_{1}](s') \right| \\
&\leq \discount \sum_{s',s'} \sum_{s} \kerMDP[c](s | s',s') \left|\policy_{2}(a'|s') -\policy_{1}(a'|s') \right| \occupancy[c][\initdist,\policy_{2}](s') \\
&+ \discount \sum_{s',a'} \sum_{s}  \kerMDP[c](s | s',a') \policy_{1}(a'|s') \left| \occupancy[c][\initdist,\policy_{1}](s') - \occupancy[c][\initdist,\policy_2](s') \right| \\
& \leq \discount \sup_{s \in \S} \norm{\policy_{1}(\cdot|s) - \policy_{2}(\cdot|s)}[1] +  \discount \sum_{s'} |\occupancy[c][\initdist,\policy_{1}](s') - \occupancy[c][\initdist,\policy_{2}](s')|\eqsp,
\end{align*} 
which concludes the proof.
\end{proof}

\subsection{Properties of softmax parametrization and value}
In this section, we derive useful technical inequalities that show bounds on the derivatives of the softmax parametrization. The results for the first two differentials could be extracted from \cite{mei2020global}.

\begin{lemma}\label{lem:softmax_parametriztaion_derivatives}
    For any $u,v,w \in \rset^{\S \times \A}$, we have
    \begin{align*}
        |\rmd \policy_\theta[u] (a|s) | &\leq 2 \policy_\theta(a|s) \norm{u}[\infty]\,, \\
        |\rmd^2 \policy_\theta[u,v] (a|s) | &\leq 8 \policy_\theta(a|s) \norm{u}[\infty]\norm{v}[\infty]\,, \\ 
        |\rmd^3 \policy_\theta[u,v,w] (a|s) | &\leq 48 \policy_\theta(a|s) \norm{u}[\infty]\norm{v}[\infty]\norm{w}[\infty]\,.
    \end{align*}
\end{lemma}
\begin{proof}
    Let us start from the expression for the derivative of parametrization (see, e.g., Lemma C.1. of \cite{agarwal2020optimality})
    \[
        \frac{\partial \policy_\theta(a|s)}{\partial \theta(s,a_1)} = \policy_\theta(a|s) (\Ind_{a}(a_1) -  \policy_\theta(a_1|s))\,,
    \]
    thus 
    \[
        \rmd \policy_\theta[u](a|s) = \policy_\theta(a|s) \cdot \left(u(s,a) - \langle \policy_\theta(\cdot | s), u(s,\cdot) \rangle\right)\,. 
    \]
    To simplify the following notation, we define a random variable $A \sim \policy_\theta(\cdot | s)$, then we have
    \[
        \rmd \policy_\theta[u](a|s) = \policy_\theta(a|s) \cdot \left( u(s,a) - \PE_{\policy_\theta}[u(s,A)]\right)\,.
    \]
    Using the fact that $|u(s,a) - \PE_{\policy_\theta}[u(s,A)]| \leq 2 \norm{u}[\infty]$, we conclude the first statement.
    
    Next, we continue by deriving the second derivative
    \begin{align*}
        \frac{\partial^2 \policy_\theta(a|s)}{\partial\theta(s,a_1)\partial\theta(s,a_2)} &= \policy_\theta(a|s) (\Ind_{a}(a_2) -\policy_\theta(a_2|s)) (\Ind_{a}(a_1) - \policy_\theta(a_1|s)) \\
        &-  \policy_\theta(a|s) \policy_\theta(a_1|s) (\Ind_{a_1}(a_2) -\policy_\theta(a_2|s)) \\
        &= \policy_\theta(a|s) \left( (\Ind_{a}(a_2) -\policy_\theta(a_2|s)) (\Ind_{a}(a_1) - \policy_\theta(a_1|s)) - \policy_\theta(a_1|s) (\Ind_{a_1}(a_2) -\policy_\theta(a_2|s)) \right)\,.
    \end{align*}
    In particular, we have
    \begin{align*}
        \rmd^2 \policy_\theta[u,v](a|s) &= \policy_\theta(a|s) \sum_{a_1,a_2}  ((\Ind_{a}(a_2) -\policy_\theta(a_2|s)) (\Ind_{a}(a_1) - \policy_\theta(a_1|s))))  u(s,a_1) u(s,a_2)\\
        &\quad - \policy_\theta(a|s) \sum_{a_1,a_2} \policy_\theta(a_1|s) (\Ind_{a_1}(a_2) -\policy_\theta(a_2|s)) u(s,a_1) v(s,a_2) \\
        &= \policy_\theta(a|s) (u(s,a) - \langle \policy_\theta(\cdot|s), u(s,\cdot) \rangle ) (v(s,a) - \langle \policy_\theta(\cdot|s), v(s,\cdot) \rangle ) \\
        &\quad - \policy_\theta(a|s) \left( \langle \policy_\theta(\cdot | s) , u(s,\cdot ) \cdot  v(s,\cdot)\rangle - \langle \policy_\theta(\cdot | s), u(s,\cdot) \rangle \cdot  \langle \policy_\theta(\cdot | s), v(s,\cdot) \rangle \right)\,.    
    \end{align*}
    Using the same inequality, we have
    \[
        |\rmd^2 \policy_\theta[u,v](a|s) |\leq 8 \policy_\theta(a|s) \norm{u}[\infty] \norm{v}[\infty]\,.
    \]
    Finally, we continue with the computation of the third differential:
    \begin{align*}
        \rmd^3 \policy_\theta[u,v,w](a|s) &= \underbrace{\rmd\left[\policy_\theta(a|s) (u(s,a) - \langle \policy_\theta(\cdot|s), u(s,\cdot) \rangle ) (v(s,a) - \langle \policy_\theta(\cdot|s), v(s,\cdot) \rangle )\right][w]}_{\term{D_1}} \\
        &\quad - \underbrace{\rmd\left[\policy_\theta(a|s) \left( \langle \policy_\theta(\cdot | s) , u(s,\cdot ) \cdot  v(s,\cdot)\rangle - \langle \policy_\theta(\cdot | s), u(s,\cdot) \rangle \cdot  \langle \policy_\theta(\cdot | s), v(s,\cdot) \rangle \right)\right][w]}_{\term{D_2}}\,.
    \end{align*}
    Next, we consider each term separately. First, we have
    \begin{align*}
        \term{D_1} &= \rmd\policy_\theta[w](a|s) \cdot  (u(s,a) - \langle \policy_\theta(\cdot|s), u(s,\cdot) \rangle ) (v(s,a) - \langle \policy_\theta(\cdot|s), v(s,\cdot) \rangle )[w] \\
        &\quad - \policy_\theta(a|s)  \langle \rmd \policy_\theta[w](\cdot|s), u(s,\cdot) \rangle ) (v(s,a) - \langle \policy_\theta(\cdot|s), v(s,\cdot) \rangle ) \\
        &\quad - \policy_\theta(a|s) (u(s,a) - \langle \policy_\theta(\cdot|s), u(s,\cdot) \rangle ) \langle \rmd \policy_\theta(\cdot |s)[w],  v(s,\cdot) \rangle\,. 
    \end{align*}
    To bound this term, we notice that for any $x \in \rset^{\S \times \A}$ it holds
    \begin{align*}
        \langle \rmd \policy_\theta(\cdot |s)[w],  x(s,\cdot) \rangle &=  \sum_{a\in \A} \rmd \policy_\theta(a |s)[w] \cdot  x(s,a)  \\
        &= \sum_{a\in \A} \policy_\theta(a|s) (w(s,a) - \langle \policy_\theta(\cdot | s), w(s,\cdot) \rangle ) x(s,a)  \\
        &= \PE\left[ x(s,A) w(s,A) \right] - \PE\left[ x(s,A) \right]\PE\left[ w(s,A) \right] = \mathrm{Cov}(x(s,A), w(s,A))\,,
    \end{align*}
    where a random variable $A$ follows $\policy_\theta(\cdot |s)$. Using this relation, we have
    \begin{align*}
        |\term{D_1}| &\leq \policy_\theta(a|s) \cdot |w(s,a) - \PE[w(s,A)] | \cdot |u(s,a) - \PE[w(s,A)]| \cdot |v(s,a) - \PE[w(s,A)]| \\
        &+ \policy_\theta(a|s) |\mathrm{Cov}(u(s,A), w(s,A))| |v(s,a) - \PE[v(s,A)]| \\
        &+ \policy_\theta(a|s) |\mathrm{Cov}(v(s,A), w(s,A))| |u(s,a) - \PE[u(s,A)]|\,.
    \end{align*}
    Next, we notice that $|x(s,a) - \PE[x(s,A)]| \leq 2\norm{x}[\infty]$ for any $x \in \rset^{\S \times \A}$, and, as a result, $|\mathrm{Cov}(x(s,A), w(s,A))| \leq 4\norm{x}[\infty] \norm{w}[\infty]$. Thus, we have
    \[
        |\term{D_1}| \leq 24 \policy_\theta(a|s) \norm{u}[\infty] \norm{v}[\infty] \norm{w}[\infty]\,.
    \]
    Next, we analyze the second term. For this term, we have
    \begin{align*}
        \term{D_2} &= \rmd\policy_\theta(a|s)[w] \cdot \mathrm{Cov}(u(s,A), v(s,A))  + \policy_\theta(a|s) \bigg( \langle \rmd \policy_\theta(\cdot | s)[w] , u(s,\cdot ) \cdot  v(s,\cdot)\rangle \\
        &\qquad - \langle \rmd \policy_\theta(\cdot | s)[w], u(s,\cdot) \rangle \cdot  \langle \policy_\theta(\cdot | s), v(s,\cdot) \rangle
        -  \langle  \policy_\theta(\cdot | s), u(s,\cdot) \rangle \cdot  \langle \rmd \policy_\theta[w](\cdot | s), v(s,\cdot) \rangle\bigg)\,.
    \end{align*}
    By the same reasoning as for term $\term{D_1}$, we have
    \[
        |\term{D_2}|  \leq 24  \policy_\theta(a|s)\norm{u}[\infty] \norm{v}[\infty] \norm{w}[\infty]\,,
    \]
    thus we have
    \[
        |\rmd^3 \policy_\theta[u,v,w]|(a|s)  \leq 48 \policy_\theta(a|s)\norm{u}[\infty] \norm{v}[\infty] \norm{w}[\infty]\,.
    \]
\end{proof}

\begin{lemma}\label{lem:derivatives_entropy}
    Let $\entr(\policy_\theta) \in \rset^\S$ be a vector of entropies of policy $\policy_\theta$. Then we have
    \begin{align*}
        \norm{\rmd \entr(\policy_\theta)[u] }[\infty] &\leq 2\log|\A| \cdot  \norm{u}[\infty]\,, \\
        \norm{\rmd^2 \entr(\policy_\theta)[u,v] }[\infty] &\leq (4 + 8\log |\A|) \norm{u}[\infty] \norm{v}[\infty]\,, \\
        \norm{\rmd^3 \entr(\policy_\theta)[u,v,w] }[\infty] &\leq (56 + 48 \log |\A|)  \norm{u}[\infty] \norm{v}[\infty] \norm{w}[\infty]\,
    \end{align*}
\end{lemma}
\begin{proof}
    We recall that $\entr(\policy_\theta)(s) = - \sum_{a \in \A} \policy_\theta(a |s) \log \policy_\theta(a |s)$\,. 
    
    Define a function $h(x) = -x \log x$, then we have $h'(x) = -(\log x + 1), h''(x) = -1/x, h'''(x) = 1/x^2$. Thus, by \Cref{lem:softmax_parametriztaion_derivatives} we have
    \begin{align*}
        \rmd \entr(\policy_\theta)[u](s) &= \sum_{a \in \A} \rmd h(\policy_\theta(a|s))[u] = \sum_{a \in \A}h'(\policy_\theta(a|s)) \cdot  \rmd \policy_\theta[u](a |s)  \\
        &= \sum_{a \in \A} -(\log \policy_\theta(a|s) + 1) \cdot \policy_\theta(a|s) (u(s,a) - \langle \policy_\theta(\cdot|s), u(s,\cdot)\rangle).
    \end{align*}
    Notice that
    \[
        \sum_{a \in \A}\policy_\theta(a|s) (u(s,a) - \langle \policy_\theta(\cdot|s), u(s,\cdot)\rangle) = 0\,,
    \]
    thus, using $|u(s,a) - \langle \policy_\theta(\cdot, s), u\rangle| \leq 2 \norm{u}[\infty]$ and $\sum_{a \in \A}|\policy_\theta(a|s) \log \policy_\theta(a|s)| \leq \log|\A|$, we conclude the first statement.

    Next, we have to compute the second differential; here we have by a high-order chain rule
    \begin{align*}
        \rmd^2 \entr(\policy_\theta)[u,v](s) &= \sum_{a\in \A} \rmd^2 h(\policy_\theta(a|s))[u] \\
        &=   \sum_{a\in\A} h''(\policy_\theta(a|s)) \rmd \policy_\theta(a|s)[u] \rmd \policy_\theta(a|s)[v] + \sum_{a\in\A} h'(\policy_\theta(a|s)) \rmd^2 \policy_\theta(a|s)[u,v] \\
&= \sum_{a\in\A} \left(-\frac{1}{\policy_\theta(a|s)}\right) \rmd \policy_\theta(a|s)[u] \rmd \policy_\theta(a|s)[v] - \sum_{a\in\A} (\log \policy_\theta(a|s) + 1) \rmd^2 \policy_\theta(a|s)[u,v]\,.
    \end{align*}
    Next, we see that by linearity
    \[
        \sum_{a \in \A} \rmd \pi_\theta(a|s)[u] = \rmd  \left(\sum_{a \in \A} \pi_\theta(a|s)\right)[u] = 0\,,
    \]
    thus the sum of second and third derivatives also should be equal to zero. 
    
    Using a bound from \Cref{lem:softmax_parametriztaion_derivatives}, we have
    \[
        |\rmd^2 \entr(\policy_\theta)[u,v](s)| \leq \sum_{a \in \A} 4 \pi_\theta(a|s) \norm{u}[\infty] \norm{v}[\infty] + 8 \sum_{a \in \A} |\log \policy_\theta(a|s)| \cdot \policy_\theta(a|s)\norm{u}[\infty] \norm{v}[\infty]\,.
    \]
    By a bound on entropy, we conclude the second statement.

    For the last statement, we also apply the high-order chain rule to have
    \begin{align*}
        \rmd^3 \entr(\policy_\theta)[u,v,w](s) &= \sum_{a \in \A} h'''(\pi_\theta(a|s)) \rmd \policy_\theta(a|s)[u] \rmd \policy_\theta(a|s)[v] \policy_\theta(a|s)[w]  \\
        &\quad + \sum_{a\in\A} h''(\policy_\theta(a|s)) \rmd^2 \policy_\theta(a|s)[u,w] \rmd \policy_\theta(a|s)[v] \\
        &\quad + \sum_{a\in\A} h''(\policy_\theta(a|s)) \rmd \policy_\theta(a|s)[u] \rmd^2 \policy_\theta(a|s)[v,w] \\
        &\quad+ \sum_{a \in \A} h''(\policy_\theta(a|s)) \rmd \policy_\theta(a|s)[w] \rmd^2 \policy_\theta(a|s)[u,v] \\
        &\quad + \sum_{a\in\A} h'(\policy_\theta(a|s)) \rmd^3 \policy_\theta(a|s)[u,v,w]\,.
    \end{align*}
    Using a fact that $\sum_{a\in \A} \rmd^3 \policy_\theta(a|s)[u,v,w] = 0$, we have the following from by \Cref{lem:softmax_parametriztaion_derivatives}
    \[
        |\rmd^3 \entr(\policy_\theta)[u,v,w](s)| \leq (56 + 48 \log |\A|)  \norm{u}[\infty] \norm{v}[\infty] \norm{w}[\infty]\,.
    \]

\end{proof}

\begin{lemma}\label{lem:derivatives_value}
    Let $\regvaluefunc[c][\pi_\theta]$ be a regularized value function in the MDP that corresponds to an agent $c \in [\nagent]$. Then for any $u,v,w \in \rset^{\S \times \A}$, its directional derivatives satisfy the following bounds
    \begin{align*}
        \norm{\rmd \regvaluefunc[c][\pi_{\theta}][u]}[\infty] &\leq \frac{8 + 10 \temp \log |\A|}{1-\discount} \norm{u}[\infty]\,, \\
        \norm{\rmd^2 \regvaluefunc[c][\pi_{\theta}][u,v]}[\infty] & \leq \frac{40 + 60 \temp \log |\A|}{(1-\discount)^3} \norm{u}[\infty] \norm{v}[\infty] \,, \\
        \norm{\rmd^3 \regvaluefunc[c][\pi_{\theta}][u,v,w]}[\infty] & \leq \frac{ 480 + 832 \temp \log |\A|}{(1-\discount)^4} \norm{u}[\infty]\norm{v}[\infty] \norm{w}[\infty]\,.
    \end{align*}
\end{lemma}
\begin{proof}
    Let us start by writing down regularized Bellman equations (see, e.g., \cite{geist2019theory}). In the following, we treat $\regqfunc[c][\policy]$ as a matrix of size $\S \times \A$ with elements $\regqfunc[c][\policy](s,a)$ and $\policy_\theta$ as a matrix of size $\A \times \S$ with elements $\policy_\theta(a|s)$,
    \[
        \regvaluefunc[c][\pi_{\theta}] =  \regqfunc[c][\policy_\theta] \cdot \policy_\theta + \temp \entr(\policy_\theta)\,,\qquad  \regqfunc[c][\policy_\theta] = \rewardMDP + \discount \kerMDP[c] \regvaluefunc[c][\policy_\theta],
    \]
    where $\kerMDP[c]$ is a linear operator from a space of vectors of size $\S$ to a space of matrices of size $\S \times \A$, and $\entr(\policy) \in \rset^\S$ is a vector of policy entropies for each state.

    \paragraph{First differential.} We start as follows
    \[
        \rmd \regvaluefunc[c][\pi_{\theta}][u] = \regqfunc[c][\policy_\theta] \cdot \rmd \policy_\theta[u] + \rmd \regqfunc[c][\policy_\theta][u] \cdot \policy_\theta + \temp \rmd \entr(\policy_\theta)[u], \qquad \rmd \regqfunc[c][\policy_\theta][u] = \discount \kerMDP[c] \rmd \regvaluefunc[c][\policy_\theta][u]\,.
    \]
    Thus, we have
    \[
        \rmd \regvaluefunc[c][\pi_{\theta}][u] = \regqfunc[c][\policy_\theta] \cdot \rmd \policy_\theta[u] + \discount\kerMDP[c] \rmd \regvaluefunc[c][\policy_\theta][u] \cdot \policy_\theta + \temp \rmd \entr(\policy_\theta)[u]\,.
    \]
    As a result, we have
    \begin{equation}\label{eq:value_differnetial}
        \norm{\rmd \regvaluefunc[c][\pi_{\theta}][u]}[\infty] \leq  \norm{\regqfunc[c][\policy_\theta] \cdot \rmd \policy_\theta[u]}[\infty] + \discount \norm{\kerMDP[c] \rmd \regvaluefunc[c][\policy_\theta] [u] \cdot \policy_\theta}[\infty] +  \temp \norm{\rmd \entr(\policy_\theta)[u]}[\infty]\,.
    \end{equation}
    For the first term, we have for any $s \in \S$ by a simple bound on Q-value and \Cref{lem:softmax_parametriztaion_derivatives}
    \[
        |\regqfunc[c][\policy_\theta] \cdot \rmd \policy_\theta[u]|(s) \leq \frac{1 + \temp \log |\A|}{1-\discount} \sum_{a\in \A} | \rmd \policy_\theta[u](a|s) | \leq \frac{8(1 + \temp \log A)}{1-\discount} \norm{u}[\infty]\,.
    \]
    For the second term, we have for any $s \in \S$
    \[
         \norm{\kerMDP[c] \rmd \regvaluefunc[c][\policy_\theta] [u] \cdot \policy_\theta}[\infty]  = \max_{s} \left|\sum_{a,s'} \kerMDP[c](s'|s,a) \rmd \regvaluefunc[c][\policy_\theta][u] \cdot \policy_\theta(a|s)\right| \leq \norm{\regvaluefunc[c][\policy_\theta][u]}[\infty]\,.
    \]
    finally, by \Cref{lem:derivatives_entropy} we have
    \[
        \norm{\rmd \entr(\policy_\theta)[u]}[\infty] \leq 2 \log |\A| \cdot  \norm{u}[\infty].
    \]
    Thus, from \eqref{eq:value_differnetial} it holds
    \[
        \norm{\rmd \regvaluefunc[c][\pi_{\theta}][u]}[\infty] \leq \discount \norm{\rmd \regvaluefunc[c][\pi_{\theta}][u]}[\infty] + \frac{8 + 10 \temp \log |\A|}{1-\discount} \norm{u}\,.
    \]
    Rearranging the terms, we conclude the first statement.

    \paragraph{Second differential.} For the second differential, we have
    \begin{align*}
        \rmd^2 \regvaluefunc[c][\pi_{\theta}][u,v] &= \rmd \left(\regqfunc[c][\policy_\theta] \cdot \rmd \policy_\theta[u]\right)[v] + \discount \rmd\left( \kerMDP[c] \rmd \regvaluefunc[c][\policy_\theta][u] \cdot \policy_\theta \right)[v] + \temp \rmd^2 \entr(\policy_\theta)[u,v] \\
        &= \left( \rmd\regqfunc[c][\policy_\theta][v] \right) \rmd \policy_\theta[u] + \regqfunc[c][\policy_\theta] \cdot \rmd^2 \policy_\theta[u,v] + \discount \kerMDP[c]\rmd^2 \regvaluefunc[c][\policy_\theta][u,v] \cdot \policy_\theta \\
        &\qquad + \discount \kerMDP[c]\rmd \regvaluefunc[c][\policy_\theta][u] \rmd \policy_\theta[v] + \temp \rmd^2 \entr(\policy_\theta)[u,v] \\
        &= \regqfunc[c][\policy_\theta] \cdot \rmd^2 \policy_\theta[u,v] + \discount \kerMDP[c]\rmd \regvaluefunc[c][\policy_\theta][u] \rmd \policy_\theta[v] + \discount \kerMDP[c]\rmd \regvaluefunc[c][\policy_\theta][v] \rmd \policy_\theta[u] \\
        &\qquad +\discount \kerMDP[c]\rmd^2 \regvaluefunc[c][\policy_\theta][u,v] \cdot \policy_\theta + \temp \rmd^2 \entr(\policy_\theta)[u,v]\,.
    \end{align*}
    Next, to derive a bound, we apply the bound on the first differential of the value as well \Cref{lem:softmax_parametriztaion_derivatives} and \Cref{lem:derivatives_entropy}:
    \begin{align*}
        |\regqfunc[c][\policy_\theta] \cdot \rmd^2 \policy_\theta[u,v]|(s) &\leq \frac{1+\temp \log |\A|}{1-\discount} \sum_{a \in \A}  |\rmd^2 \policy_\theta(a|s)[u,v]|  \leq \frac{8(1 + \temp \log |\A|))}{1-\discount} \norm{u}[\infty] \norm{v}[\infty]\,, \\
        |\kerMDP[c]\rmd \regvaluefunc[c][\policy_\theta][u] \rmd \policy_\theta[v]|(s) &\leq  \norm{ \rmd \regvaluefunc[c][\policy_\theta][u]}[\infty] \sum_{a\in \A}  | \rmd \policy_\theta(a|s)[v]| \leq \frac{16 + 20 \lambda \log |\A|}{(1-\discount)^2} \norm{u}[\infty] \norm{v}[\infty]\,,\\
        |\kerMDP[c]\rmd^2 \regvaluefunc[c][\policy_\theta][u,v] \cdot \policy_\theta|(s) &\leq \norm{\rmd^2 \regvaluefunc[c][\policy_\theta][u,v]}[\infty]\,,\\
        |\rmd^2 \entr(\policy_\theta)[u,v]|(s) &\leq (4 + 8 \log |\A|) \norm{u}[\infty] \norm{v}[\infty]\,,
    \end{align*}
    thus
    \begin{align*}
        \norm{\rmd^2 \regvaluefunc[c][\pi_{\theta}][u,v]}[\infty] &\leq \discount  \norm{\rmd^2 \regvaluefunc[c][\pi_{\theta}][u,v]}[\infty] \\
        &+ \left( \frac{32 + 40 \temp \log |\A|}{(1-\discount)^2} + \frac{8(1 + \temp \log |\A|)}{1-\discount} + \temp (4 + 8 \log |\A|) \right) \norm{u}[\infty] \norm{v}[\infty]\,.
    \end{align*}
    Since $|\A| \geq 2$, then $2\log |\A| \geq 1$, we can simplify it as follows
    \[
        \norm{\rmd^2 \regvaluefunc[c][\pi_{\theta}][u,v]}[\infty]  \leq \frac{40 + 64 \temp \log |\A|}{(1-\discount)^3} \norm{u}[\infty] \norm{v}[\infty]\,.
    \]

    \paragraph{Third differential.} Next, we proceed with the third differential as follows
    \begin{align*}
        \rmd^3 \regvaluefunc[c][\pi_{\theta}][u,v,w] &= \regqfunc[c][\policy_\theta] \cdot \rmd^3 \policy_\theta[u,v,w]  + \discount \kerMDP[c]\rmd \regvaluefunc[c][\policy_\theta][w] \cdot \rmd^2 \policy_\theta[u,v] \\
        &\quad + \discount \kerMDP[c]\rmd^2 \regvaluefunc[c][\policy_\theta][u,w] \rmd \policy_\theta[v] + \discount \kerMDP[c]\rmd \regvaluefunc[c][\policy_\theta][u] \rmd^2 \policy_\theta[v,w] \\
        &\quad + \discount \kerMDP[c]\rmd^2 \regvaluefunc[c][\policy_\theta][v,w] \rmd \policy_\theta[u] + \discount \kerMDP[c]\rmd \regvaluefunc[c][\policy_\theta][v] \rmd^2 \policy_\theta[u,w] \\
        &\quad + \discount \kerMDP[c]\rmd^2 \regvaluefunc[c][\policy_\theta][u,v] \cdot \rmd \policy_\theta[w] + \discount \kerMDP[c]\rmd^3\regvaluefunc[c][\policy_\theta][u,v,w] \cdot \policy_\theta + \rmd^3 \entr(\policy_\theta)[u,v,w]\,.
    \end{align*}
    By the triangle inequality
    \begin{align*}
        \norm{\rmd^3 \valuefunc[c][\pi_{\theta}][u,v,w]}[\infty] &\leq \norm{\qfunc[c][\policy_\theta] \cdot \rmd^3 \policy_\theta[u,v,w]}[\infty]  + \discount \norm{\kerMDP[c]\rmd \valuefunc[c][\policy_\theta][w] \cdot \rmd^2 \policy_\theta[u,v]}[\infty] \\
        &\quad + \discount \norm{\kerMDP[c]\rmd^2 \valuefunc[c][\policy_\theta][u,w] \rmd \policy_\theta[v]}[\infty] + \discount \norm{\kerMDP[c]\rmd \valuefunc[c][\policy_\theta][u] \rmd^2 \policy_\theta[v,w]}[\infty] \\
        &\quad + \discount \norm{\kerMDP[c]\rmd^2 \valuefunc[c][\policy_\theta][v,w] \rmd \policy_\theta[u]}[\infty] + \discount \norm{\kerMDP[c]\rmd \valuefunc[c][\policy_\theta][v] \rmd^2 \policy_\theta[u,w]}[\infty] \\
        &\quad + \discount \norm{\kerMDP[c]\rmd^2 \valuefunc[c][\policy_\theta][u,v] \cdot \rmd \policy_\theta[w]}[\infty] + \discount \norm{\kerMDP[c]\rmd^3\valuefunc[c][\policy_\theta][u,v,w] \cdot \policy_\theta}[\infty] \\
        &\quad + \norm{\rmd^3 \entr(\policy_\theta)[u,v,w]}[\infty]\,.
    \end{align*}
    To simplify notation, let us define $R_{1,2}(u,v,w) \eqdef \norm{\kerMDP[c]\rmd \valuefunc[c][\policy_\theta][u,v] \rmd^2 \policy_\theta[w]}[\infty]$ and $R_{2,1}(u,v,w) \eqdef \norm{\kerMDP[c]\rmd^2 \valuefunc[c][\policy_\theta][u,v] \rmd \policy_\theta[w]}[\infty]$. Next, we notice that
    \[
        \norm{\kerMDP[c]\rmd^3\valuefunc[c][\policy_\theta][u,v,w] \cdot \policy_\theta}[\infty] = \max_{s} \left| \sum_{s'} \policy_\theta(a|s) \kerMDP(s'|s,a) \rmd^3\valuefunc[c][\policy_\theta][u,v,w]_{s'} \right|   \leq  \norm{\rmd^3\valuefunc[c][\policy_\theta][u,v,w]}[\infty]\,,
    \]
    thus, we have a contraction argument that implies
    \begin{equation}\label{eq:third_derivative_recursive_bound}
        \begin{split}
        \norm{\rmd^3 \valuefunc[c][\pi_{\theta}][u,v,w]}[\infty]       \leq \frac{1}{1-\discount}\bigg(& \norm{\qfunc[c][\policy_\theta] \cdot \rmd^3 \policy_\theta[u,v,w]}[\infty] + \norm{\rmd^3 \entr(\policy_\theta)[u,v,w]}[\infty] \\
        &+ \discount (R_{1,2}(w,u,v) + R_{1,2}(u,v,w) + R_{1,2}(v,u,w)) \\
        & + \discount (R_{2,1}(u,w,v) + R_{2,1}(v,w,u) + R_{2,1}(u,v,w))  \bigg)\,.
        \end{split}
    \end{equation}

    Next, we bound all terms that appear in the bound above. First, we apply \Cref{lem:softmax_parametriztaion_derivatives} for a fixed state $s \in \S$
    \begin{align*}
        |\regqfunc[c][\policy_\theta] \cdot \rmd^3 \policy_\theta[u,v,w]|(s) &\leq \sum_{a \in \A} |\regqfunc[c][\policy_\theta](s,a) \cdot \rmd^3 \policy_\theta[u,v,w](a|s) |\\
        &\leq \frac{1+\temp \log |\A|}{1-\discount} \sum_{a\in \A}|\rmd^3 \policy_\theta[u,v,w](a|s)| \\
        &\leq \frac{48(1+\temp \log |\A|)}{1-\discount} \norm{u}[\infty] \norm{v}[\infty] \norm{w}[\infty]\,.
    \end{align*}
    Also, by \Cref{lem:derivatives_entropy} we have
    \[  
        \norm{\rmd^3 \entr(\policy_\theta)[u,v,w]}[\infty] \leq (56 + 48 \log |\A|) \norm{u}[\infty] \norm{v}[\infty] \norm{w}[\infty]\,. 
    \]

    Next we bound $R_{1,2}$ as follows
    \begin{align*}
        R_{1,2}(u,v,w) &= \max_{s\in \S} \left| \sum_{a \in \A} (\kerMDP[c]\rmd \regvaluefunc[c][\policy_\theta][u])(s,a) \rmd^2 \policy_\theta[v,w](a|s)\right| \\
        &\leq \norm{\rmd \regvaluefunc[c][\policy_\theta][u]}[\infty] \cdot \max_{s\in \S} \sum_{a \in \A} \left| \rmd^2 \policy_\theta[v,w](a|s) \right|\,.
    \end{align*}
    Applying the bound for the first differential as well as \Cref{lem:softmax_parametriztaion_derivatives}
    \[
        R_{1,2}(u,v,w) \leq \frac{8\cdot(8 + 10 \temp \log |\A|)}{(1-\discount)^2} \norm{u}[\infty] \norm{v}[\infty] \norm{w}[\infty]\,.
    \]
    Finally, using the same idea, we have the following bound for $R_{2,1}$:
    \begin{align*}
        R_{2,1}(u,v,w) &\leq \norm{\rmd^2 \regvaluefunc[c][\policy_\theta][u,v]}[\infty] \cdot \max_{s\in \S} \sum_{a \in \A} \left| \rmd \policy_\theta[w](a|s) \right| \\
        &\leq \frac{2 \cdot \left( 40 + 64 \temp \log |\A| \right)}{(1-\discount)^3} \norm{u}[\infty] \norm{v}[\infty] \norm{w}[\infty]\,.
    \end{align*}
    Overall, we can rewrite \eqref{eq:third_derivative_recursive_bound} as follows
    \begin{align*}
        \norm{\rmd^3 \valuefunc[c][\pi_{\theta}][u,v,w]}[\infty]       &\leq \frac{1}{1-\discount} \bigg( \frac{48(1 + \temp \log |\A|)}{1-\discount} + \temp (56 + 48 \log |\A|)\\
        &\quad +\frac{24\cdot(8 + 10 \temp \log |\A|)}{(1-\discount)^2} +  \frac{6 \cdot \left( 40 + 64 \temp \log |\A| \right)}{(1-\discount)^3} \bigg) \norm{u}[\infty]\norm{v}[\infty] \norm{w}[\infty]\,,
    \end{align*}
    and, after rearranging the terms and using a bound $2 \log |\A| \geq 1$, we have the following bound
    \[
        \norm{\rmd^3 \valuefunc[c][\pi_{\theta}][u,v,w]}[\infty] \leq \frac{ 480 + 832 \temp \log |\A|}{(1-\discount)^4} \norm{u}[\infty]\norm{v}[\infty] \norm{w}[\infty]\,.
    \]
\end{proof}